\documentclass{article}

\usepackage[table,xcdraw]{xcolor}
\usepackage{microtype}
\usepackage{graphicx}
\usepackage{booktabs} %

\usepackage{hyperref}

\usepackage[accepted]{icml2024}

\usepackage{amsmath}
\usepackage{amssymb}
\usepackage{mathtools}
\usepackage{amsthm}

\usepackage[capitalize,noabbrev]{cleveref}

\theoremstyle{plain}

\theoremstyle{definition}

\theoremstyle{remark}

\usepackage[textsize=tiny]{todonotes}

\usepackage{macros}

\icmltitlerunning{A Group Fairness Framework for Post-Processing Everything}

\begin{document}

\twocolumn[
\icmltitle{\ours: A Group Fairness Framework for Post-Processing Everything}
\icmlsetsymbol{work}{*}

\begin{icmlauthorlist}
\icmlauthor{Alexandru \c{T}ifrea$^*$}{eth}
\icmlauthor{Preethi Lahoti}{g}
\icmlauthor{Ben Packer}{g}
\icmlauthor{Yoni Halpern}{g}
\icmlauthor{Ahmad Beirami}{g}
\icmlauthor{Flavien Prost}{g}
\end{icmlauthorlist}

\icmlaffiliation{eth}{Department of Computer Science, ETH Zurich}
\icmlaffiliation{g}{Google DeepMind}

\icmlcorrespondingauthor{Alexandru \c{T}ifrea}{alexandru.tifrea@inf.ethz.ch}
\icmlcorrespondingauthor{Flavien Prost}{fprost@google.com}

\vskip 0.3in
]

\printAffiliationsAndNotice{$^*$Work done during an internship at Google.} %

\begin{abstract}

Despite achieving promising fairness-error trade-offs, in-processing mitigation techniques for group fairness cannot be employed in numerous practical applications with limited computation resources or no access to the training pipeline of the prediction model. In these situations, post-processing is a viable alternative. However, current methods are tailored to specific problem settings and fairness definitions and hence, are not as broadly applicable as in-processing. In this work, we propose a framework that turns any regularized in-processing method into a post-processing approach. This procedure prescribes a way to obtain post-processing techniques for a much broader range of problem settings than the prior post-processing literature. We show theoretically and through extensive experiments that our framework preserves the good fairness-error trade-offs achieved with in-processing and can improve over the effectiveness of prior post-processing methods. Finally, we demonstrate several advantages of a modular mitigation strategy that disentangles the training of the prediction model from the fairness mitigation, including better performance on tasks with partial group labels.\footnote{Code is available at \url{https://github.com/google-research/google-research/tree/master/postproc_fairness}.}

\end{abstract}

\section{Introduction}
\label{sec:intro}

As machine learning (ML) algorithms are deployed in applications with a profound social impact, it becomes crucial that the biases they exhibit \citep{bickel75, dastin18, mehrabi21} are properly mitigated. 
Of particular importance is being equitable with respect to the different subgroups in the data (i.e.\ \emph{group fairness}) where groups
are determined by sensitive demographic attributes such as race, sex, age, etc \citep{barocas23}. To operationalize fairness, one can choose between the several dozen alternative definitions of fairness \citep{narayanan18, hort23} capturing different notions of equity.

\begin{figure}[t]
    \centering
    \includegraphics[width=\columnwidth]{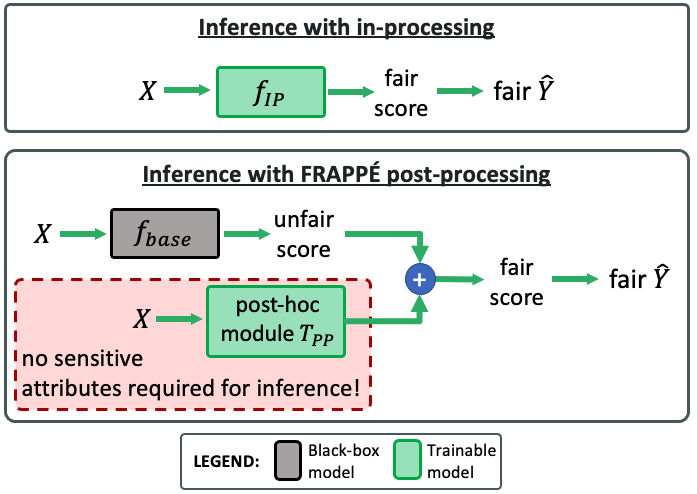}
    \caption{\textbf{Inference with $\ours$ and in-processing.} $\ours$ methods add the output of post-hoc module $\mult$ to the unfair scores output by pre-trained model $\fctbase$. Unlike prior post-processing methods, $\ours$ does not require sensitive attributes for inference. 
    While in-processing trains the entire prediction model $\fct_{IP}$ to induce fairness, $\ours$ only trains the post-hoc module.
    Note that, for classification, thresholding the predicted scores yields outputs $\Yhat$, while for regression $\Yhat$ coincides with the score.
    }
    \label{fig:teaser_inference}
\vspace{-0.7cm}
\end{figure}

One of the most studied mitigation paradigms for group fairness is in-processing \citep{hort23}, which changes the training procedure, for instance, by adding a fairness regularizer or constraint to the training loss.
In addition to their good performance, in-processing methods are appealing thanks to their broad applicability: to induce a new notion of fairness, one simply needs to quantify fairness violations and use that as a regularizer or constraint to train a new prediction model.

In practice, however, retraining the prediction model to induce fairness is often infeasible. For example, complex prediction models are challenging to retrain when only limited computational resources are available \citep{cruz23}. To make matters worse, there is often no access to the parameters of the prediction model, which can only be queried to produce outputs (e.g. logits) for the provided inputs. For instance, when using one of the increasingly popular AutoML platforms \citep{he21}, one often has little to no control over the training objective, making it impossible to induce the desired fairness notion via in-processing. Additionally, in-processing might not be effective in a multi-component system, as prior research on compositional fairness \citep{dwork19, atwood23} shows that debiasing each component individually might not lead to a fair outcome.

In these situations, post-processing techniques offer the most compelling way to ensure fair predictions via a post-hoc module that transforms the outputs of a pre-trained base model (\Cref{fig:teaser_inference}).
However, current post-processing techniques are not nearly as broadly applicable as in-processing. The recent survey of \citet{hort23} found that over $200$ methods (out of $341$ surveyed approaches) used in-processing, covering a broad class of fairness definitions and problem settings. 
In contrast, the survey identifies only $56$ post-processing methods which are tailored to specific problem settings (e.g.\ problems with binary sensitive attributes \citep{pleiss17, kim20}) and specific fairness definitions (e.g.\ \citet{hardt16} focuses on equal opportunity/odds; \citet{xian23} considers statistical parity). The recent method of \citet{alghamdi22} significantly extends the applicability of post-processing but is still confined to problems with discrete sensitive attributes and notions of fairness based on a conditional mean score.

Furthermore, existing post-processing methods require that sensitive attributes are known at inference time, despite it being often untenable in practice \citep{veale17}. This complex landscape leads to the question:

\begin{figure}[t]
    \centering
    \includegraphics[width=\columnwidth]{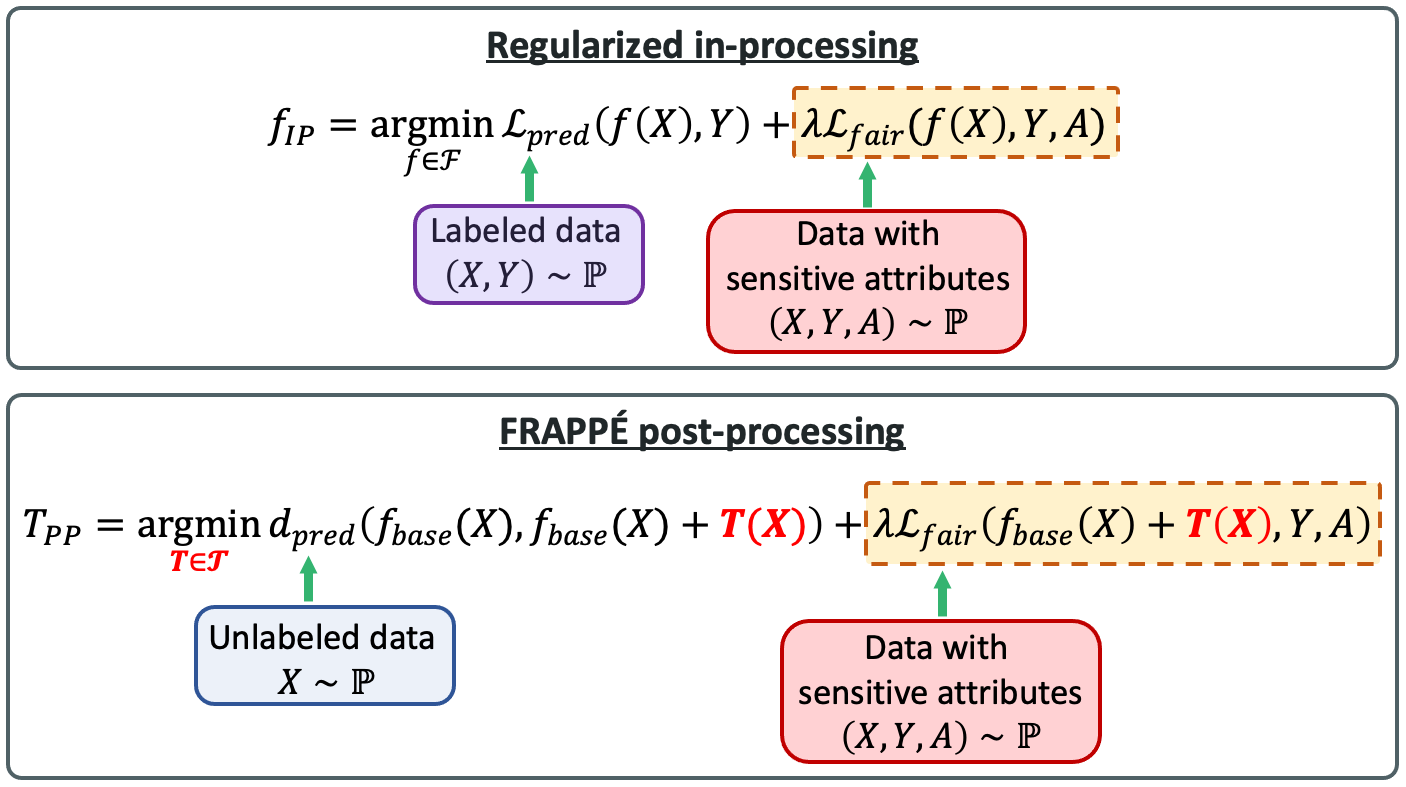}
    \caption{\textbf{$\ours$ and in-processing training objectives.} Unlike existing post-processing techniques, $\ours$ methods can be trained with \emph{any} in-processing fairness regularizer $\fairterm$ (orange box). In contrast to in-processing, $\ours$ only trains the post-hoc module $\mult(X)$ instead of the entire prediction model $\fct$. Loss terms are computed on data that is labeled, unlabeled or annotated with sensitive attributes, as indicated. $\regpp$ measures the difference between the outputs of the base and the fair models (see \Cref{sec:method}).}
    \label{fig:teaser_training}
\vspace{-0.4cm}
\end{figure}

\vspace{-0.3cm}
\begin{center}
\emph{Can we design a post-processing module to induce group fairness \\which satisfies the following desiderata?}
\end{center}
\vspace{-0.3cm}
\begin{enumerate}[label=D{\arabic*}.]
\vspace{-0.2cm}
    \item \emph{Works for any pre-trained models that output scores (e.g. logits, continuous labels).}
    \vspace{-0.1cm}
    \item \emph{Can trade off fairness and prediction error effectively for any quantifiable notion of fairness.}
    \vspace{-0.1cm}
    \item \emph{Does not require sensitive attributes at inference time.}
\end{enumerate}
\vspace{-0.2cm}

To answer this question, in this work we propose a \textit{Fairness Framework for Post-Processing Everything} (\textbf{$\ours$}) that turns any regularized in-processing method into a post-processing approach. As highlighted in \Cref{fig:teaser_training}, the resulting method is \textit{modular}, namely it decouples training the (unfair) base model from learning the post-hoc module $\mult(X)$. Importantly, $\ours$ methods are designed to allow training the post-hoc module using \textit{any} arbitrary fairness regularizer. 
Finally, the post-hoc module of $\ours$ methods models a function of covariates $X$, thus not requiring explicit knowledge of the sensitive attributes at inference time, as shown in \Cref{fig:teaser_inference}.

Our contributions are as follows:
\vspace{-0.2cm}
\begin{enumerate}[leftmargin=*]

\vspace{-0.2cm}
    \item Motivated by our theoretical result that establishes a connection between the in- and post-processing training objectives, we propose a novel framework to train a modular post-processing method using an in-processing fairness regularizer (\Cref{sec:method}). The procedure is designed to solve the limitations of prior in- and post-processing approaches captured by desiderata D1 -- D3.
    \vspace{-0.2cm}
    \item We complement the theoretical findings with extensive experiments (\Cref{sec:experiments}) 
    which show that $\ours$ methods do not degrade the fairness-error trade-off of their in-processing counterparts
    for several in-processing regularizers targeting different fairness definitions (MinDiff \citep{prost19} for equal opportunity, \citet{cho20} for statistical parity, \citet{mary19} for equal odds) and various commonly used datasets (Adult, COMPAS, HSLS, ENEM, Communities~\&~Crime).

    \vspace{-0.2cm}
    \item We demonstrate empirically the advantages of our framework's modular design.
    Unlike prior post-processing techniques, $\ours$ can induce \textit{any} notion of group fairness. Moreover, even when prior approaches are applicable, our experiments reveal that $\ours$ methods are on par or better compared to competitive post-processing techniques such as \citet{alghamdi22}. Finally, we provide evidence that modular $\ours$ methods can perform significantly better than their in-processing counterparts 
    on data with partial group labels, when sensitive annotations are scarce (\Cref{sec:novel_failure}).
\end{enumerate}

\vspace{-0.5cm}
\section{Problem setting}
\label{sec:pb_setting}

We consider prediction tasks where the goal is to predict labels $y \in \Yset$ (discrete in classification, or continuous for regression) from covariates $\xvec \in \Xset$ with low prediction error. A simple learning algorithm, Empirical Risk Minimization or ERM \citep{vapnik91}, that minimizes the prediction loss on an i.i.d.\ dataset is known to achieve great average-case error. However, this simple strategy does not guarantee good fairness \citep{menon18, chen18, zhao19, sagawa20, bardenhagen21, sanyal2022}. In applications where fairness is important, it is necessary to adjust the learning algorithm to promote greater equity.

\vspace{-0.3cm}
\paragraph{Group fairness.} A common fairness consideration is the model impact on sensitive groups. In the framework of group fairness, there exists a sensitive attribute $A$ (discrete or continuous) with respect to which we expect an algorithm to be equitable. Different flavors of group fairness are captured formally by different definitions, e.g.\ statistical parity (SP) \citep{calders09, dwork12}, equal opportunity (EqOpp), equalized odds (EqOdds) \citep{hardt16}. We refer to \citet{barocas23} and the respective prior works for details on these fairness definitions. Since fairness and predictive performance are often at odds \citep{menon18, chen18, zhao19}, the literature focuses on achieving a good trade-off. Two remarkably effective paradigms at reaching a good fairness-error trade-off in practice are in-processing and post-processing \citep{caton23, hort23}.\footnote{There also exist pre-processing approaches \citep{zemel13, madras18, lahoti19} that try to debias the data distribution. However, their performance in practice is usually significantly worse compared to in- and post-processing methods \citep{zehlike21, caton23, hort23}.}

\vspace{-0.3cm}
\paragraph{In-processing for group fairness.} Methods in this category seek to optimize a prediction loss (e.g.\ cross-entropy, mean squared error) while at the same time encouraging the prediction model to be fair. 
Regularized in-processing methods \citep{beutel2019, prost19, mary19, cho20, lowy2022} consider an optimization objective with a fairness violation penalty added to the prediction loss as a regularization term (\Cref{fig:teaser_training}).

\vspace{-0.3cm}
\paragraph{Post-processing for group fairness.} To induce fairness, post-processing techniques \citep{hardt16, kamiran18, nandy22, alghamdi22, cruz23} adjust the scores output by a pre-trained prediction model $\fctbase$. The current literature on post-processing methods for classification or regression focuses exclusively on \textit{group-dependent} transformations $\fctfair(\xvec) = T_A(\fctbase(\xvec))$, where $\fctbase$ and $\fctfair$ are the pre-trained and the fair models, respectively. The \emph{post-hoc transformation} (or \emph{post-hoc module}) $T_A$ is selected 
based on the value of the sensitive attribute $A$, from a set containing one learned transformation for each possible value of $A$.

\vspace{-0.2cm}
\section{Proposed framework}
\label{sec:method}

\vspace{-0.1cm}
In this section we introduce the $\ours$ framework that transforms a regularized in-processing method for group fairness into a post-processing one.\footnote{While we focus on \textit{regularized} in-processing objectives, our framework can easily be extended to constrained methods as well \citep{cotter19, chierichetti19, celis19}, as exemplified in~\Cref{fig:ip_vs_pp_reductions_main} for the method of \citet{agarwal18}.} Unlike prior post-processing approaches, instead of a group-dependent transformation that depends explicitly on the sensitive attribute, $\ours$ methods employ an additive term that is a function of all covariates $\xvec$ (the choice of the additive transformation is explained in \Cref{sec:method_description}):
\begin{flalign}
  \fctfair(\xvec) = \fctbase(\xvec) + \mult(\xvec).
\end{flalign} 
In what follows, we argue that $\ours$ methods are specifically designed to overcome the shortcomings of prior post-processing approaches captured in desiderata D1 -- D3 (and discussed in more detail in~\Cref{sec:related_work}), while also enjoying the advantages of a modular design (e.g.\ reduced computation time, no need to access training pipeline and data). 

\vspace{-0.2cm}
\subsection{Theoretical motivation: An equivalence between in- and post-processing for GLMs}
\label{sec:warmup}

We begin by motivating the proposed method by establishing a connection between a regularized in-processing objective and a bi-level optimization problem akin to post-processing.
To illustrate this intuition, we consider predictors that are generalized linear models (GLM) \citep{nelder72} and take the form $\fctglm(\xvec)=\glmlink(\thetavec^\top \xvec)$ for parameters $\thetavec \in \RR^D$ and a link function $\glmlink: \RR \rightarrow \RR$ (e.g.\ identity or sigmoid, for linear or logistic regression, respectively). Given datasets $\Dpred=\{(\xvec_i, y_i)\}_{i=1}^\npred$ and $\Dsens=\{(\xvec_j, y_j, a_j)\}_{i=1}^\nsens$ drawn i.i.d.\ from the same distribution\footnote{Often, in practice, $\Dpred$ and $\Dsens$ may even coincide. We leave a discussion of the implications of a potential distribution shift between $\Dpred$ and $\Dsens$ as future work.}, we analyze generic regularized optimization problems of the form 
\begin{flalign}
    \optip(\thetavec; \lambda) &= \frac{1}{\npred} \sum_{(\xvec, y) \in \Dpred} \predloss(\xvec, y; \thetavec) \notag \\
    &+ \lambda \regterm(\thetavec; \Dsens),
    \label{eq:optipglm}
\end{flalign}
where $\predloss$ is the prediction loss and $\regterm$ is an arbitrary regularizer capturing a fairness violation penalty. We consider loss functions that can be written as 
\begin{equation}
\glmloss(\xvec, y; \thetavec)=\glmpart(\thetavec) - \thetavec^\top\glmtrans(\xvec, y),
\label{eq:glmloss}
\end{equation}
with $\glmpart: \RR^D \rightarrow \RR$ strictly convex. The function $\glmtrans$ is a sufficient statistic with respect to $\thetavec$ and $\glmpart$ can be viewed as a partition function \citep{wainwright08}. Standard loss functions for linear regression (e.g.\ mean squared error) or classification (e.g.\ logistic loss) follow this pattern.

For every optimization problem that takes the form in~\Cref{eq:optipglm}, consider the following corresponding bi-level optimization problem: 
\begin{align}
&\optpp(\thetavec; \lambda)= \breg(\thetavec, \thetabase) + \lambda \regterm(\thetavec; \Dsens), \notag \\&\text{with } \thetabase := \arg\min_{\thetavec} \frac{1}{\npred} \sum_{(\xvec, y) \in \Dpred} \predloss(\xvec, y; \thetavec),
\label{eq:optppglm}
\end{align}
where we denote by $\breg(\thetavec, \phivec):=\glmpart(\thetavec) - \glmpart(\phivec) - \nabla \glmpart(\phivec)^\top (\thetavec - \phivec)$ the Bregman divergence \citep{bregman67} of the partition function $\glmpart(\thetavec)$. Intuitively, the first term encourages that the outputs produced by the GLMs determined by $\thetavec$ and $\thetabase$ are similar.

\paragraph{Example for linear regression.} For instance, for linear regression and the mean squared error (MSE), the link function is the identity $\glmlink(z)=z$, and the loss function can be written as $\ell_\text{MSE}(\xvec, y; \thetavec)=\|\thetavec^\top \xvec - y\|^2=\thetavec^\top \xvec\xvec^\top \thetavec - 2y\thetavec^\top\xvec+c$, for a constant $c\ge0$. It follows that the Bregman divergence of $\glmpart(\thetavec)=\thetavec^\top \xvec\xvec^\top \thetavec$ is given by $\breg(\thetavec, \phivec)=(\thetavec - \phivec)^\top \xvec \xvec^\top (\thetavec - \phivec) = \|\thetavec^\top \xvec - \phivec^\top\xvec\|^2$, namely the MSE between the outputs of the models parameterized by $\thetavec$ and $\phivec$, for arbitrary $\thetavec, \phivec \in \RR^D$. 

\paragraph{Equivalence between $\optip$ and $\optpp$.} The following result establishes a connection between the optimization objectives $\optip$ (\Cref{eq:optipglm}) and $\optpp$ (\Cref{eq:optppglm}) introduced above (the proof is provided in~\Cref{appendix:ip_vs_pp_proof}).

\begin{restatable}{proposition}{ipvspp}
    Consider the optimization objectives introduced in~\Cref{eq:optipglm,eq:optppglm}. There exists a constant $C\in \RR$ such that for any $\thetavec \in \RR^D$ and $\lambda \ge 0$ we have \begin{equation}
        \optpp(\thetavec; \lambda) = \optip(\thetavec; \lambda)+C.
        \label{eq:ip_pp_equivalence}
    \end{equation}
 \label{prop:ip_vs_pp}
\end{restatable}
\vspace{-0.7cm}
It follows from~\Cref{prop:ip_vs_pp} that minimizing $\optip$ and $\optpp$ leads to the same solution, for any choice of the regularizer $\regterm$ and the regularization strength $\lambda$. In the context of fairness, this result implies that sweeping over the hyperparameter $\lambda$ gives rise to identical fairness-error Pareto frontiers between any regularized in-processing and the corresponding post-processing method trained with $\optpp$. Moreover, since the two optimization problems are identical up to a universal constant, properties established for an in-processing objective (e.g.\ smoothness \citep{cho20}, convergence rate of gradient descent \citep{lowy2022}, etc) carry over intrinsically to the $\optpp$ counterpart. 

\begin{figure*}[t]
\centering
\vspace{-0.3cm}
\begin{subfigure}[t]{0.23\textwidth}
\captionsetup{justification=centering}
    \includegraphics[height=3.55cm]{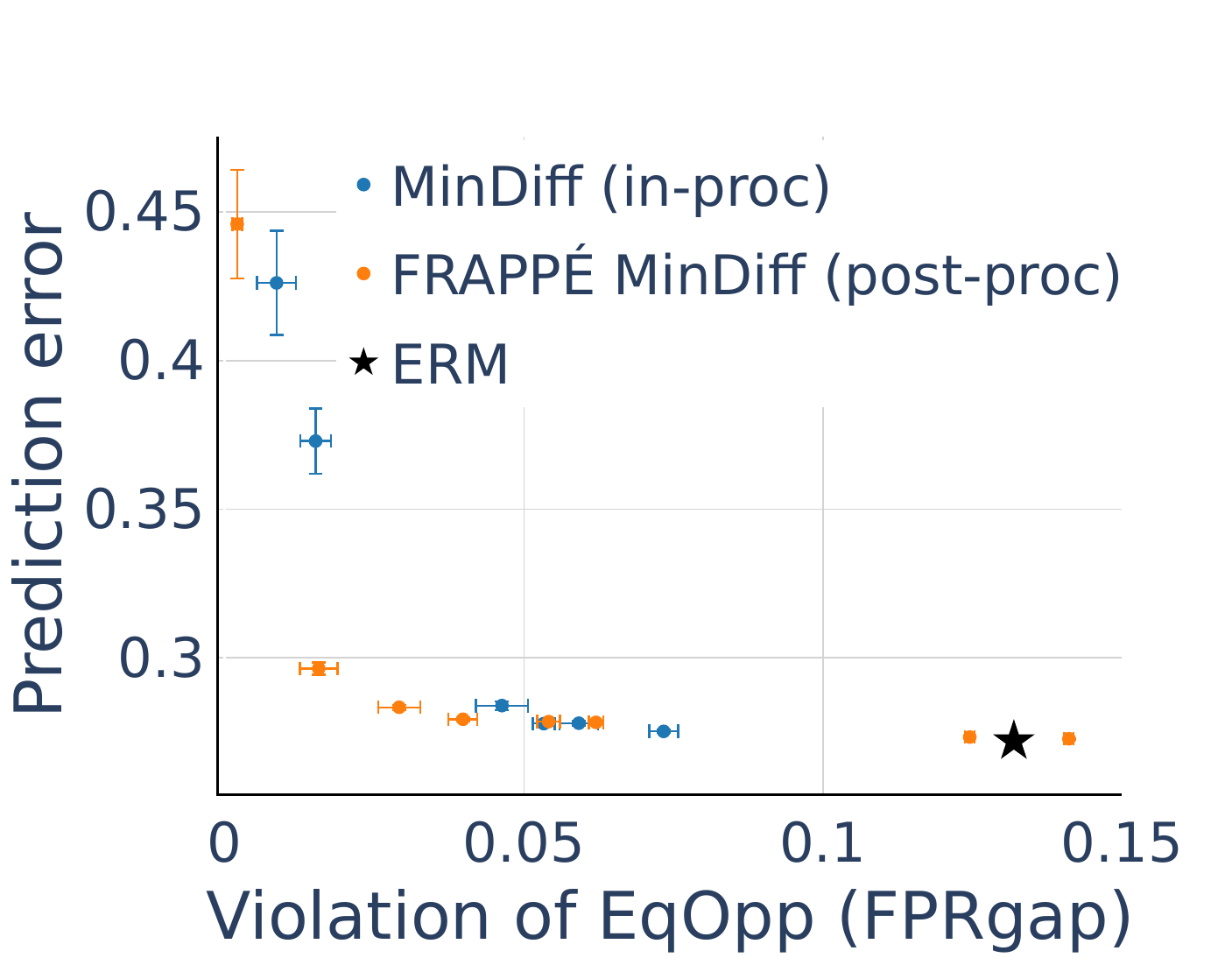}
\caption{\scriptsize $\optip=$ \textit{MinDiff}.\\Dataset: HSLS.}
\end{subfigure}
\hfill
\begin{subfigure}[t]{0.23\textwidth}
\captionsetup{justification=centering}
    \includegraphics[height=3.55cm]{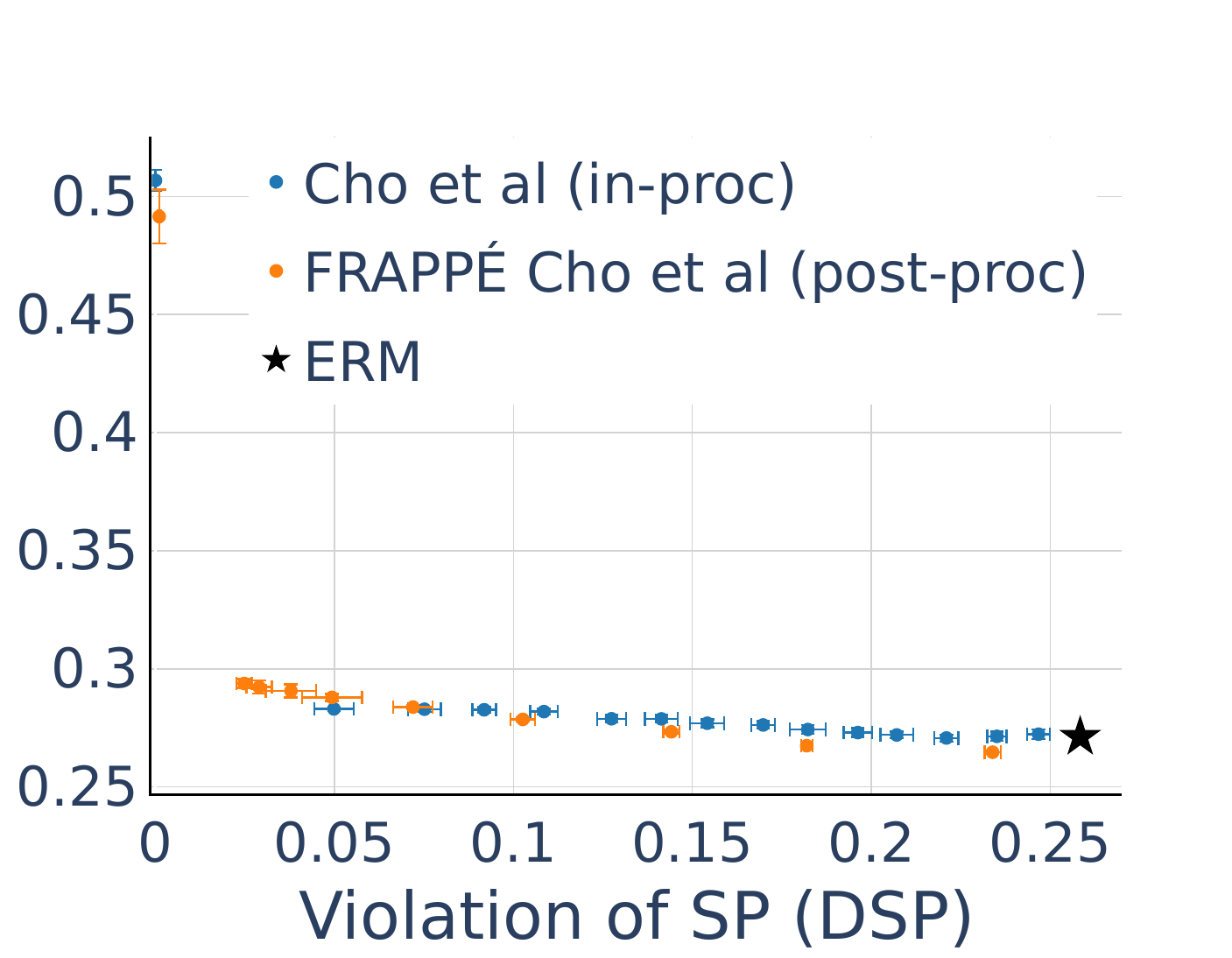}
\caption{\scriptsize $\optip=$ \citet{cho20}.\\Dataset: HSLS.}
\label{fig:ip_vs_pp_cho_main}
\end{subfigure}
\hfill
\begin{subfigure}[t]{0.23\textwidth}
\captionsetup{justification=centering}
    \includegraphics[height=3.55cm]{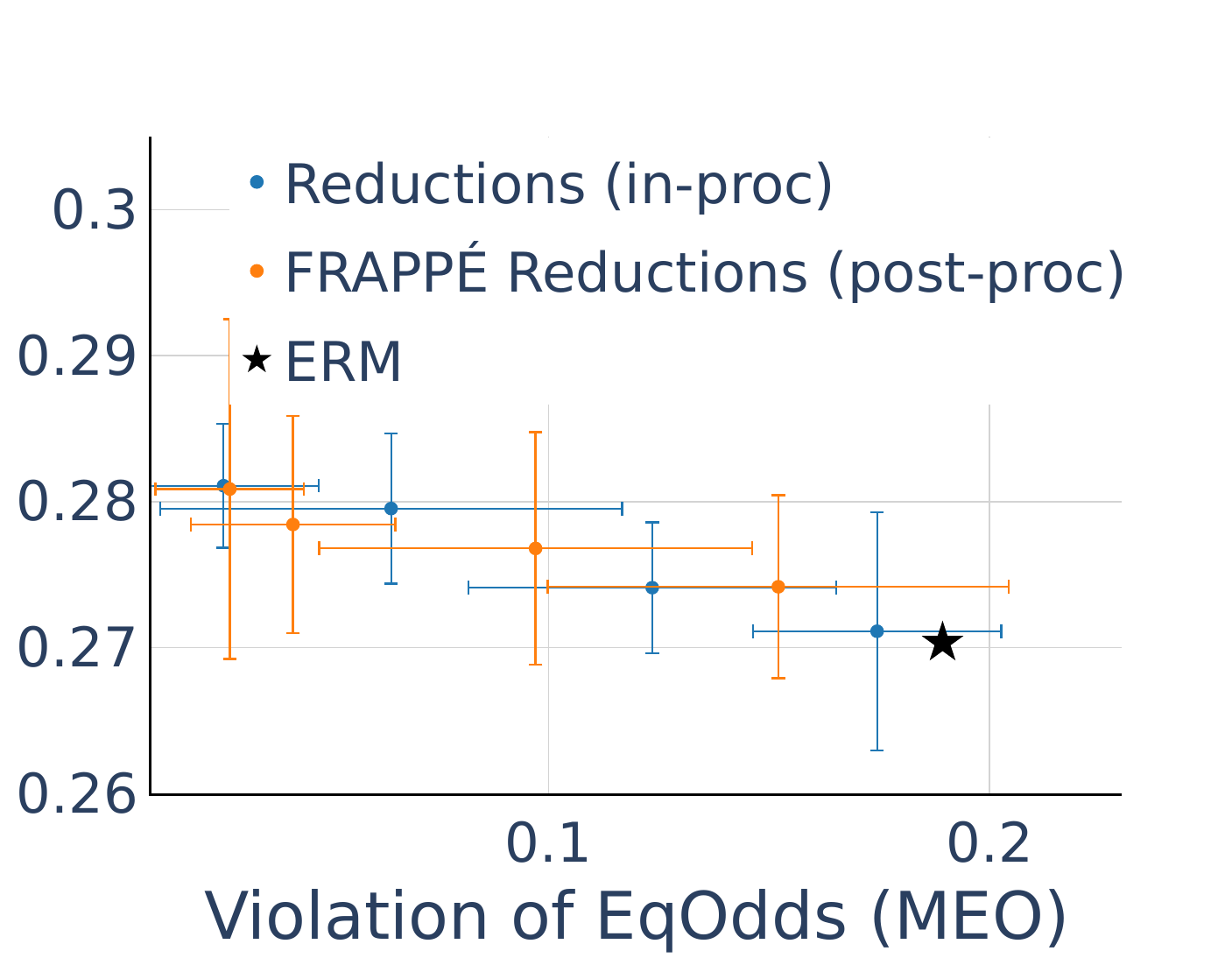}
\caption{\scriptsize $\optip=$ \textit{Reductions}.\\Dataset: HSLS.}
\label{fig:ip_vs_pp_reductions_main}
\end{subfigure}
\hfill
\begin{subfigure}[t]{0.24\textwidth}
\captionsetup{justification=centering}
    \includegraphics[height=3.55cm]{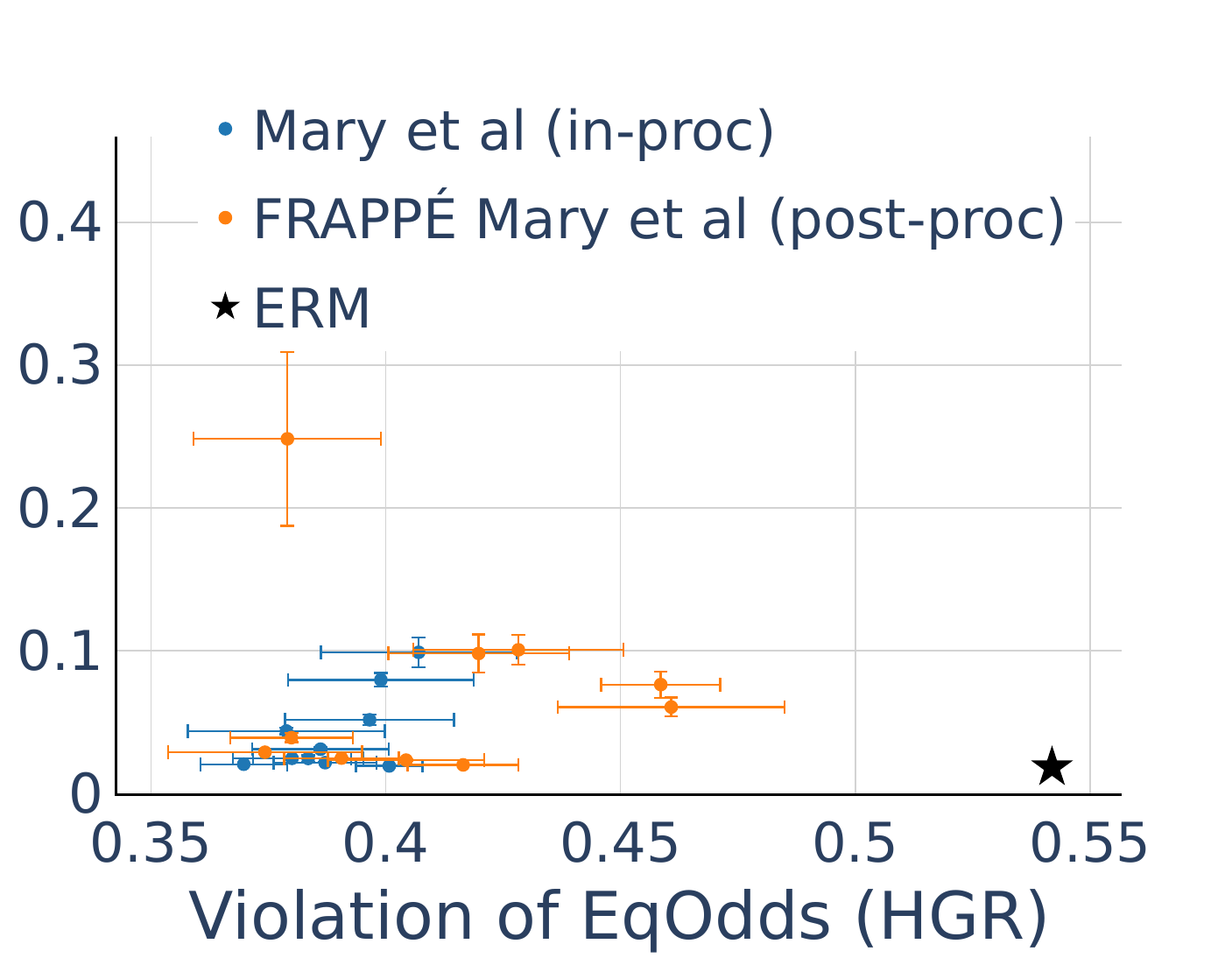}
\caption{\scriptsize $\optip=$ \citet{mary19}.\\Dataset: Communities~\&~Crime.}
\label{fig:mary}
\end{subfigure}

\vspace{-0.1cm}
\caption{Inducing three different definition of fairness (EqOpp, SP, and EqOdds) using in-processing methods and their $\ours$ post-processing variant leads to similar Pareto frontiers. Thanks to their modular design, $\ours$ methods only need to retrain the post-hoc transformation $\mult(x)$, instead of the entire prediction model.  \Cref{appendix:ip_vs_pp} shows similar results on the Adult, COMPAS and ENEM datasets.
Notably, $\ours$ \citet{mary19} is the first post-processing method that can operate on data with continuous sensitive attributes, such as Communities~\&~Crime.}
\label{fig:ip_vs_pp}
\vspace{-0.4cm}
\end{figure*}

\subsection{Proposed post-processing framework}
\label{sec:method_description}

We now describe how to turn the insights revealed by~\Cref{prop:ip_vs_pp} into a practical framework for training a post-processing method for group fairness with an in-processing objective. Moreover, we extend the intuition developed for GLMs to more generic function classes.

First, as mentioned before, numerous in-processing fairness mitigations \citep{prost19,mary19,cho20,lowy2022} consider optimization objectives that can be written like $\optip$, where the regularizer $\regterm$ captures a fairness violation penalty. 
On the other hand, the bi-level problem $\optpp$ can be used to train a post-processing method, where the inner optimization corresponds to obtaining the pre-trained model parameters $\thetabase$. While~\Cref{prop:ip_vs_pp} only holds for GLMs, both the $\optip$ objective and $\optpp$ can be considered in the context of training more generic model classes (e.g.\ neural networks).

Furthermore, recall that post-processing methods only modify the outputs of a pre-trained model, instead of training a model from scratch. The GLM scenario introduced above suggests to choose this post-hoc transformation to have the following additive form: $\fctfair(\xvec)=\glmlink((\thetabase + \thetamult)^\top\xvec)$. 
In particular, depending on the link function $\glmlink$, the post-hoc transformation can be additive in output space (e.g.\ for linear regression) or logit space (e.g.\ for logistic regression).
More generally, we can use the intuition developed for GLMs to propose the following post-processing transformation of a pre-trained model for regression and classification:
\begin{flalign}
  \fctfair(\xvec) = \fctbase(\xvec) + \mult(\xvec), \label{eq:pp_reg}
\end{flalign}   
where, for classification, $\fctbase(\xvec)$ and $\fctfair(\xvec)$ produce vectors of unnormalized logits. 
Importantly, $\ours$ methods only train the fairness correction $\mult(\xvec)$ which is often significantly less complex than the prediction model $\fctbase(\xvec)$, thus decreasing training time compared to in-processing. 

In conclusion, for an arbitrary in-processing method that solves a regularized objective like~\Cref{eq:optipglm}, our framework constructs the following optimization problem:
\begin{flalign}
    \mult:=\arg\min_{T} \hspace{0.2em}& \frac{1}{|\Dpp|} \sum_{\xvec \in \Dpp} \regpp((\fctbase + T)(\xvec); \fctbase(\xvec))\notag\\
    &+ \lambda \regterm(\fctbase+T; \Dsens), \label{eq:optpp}
\end{flalign}
where the two terms are computed on datasets $\Dsens =\{(\xvec_i, y_i, a_i)\}_{i=1}^{\nsens}$ and $\Dpp =\{\xvec_i\}_{i=1}^{\npp}$ drawn i.i.d.\ from the same distribution. Here $\regterm$ is the in-processing fairness regularizer, and $\regpp$ is the Bregman divergence in $\optpp$ or some other notion of discrepancy between the outputs of the $\fctfair$ and $\fctbase$ models (e.g.\ KL divergence for classification).

Unlike in-processing, $\ours$ is modular and can find a different fairness-error trade-off or mitigate fairness with respect to a different definition (e.g.\ SP, EqOdds, EqOpp) by just retraining the simple additive term $\mult$, instead of always retraining a new model $\fctfair$ from scratch (i.e.\ only modules in the green boxes in \Cref{fig:teaser_inference} need to be retrained).

Finally, minimizing the $\regpp$ term does not require labeled training data, which makes this procedure suitable when either the labels $Y$ or the sensitive attributes $A$ are difficult to collect \citep{awasthi2021evaluating, Prost2021}. This observation is particularly important for mitigating the
limitations of in-processing on data with partial group labels, when $|\Dsens|$ is small, as discussed in~\Cref{sec:novel_failure}.

\vspace{-0.2cm}
\subsection{Connection to related prior works}

The objective in~\Cref{eq:optpp} can be seen as a two-step boosting algorithm \citep{schapire90}, where the second step corrects the unfairness of the model $\fctbase$ obtained after the first step, similar to \citet{liu21, bardenhagen21}. 
This formulation of post-processing is also related to post-hoc methods for uncertainty calibration \citep{pleiss17, kumar19}, as well as techniques for \emph{model reprogramming}, such as \citet{zhang2023} (see~\Cref{appendix:reprogramming} for more details).
The additive post-hoc transformation that we employ has been considered in the past by works that focus exclusively on disparate performance and use a group-dependent \emph{logit adjustment} to improve worst-group error \citep{khan18, cao19, menon21logit}.
Furthermore, the idea of reusing the in-processing optimization objective to train a post-processing method is reminiscent of recent works on group distributionally robust optimization \citep{sagawa20}, which show that last layer retraining is equivalent to training the entire neural network \citep{menon21gdro, shi2023, labonte23}.
Finally, similar to our approach, unlabeled data has also been used to improve the trade-off between standard and robust error in the context of adversarial robustness \citep{carmon19, raghunathan20}.

\vspace{-0.2cm}
\section{Experimental setup}
\label{sec:evaluation_details}

We compare different fairness mitigation techniques by inspecting their fairness-error Pareto frontiers, obtained after varying the $\lambda$ hyperparameter in~\Cref{eq:optipglm,eq:optpp}. We quantify fairness violations with metrics tailored to each specific fairness definition: the FPR gap for EqOpp, the difference in SP for SP, and the mean equalized odds violation for EqOdds. For EqOdds, we also employ the same evaluation metric as \citet{mary19}, namely
HGR$_\infty(f(X), A|Y) \in[0,1]$, which takes small values when predictions $f(X)$ and sensitive attributes $A$ are conditionally independent given true labels $Y$ (see~\Cref{appendix:fairness_definitions} for details). For all metrics we report the mean and standard error computed over $10$ runs with different random seeds.

\paragraph{In-processing baselines.} To induce these notions of fairness, we use the $\ours$ framework to obtain a post-processing counterpart for several in-processing methods. We aim to cover a diverse set of regularized in-processing methods that have been demonstrated to perform well in recent thorough experimental studies \citep{cho20,jung2021,lowy2022,alghamdi22}. We identify the following as particularly competitive in-processing methods:
i)~for EqOpp, we consider MinDiff \citep{beutel2019, prost19}; ii)~for SP, the method of \citet{cho20}; and iii)~for EqOdds, \emph{Reductions} \citep{agarwal18} and the method of \citet{mary19}. We consider the KL divergence and the MSE as $\regpp$ in \Cref{eq:optpp}, for classification and regression, respectively.
\vspace{-0.1cm}

\begin{figure*}[t]
\centering
\vspace{-0.3cm}
\begin{subfigure}[t]{0.24\textwidth}
    \includegraphics[height=4cm]{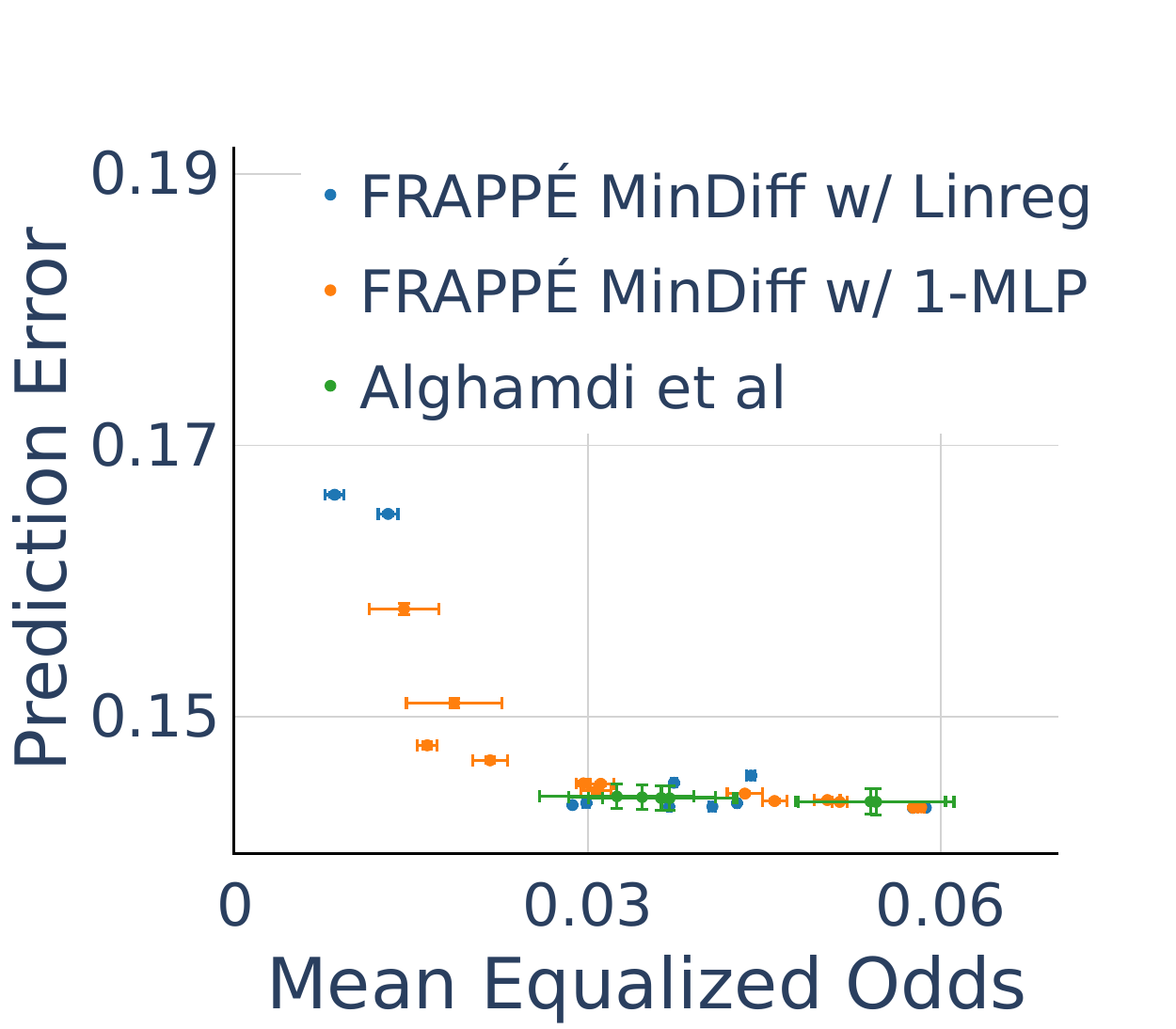}
    \caption{Adult}
\end{subfigure}
\hfill
\begin{subfigure}[t]{0.24\textwidth}
    \includegraphics[height=4cm]{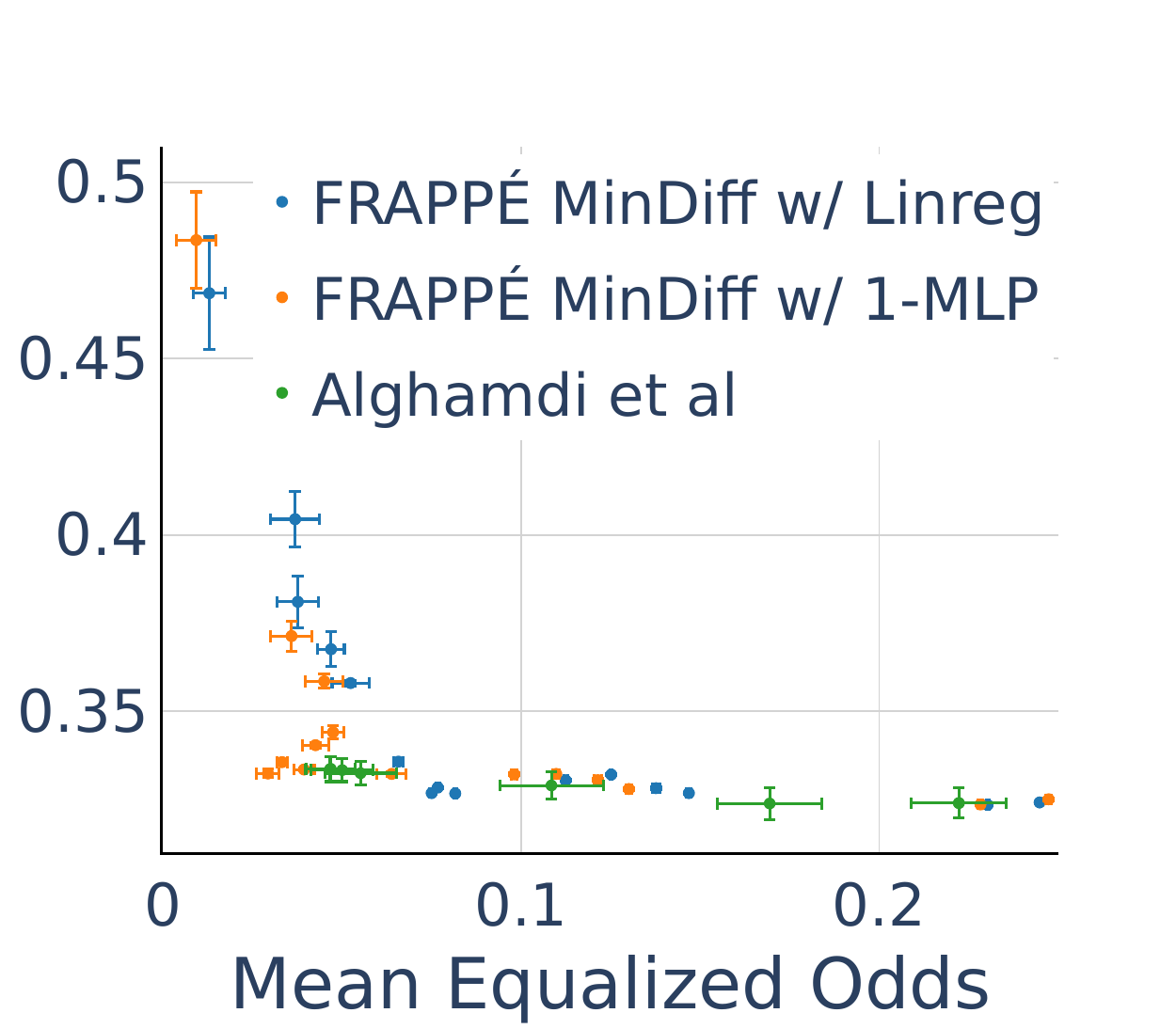}
    \caption{COMPAS}
    \label{fig:comparison_alghamdi_compas}
\end{subfigure}
\hfill
\begin{subfigure}[t]{0.24\textwidth}
    \includegraphics[height=4cm]{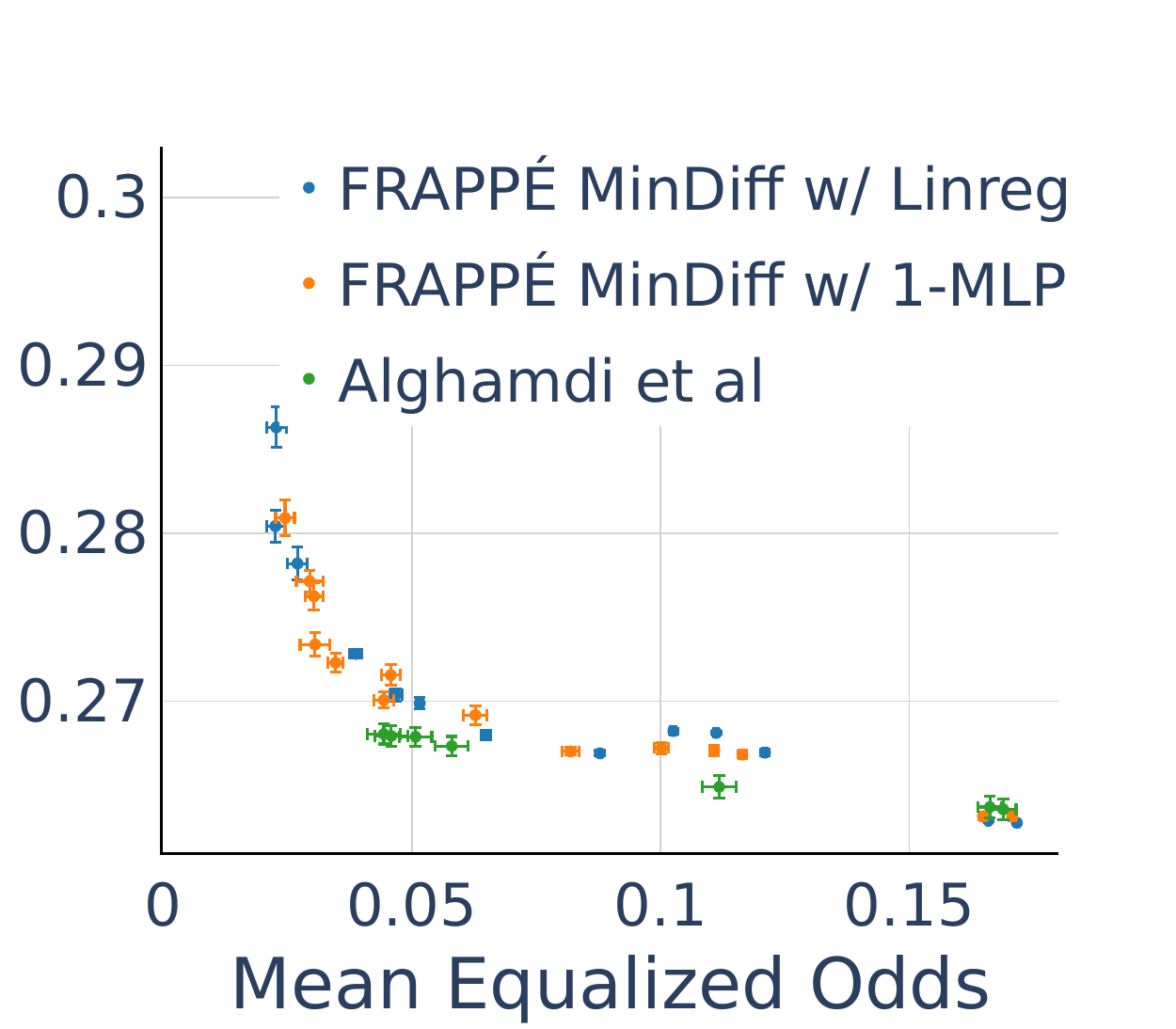}
    \caption{HSLS}
\end{subfigure}
\hfill
\begin{subfigure}[t]{0.24\textwidth}
    \includegraphics[height=4cm]{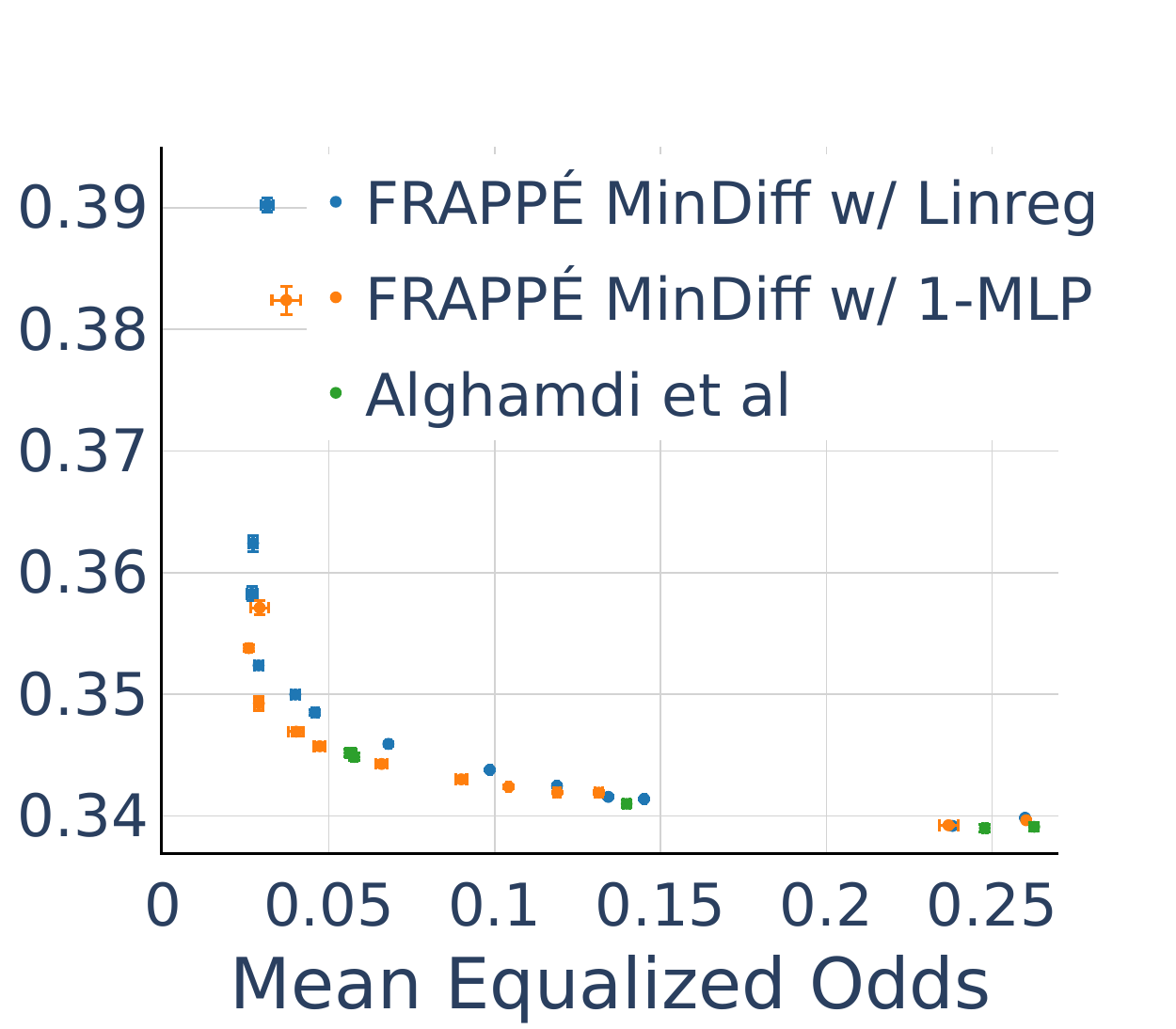}
    \caption{ENEM}
\end{subfigure}
\caption{Comparison between $\ours$ MinDiff for EqOdds and the best-performing post-processing method \citep{alghamdi22}, for random forest pre-trained models. See~\Cref{fig:app_more_baselines} for a comparison with more post-processing baselines. While in-processing MinDiff cannot be used with non-gradient based models, $\ours$ MinDiff performs on-par or better than competitive post-processing approaches such as \citet{alghamdi22}, even when the post-hoc transformation is as simple as linear regression or a 1-MLP.}
\label{fig:comparison_alghamdi}
\vspace{-0.2cm}
\end{figure*}

\paragraph{Datasets.} We conduct experiments on standard datasets for assessing fairness mitigation techniques, namely Adult \citep{becker96} and COMPAS \citep{angwin16}, as well as two recently proposed datasets: the high-school longitudinal study (HSLS) dataset \citep{jeong22}, and ENEM \citep{alghamdi22}. We also evaluate $\ours$ on data with continuous sensitive attributes (i.e.\ the Communities~\&~Crime dataset \citep{redmond09}), a setting where prior post-processing works cannot be applied. \Cref{appendix:datasets} provides more details about the datasets.

\paragraph{Prediction models.} We consider a broad set of model classes, varying from multi-layer perceptrons (MLPs) or gradient-boosted machines (GBMs), to non-gradient based models such as random forests (RFs). Both the in-processing and the pre-trained model are selected from these model classes. For $\ours$, the post-hoc module $\mult(X)$ is a much less complex model, e.g.\ linear regression.

Finally, the inherent fairness-error trade-off can pose serious challenges for hyperparameter tuning \citep{cruz2021}. We adopt the standard practice in the literature, and select essential hyperparameters such as the learning rate so as to minimize prediction error on a holdout validation set, for all the baselines in our experiments. We defer further experimental details to~\Cref{appendix:exp_details}.
\paragraph{Inducing fairness with partial group labels.}
Often, in practice, it is challenging to collect data with sensitive attributes, e.g.\ users of an online service may not be willing to disclose their gender, ethnicity etc \citep{hashimoto18, coston19, lahoti20, liu21, bardenhagen21, awasthi2021evaluating, Prost2021}. Therefore, the size of $\Dsens$ used to train fairness mitigations is significantly reduced, while $\Dpred$ may still be fairly large. In \Cref{sec:novel_failure} we present experiments on data with partial group labels, in which $\Dpred$ and $\Dpp$ from \Cref{eq:optpp,eq:optipglm} consist of all the available training data, while $\Dsens$ contains only a fraction of this data, annotated with sensitive attributes. 

\vspace{-0.1cm}
\section{Experimental results}
\label{sec:experiments}

In this section, we show empirically that $\ours$ methods satisfy the desiderata from \Cref{sec:intro}. More specifically, we show in extensive experiments that $\ours$ methods preserve the competitive fairness-error trade-offs achieved with in-processing techniques, for various notions of fairness (\underline{\textbf{$\texttt{D2}$}}), while enjoying the advantages of a post-processing method and being entirely agnostic to the prediction model class (\underline{\textbf{$\texttt{D1}$}}). Moreover, $\ours$ methods perform on par or better than existing post-processing approaches, without requiring that sensitive attributes be known at inference time (\underline{\textbf{$\texttt{D3}$}}). Finally, the $\ours$ framework helps to make post-processing fairness mitigations more broadly applicable, providing, for instance, the first post-processing method for data with \emph{continuous} sensitive attributes.

\vspace{-0.2cm}
\subsection{Can $\ours$ perform as well as in-processing?}
\label{sec:exp_ip_vs_pp}

The data processing inequality \citep{cover91} suggests that it may be challenging for post-processing approaches to match the performance of in-processing methods. In this section, we demonstrate experimentally that $\ours$ methods preserve the good fairness-error trade-offs achieved by their in-processing counterparts, for several different notions of fairness and in-processing techniques.
As suggested by the intuition in~\Cref{sec:warmup}, \Cref{fig:ip_vs_pp} confirms that for several fairness definitions it is indeed possible to match the Pareto frontiers of in-processing methods using a modular $\ours$ variant. We observe the same equivalence on all datasets we considered (\Cref{appendix:ip_vs_pp}). While the theoretical result assumes the same function class for the pre-trained model $\fctbase$ and the post-hoc module $\mult$, 
these experiments suggest that, in practice, the complexity of $\mult$ (i.e.\ linear model) can be significantly smaller than $\fctbase$ (i.e.\ 3-layer MLP). Importantly, the $\ours$ variant of \citet{mary19} constitutes the first post-processing approach that can be utilized when the sensitive attributes are continuous. We note that the large error bars in~\Cref{fig:mary} are due to the challenges of optimizing the loss of \citet{mary19}, discussed at length in \citet{lowy2022}.

\begin{figure*}[t]
\centering
\includegraphics[width=0.8\textwidth]{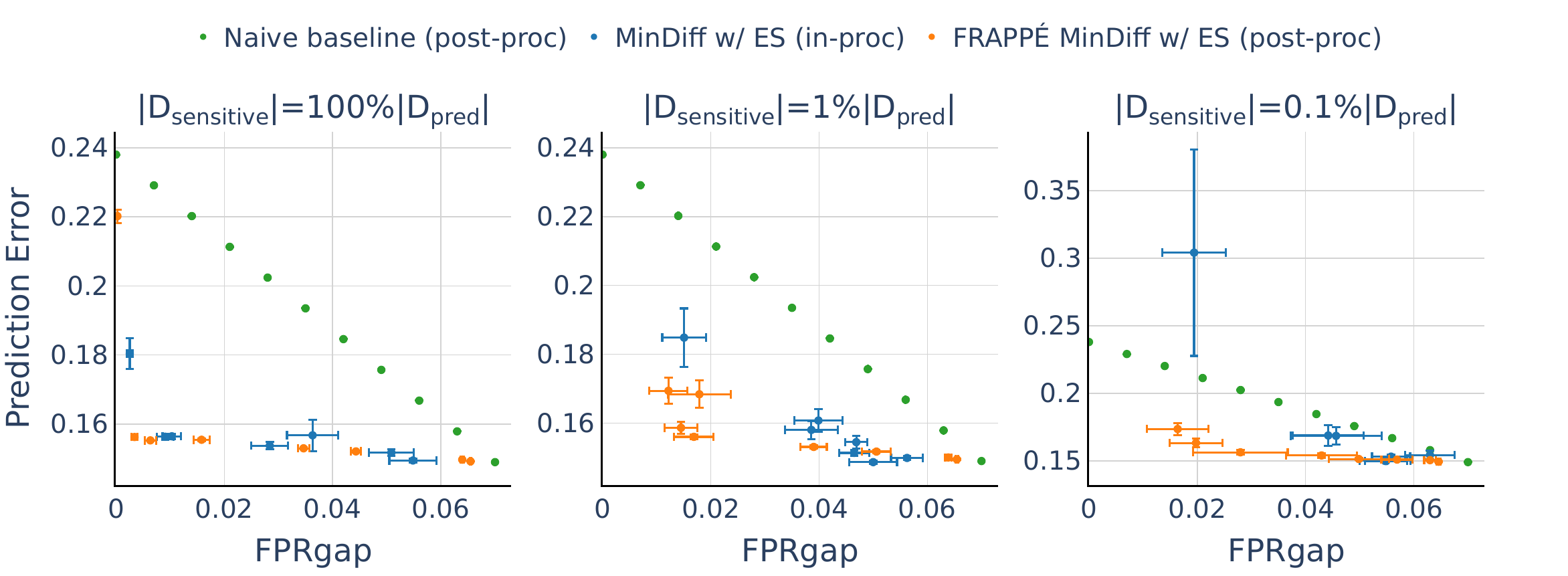}
\caption{In-processing MinDiff and $\ours$ MinDiff with partial group labels on the Adult dataset with optimal early-stopping (ES) regularization. Our post-processing algorithm continues to perform well even in the extreme case where in-processing cannot outperform the trivial baseline described in~\Cref{sec:novel_failure}.}
\label{fig:low_sample}
\vspace{-0.1cm}
\end{figure*}

\vspace{-0.2cm}
\paragraph{Computation cost.} Since $\ours$ methods only train the post-hoc module instead of the entire prediction model, they require only a fraction of the computational resources necessary for in-processing methods. Indeed, it takes $85.4$ and $113.4$ minutes to obtain the Pareto frontiers of MinDiff and \citet{cho20}, respectively, on the Adult dataset. In contrast, their $\ours$ variants only require $11.7$ and $14.8$ minutes, respectively (including the cost of training the base model), which is in line with the computation times we obtain on the same hardware with prior post-processing approaches (e.g.\ \citet{alghamdi22}). Moreover, the modular nature of $\ours$ significantly reduces the cost of changing the desired notion of fairness, set of sensitive attributes, or fairness-error trade-off. For instance, a total of two base models (one per dataset) has been used for \textit{all} the $\ours$ runs presented in \Cref{fig:ip_vs_pp}, while in-processing methods require retraining the entire prediction model from scratch repeatedly for each point shown in the figures.

\vspace{-0.1cm}
\subsection{$\ours$ compared to other post-processing}
\label{sec:pp_sota}
Post-processing techniques such as $\ours$ methods only train the post-hoc module instead of the entire prediction model, and hence, can mitigate group fairness for \textit{any} class of prediction models. 
However, unlike $\ours$, which can induce \emph{any} quantifiable notion of fairness, prior post-processing algorithms are only applicable for specific fairness definitions and problem settings. In this section we only focus on settings that are compatible with competitive prior post-processing approaches such as FairProjection \citep{alghamdi22}, and show that $\ours$ methods perform on-par or better. To illustrate the versatility of $\ours$, we consider non-gradient based models 
(e.g.\ random forests (RF)),
for which in-processing techniques such as MinDiff cannot be applied directly.
In \Cref{fig:comparison_alghamdi} we compare the Pareto frontiers obtained with $\ours$ MinDiff to the recent method of \citet{alghamdi22}, which significantly outperforms prior post-processing approaches.
$\ours$ MinDiff can sometimes surpass FairProjection considerably.
More specifically, compared to FairProjection, our approach can reduce the MEO by $53\%$, $37\%$, $32\%$ and $50\%$ on Adult, COMPAS, HSLS and ENEM, respectively, without sacrificing more than $2\%$ of the prediction error.
In \Cref{appendix:more_baselines} we compare $\ours$ MinDiff with more baselines that perform worse than FairProjection, and present results with several other pre-trained model classes, i.e.\ logistic regression and GBMs.

\vspace{-0.2cm}
\subsection{Modular methods on data with partial group labels}
\label{sec:novel_failure}

In this section we demonstrate how $\ours$ methods can alleviate the challenges faced by in-processing, when only training data with partial group labels is available. We 
argue that the good performance of $\ours$ in this regime is due to its modular design and present proof-of-concept experiments on the Adult dataset. Experiment details are deferred to \Cref{appendix:exp_details}.

\vspace{-0.2cm}
\paragraph{Limitations of in-processing with partial group labels.}
It has been observed recently that in-processing methods tend to perform poorly when only partial group labels are available for training \citep{jung2021, lokhande2022, nam2022, sohoni2022, zhang2023}.
Our experiments corroborate these findings.
In particular, we observe that minimizing the objective in~\Cref{eq:optipglm} can lead to overfitting the fairness regularizer term $\fairterm$, as shown in \Cref{fig:app_overfitting} in \Cref{appendix:overfitting}. 
Strong regularization (e.g.\ early-stopping \citep{caruana00}) can prevent overfitting, but it may induce unnecessary underfitting of the prediction loss $\predloss$ in~\Cref{eq:optipglm}, thus hurting the fairness-error trade-off. Indeed, \Cref{fig:low_sample} reveals that the performance of in-processing MinDiff deteriorates significantly on data with partial group labels.  
To show how challenging this setting is, we also present, for reference, the performance of a naive post-processing baseline that simply outputs the same prediction as the pre-trained model with probability $p$, and outputs the more favorable outcome with probability $1-p$. Varying the probability $p$ interpolates between prioritizing prediction error (for $p=1$) or fairness (for $p=0$). Even though this baseline is clearly inferior to in-processing when data is plentiful (\Cref{fig:low_sample} Left), when only partial group labels are available (\Cref{fig:low_sample} Right), in-processing struggles to surpass this naive approach. 

\vspace{-0.2cm}
\paragraph{$\ours$ methods with partial group labels.}
In-processing approaches train a single model to simultaneously minimize both the prediction loss and the fairness regularizer, and hence, finding the right balance between under- and overfitting can be challenging.
In contrast, the modular $\ours$ methods disentangle the training of the prediction model from the fairness mitigation, and hence, allow for finer-grained control of under- and overfitting.

Indeed, \Cref{fig:low_sample} shows a significant gap in performance between MinDiff and its $\ours$ variant when sensitive annotations are scarce, despite their similar performance when $\Dsens$ consists of all training data.
More specifically, both in- and post-processing achieve similar low values of the FPR gap, but only $\ours$ can maintain a good prediction error in addition to good fairness.
Futhermore, we show in~\Cref{appendix:frappe_without_es} that $\ours$ does not require early-stopping to outperform (early-stopped) MinDiff, thus eliminating
an important hyperparameter.

\begin{figure*}[t]
\centering
\includegraphics[width=0.65\textwidth]{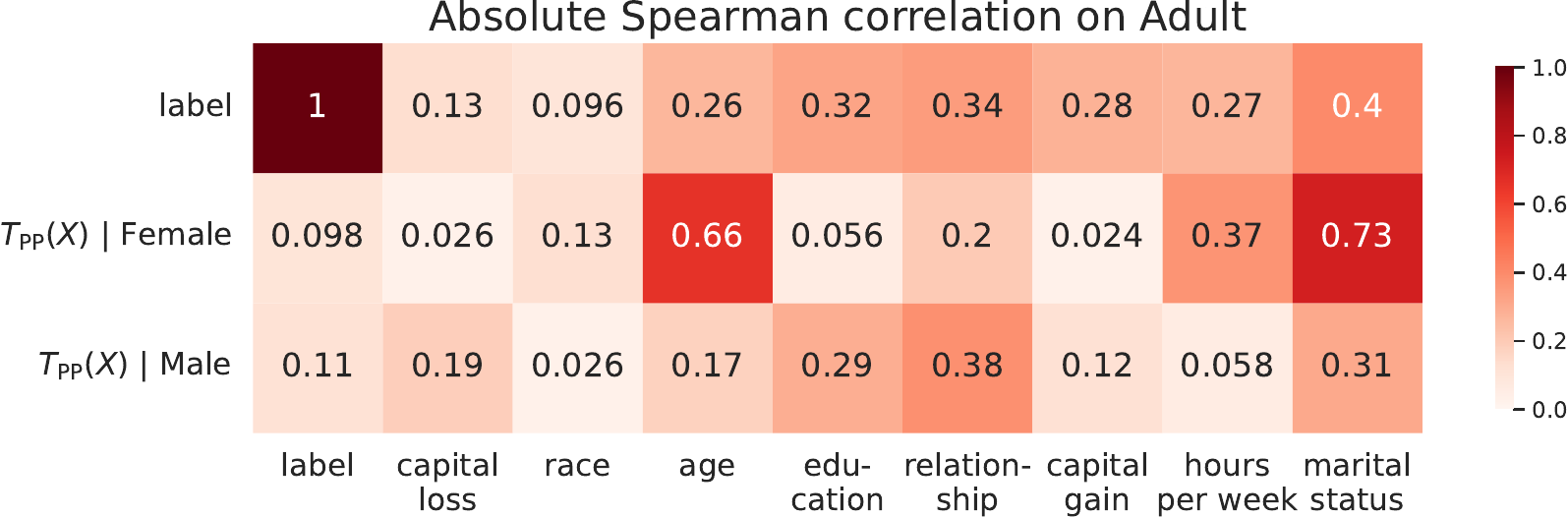}
\caption{The post-hoc transformation $\mult(X)$ is highly correlated with features that are predictive of the label (e.g.\ marital status, relationship), conditioned on gender.}
\label{fig:analysis_main}
\vspace{-0.2in}
\end{figure*}

\subsection{Analysis: What does the post-hoc module capture?}
\label{sec:analysis_of_t}

We now provide insights about the information captured in the learned post-hoc module $\mult(X)$. To this end, we use the absolute value of the Spearman's coefficient to measure the statistical correlation between the values of $\mult(X)$ and each of the input features, conditioned on the sensitive attribute (i.e.\ gender). For visualization purposes, we focus on datasets with a small number of covariates, i.e.\ Adult and COMPAS (we defer results on COMPAS to~\Cref{appendix:transformation_analysis}).

On the one hand, $\mult(X)$ is expected to be correlated with the sensitive attribute (as illustrated in \Cref{appendix:transformation_analysis}), since it is  trained specifically to improve fairness with respect to that sensitive attribute.
On the other hand, the $\ours$ post-hoc module is allowed to depend on all the covariates, not only the sensitive attribute. Therefore, alongside information about gender, $\mult(X)$ can also embed features that help to achieve a better fairness-error trade-off. Indeed, \Cref{fig:analysis_main} suggests that $\mult(X)$ tends to be conditionally correlated with features that are predictive of the class label (e.g.\ marital status, relationship). In contrast, a group-dependent transformation (like the ones considered by prior post-processing methods) would be conditionally independent of all features.
The explicit dependence of $\mult(X)$ on the entire $X$ (instead of only $A$), helps $\ours$ methods to achieve better fairness-error trade-offs than group-dependent post-processing techniques, as indicated in~\Cref{sec:pp_sota}.
\section{Related work}
\label{sec:related_work}

We now highlight the advantages and shortcomings of prior in- and post-processing methods (summarized in~\Cref{table:related_work}).

In-processing methods can easily induce virtually any quantifiable notion of group fairness into a prediction model \citep{beutel2019, prost19, mary19, cho20, lowy2022, baharlouei24}. 
Moreover, different mathematical tools can be used to enforce a fairness definition,
e.g.\ EqOdds can be induced using HSIC \citep{perez17}, Wasserstein distance \citep{jiang20}, exponential R\'enyi mutual information~\citep{mary19, lowy2022}, and R\'enyi correlation~\citep{baharlouei2019renyi}. In particular, the MinDiff method of \citet{beutel2019, prost19} uses MMD \citep{gretton12} to great effect, can be easily scaled to multiple groups and tasks \citep{atwood23}, and is the standard approach for inducing EqOpp in Tensorflow\footnote{\url{https://www.tensorflow.org/responsible_ai/model_remediation}} thanks to its good performance.

On the other hand, in-processing methods require access to the training pipeline and data as they retrain a new prediction model. They can incur large computational costs \citep{alghamdi22, cruz23}, and are often tailored to a specific model family (e.g.\ gradient-based methods \citep{prost19, lowy2022}, GBMs \citep{cruz23fairgbm}). Moreover, any change in the fairness definition, the set of sensitive attributes or the desired fairness-error trade-off require retraining the entire prediction model from scratch. These challenges are often cited as important obstacles for the broad adoption of fairness mitigations in practice \citep{veale17, holstein19}.

Post-processing is often more appealing for real-world applications, since it alleviates the aforementioned shortcomings of in-processing (i.e.\ it makes no assumptions on the nature of the pre-trained model 
and is less computationally expensive \citep{alghamdi22, cruz23}). However, current post-processing methods suffer from several limitations that hamper their applicability more broadly in practice. First, existing post-processing approaches are heavily tailored to specific problem settings (e.g.\ binary labels \citep{hardt16}) and specific fairness definitions (e.g.\ statistical parity \citep{xian23}, or fairness definitions based on conditional mean scores \citep{wei20, alghamdi22}, which do not include, for instance, fairness notions such as calibration \citep{pleiss17}).

Furthermore, to the best of our knowledge, all prior post-processing mitigations consist in a group-dependent transformation applied to a pre-trained model's outputs. This pattern has two undesired consequences. First, group-dependent transformations require that sensitive attributes are known at inference time. However, in practice, it is often infeasible to collect sensitive attributes at inference time (e.g.\ asking for the ethnicity of a person before predicting their credit score). Furthermore, attempting to infer the sensitive attribute for test-time samples is also undesirable, due to ethical concerns \citep{veale17, holstein19}, and harms caused by data biases \citep{chen19, kallus19}.
Second, prior post-processing approaches only work with discrete  (oftentimes even binary) sensitive attributes and cannot be applied to problems with continuous $A$ (e.g.\ age, income), even though certain in-processing methods are well-suited for this setting \citep{mary19}.

\vspace{-0.2cm}
\section{Conclusion}

\vspace{-0.1cm}
In this paper we propose a generic framework for training a post-processing method for group fairness using a regularized in-processing objective. 
We show theoretically and experimentally that $\ours$ methods enjoy the advantages of post-processing while not degrading the good fairness-error Pareto frontiers achieved with in-processing. 
Unlike prior approaches, our method does not require known sensitive attributes at inference time, and can induce any quantifiable notion of fairness on a broad set of problem settings, including when sensitive attributes are continuous (e.g.\ age, income).
Finally, we demonstrate how $\ours$ methods can alleviate the drop in performance that affects in-processing when only partial group labels are available.

\vspace{-0.2cm}
\section*{Broader Impact \& Limitations}

\vspace{-0.1cm}
The framework proposed in this work significantly expands the range of problem settings where post-processing mitigation techniques can be applied. In particular, $\ours$ methods can be employed to induce fairness in applications with limited computational resources or with no access to the training pipeline and training data of the prediction model. In addition, $\ours$ post-processing methods can help to overcome challenges faced by in-processing methods such as compositional fairness problems \citep{dwork19, atwood23}.

In contrast to $\ours$ methods, which are trained on triples $(\hat{Y}, A, Y)$, prior post-processing methods require access to $(X, A, Y)$ for training. However, even for prior methods, access to the features $X$ is still necessary in order to obtain the predictions $\hat{Y}$. Regarding the computation time required to train $\ours$ methods, we note that it is similar to other competitive post-processing methods (e.g.\ \citet{alghamdi22}). Both $\ours$ methods and prior post-processing techniques perform certain computations to find an appropriate post-hoc transformation. In the case of our approach, it suffices to optimize the parameters of a simple linear regression model to obtain the results shown in our experiments.

When it comes to evaluating algorithmic fairness, popular datasets such as Adult and COMPAS suffer from several limitations which have been pointed out a number of recent works \citep{bao21, ding2021, alghamdi22}. For this reason, we also report our main experimental results of \Cref{sec:exp_ip_vs_pp,sec:pp_sota} on two recently proposed datasets, HSLS \citep{jeong22} and ENEM \citep{alghamdi22}, which specifically address concerns raised about Adult and COMPAS.

Finally, our work does not attempt to provide new arguments in favor of algorithmic fairness. As frequently noted in the ML fairness literature \citep{Corbett17, kasy21, bao21, barocas23}, algorithmic interventions to induce fairness are not always aligned with the intended societal impact. Therefore, it remains the object of active research whether notions such as SP, EqOdds and EqOpp are suitable for evaluating the inequity of decision systems \citep{Buyl22, Ruggieri23, Majumder23}. Furthermore, we note that our work focuses specifically on mitigating group fairness.
Investigating whether our findings also apply to other notions of equity, such as individual fairness \citep{dwork12}, remains an important direction for future work.

\vspace{-0.2cm}
\section*{Acknowledgements}

\vspace{-0.1cm}
We are grateful to Alexander D'Amour, Ananth Balashankar, Amartya Sanyal, Flavio Calmon, and Jilin Chen for helpful discussions, and thank the anonymous reviewers for feedback on the manuscript. We also thank Flavio Calmon, Hsiang Hsu, Sina Baharlouei and Tom Stepleton for their help with reproducing some of the prior work results.

\newpage
\bibliography{postproc}
\bibliographystyle{icml2023}

\newpage
\appendix
\onecolumn

\section{Proof of~\Cref{prop:ip_vs_pp}}
\label{appendix:ip_vs_pp_proof}

\ipvspp*

\begin{proof}
By the definition of $\thetabase$ and from the first order optimality condition it holds that:
\begin{equation}
    \nabla\glmpart(\thetabase) = \frac{1}{\npred} \sum_{i\in[\npred]} \glmtrans(\xvec_i, y_i). \notag
\end{equation}
Plugging this identity into the Bregman divergence of the strongly convex function $\glmpart(\thetavec)$ allows us to write it as follows:
\begin{flalign}
    \breg(\thetavec, \thetabase) &=\glmpart(\thetavec) - \thetavec^\top \frac{1}{\npred} \sum_{i\in[\npred]} \glmtrans(\xvec_i, y_i) \notag \\ &+ \thetabase^\top \frac{1}{\npred} \sum_{i\in[\npred]} \glmtrans(\xvec_i, y_i) - \glmpart(\thetabase) \notag \\
    &= \frac{1}{\npred} \sum_{i\in[\npred]} \predloss(\xvec_i, y_i; \thetavec) + C, \notag
\end{flalign}
where we use the notation $C=\thetabase^\top \frac{1}{\npred} \sum_{i\in[\npred]} \glmtrans(\xvec_i, y_i) - \glmpart(\thetabase)$ for the terms that are independent of $\thetavec$.

Rearranging the terms yields that $\optpp(\thetavec;\lambda)=\optip(\theta;\lambda) + C$ which concludes the proof.
\end{proof}

\section{More related work}
\label{appendix:related_work}

First, in order to help position our work in the existing in- and post-processing literature, we present in \Cref{table:related_work} the specific shortcomings that we target with the desiderata D1 -- D3 introduced in \Cref{sec:intro}.

\begin{table*}[h]
\resizebox{\textwidth}{!}{

\begin{tabular}{lllllll} 
\toprule
Method                                                                                                          & \begin{tabular}[c]{@{}l@{}}Require changing\\prediction model\end{tabular} & \begin{tabular}[c]{@{}l@{}}Require access to\\training pipeline / data\end{tabular} & \begin{tabular}[c]{@{}l@{}}Computa-\\tion cost\end{tabular} & \begin{tabular}[c]{@{}l@{}}Sensitive \\attribute A\end{tabular} & \begin{tabular}[c]{@{}l@{}}Requires A\\ for inference\end{tabular} & Fairness definition                                                                                                                                  \\ 
\hline
\begin{tabular}[c]{@{}l@{}}In-processing methods~\\e.g. \citet{agarwal18, prost19},\\\citet{mary19,cho20}\end{tabular}                  & {\cellcolor[rgb]{0.957,0.8,0.8}}Yes                                      & {\cellcolor[rgb]{0.957,0.8,0.8}}Yes                                                   & {\cellcolor[rgb]{0.957,0.8,0.8}}High                        & {\cellcolor[rgb]{0.851,0.918,0.827}}Any                   & {\cellcolor[rgb]{0.851,0.918,0.827}}No                             & {\cellcolor[rgb]{0.851,0.918,0.827}}Any fairness penalty                                                                                             \\ 
\hline
\begin{tabular}[c]{@{}l@{}}Existing post-processing methods~\\e.g. \citet{hardt16, kamiran18},\\\citet{alghamdi22, xian23}\end{tabular} & {\cellcolor[rgb]{0.851,0.918,0.827}}No                                   & {\cellcolor[rgb]{0.851,0.918,0.827}}No                                                & {\cellcolor[rgb]{0.851,0.918,0.827}}Low                     & {\cellcolor[rgb]{0.957,0.8,0.8}}Discrete             & {\cellcolor[rgb]{0.957,0.8,0.8}}Yes                                & {\cellcolor[rgb]{0.957,0.8,0.8}}\begin{tabular}[c]{@{}>{\cellcolor[rgb]{0.957,0.8,0.8}}l@{}}Tailored to specific\\fairness definitions\end{tabular}  \\ 
\hline
\begin{tabular}[c]{@{}l@{}} \\\textbf{$\ours$ methods}\\\vspace{0.001cm}\end{tabular} & {\cellcolor[rgb]{0.851,0.918,0.827}}No                                   & {\cellcolor[rgb]{0.851,0.918,0.827}}No                                                & {\cellcolor[rgb]{0.851,0.918,0.827}}Low                     & {\cellcolor[rgb]{0.851,0.918,0.827}}Any                   & {\cellcolor[rgb]{0.851,0.918,0.827}}No                             & {\cellcolor[rgb]{0.851,0.918,0.827}}Any fairness penalty                                                                                             \\
\bottomrule
\end{tabular}
}
\caption{In-processing requires retraining the entire prediction model to induce fairness, but can be applied to a broad range of problem settings and to virtually any quantifiable notion of fairness. On the other hand, current post-processing methods are tailored to specific settings and fairness definitions. In contrast, $\ours$ methods are as broadly applicable as penalized in-processing methods, while not being confined to applications with access to the training pipeline of the prediction model.}
\label{table:related_work}
\end{table*}

In the remainder of this section we elaborate on some of the limitations of in-processing that have been previously documented in the literature and have not been discussed extensively in \Cref{sec:related_work}.

\paragraph{In-processing and compositional fairness.} It is often the case, in practical applications that multiple prediction models are employed, and their outputs are then all aggregated into a single decision. For instance, a candidate may apply for several jobs, each with their own selection criteria, but the outcome that is of interest to the candidate is whether at least one of the applications is successful. Similarly, complex decision problems may be broken down into finer grained tasks for the purpose of better interpretability. These situations are prone to compositional fairness issues \citep{dwork19}: even if all individual components are fair, it is not guaranteed that the aggregated decision will also be fair. In-processing techniques are inherently susceptible to limitations due to compositional fairness \citep{atwood23}. Indeed, in order to mitigate these issues, in-processing method could train all individual models simultaneously while enforcing that the aggregated decision is fair. However, this procedure raises huge logistical challenges which are often insurmountable in practice. In contrast, post-processing bypasses compositional fairness issues altogether. Applying a post-processing method to the final decision of a complex system with multiple prediction components that are aggregated into a single decision can treat the entire decision system as a black-box and overcome limitations due to compositional fairness. A thorough investigation of this intuition is left as future work.

\paragraph{In-processing when data has partial group labels.} In addition to the works mentioned in \Cref{sec:novel_failure}, there have been a few empirical observations that attempt to study this problem. In particular, \citet{veldanda23} investigate the performance of MinDiff \citep{prost19} in the overparameterized regime, where the complexity of the model fit to the training data increases while the sample size stays fixed. The authors show in experiments on image data that explicit regularization (e.g.\ early stopping) can improve the performance of MinDiff with overparameterized models. However, the impact of overparameterization on the fairness-error Pareto frontiers is not studied in detail. 

Furthermore, \citet{lowy2022} present experiments where the amount of data with sensitive annotations is reduced to $10\%$. In this case, the authors aim to compare their proposed method to other in-processing strategies, but no catastrophic loss in performance is noticeable. We hypothesize that this is due to the amount of data with sensitive attributes being still large enough to allow for good performance. In particular, our experiments (\Cref{fig:low_sample}) reveal that when the sensitive data is reduced to $0.1\%$ of the training set size, on Adult in-processing performs no better than the naive strategy of predicting the favorable outcome with probability $p$ and the output of the pre-trained model with probability $1-p$.

\section{Experiment details}
\label{appendix:all_exp_details}

\subsection{Details on fairness definitions}
\label{appendix:fairness_definitions}

In this section, we provide more details about the notions of group fairness used throughout this paper. We note that this is not an exhaustive list of fairness definitions, and other notions are possible and considered in the literature too (e.g.\ worst-group error). We refer to surveys such as \citet{caton23} books like \citet{barocas23} for more details.

\paragraph{Statistical parity (SP).} Also known as demographic parity, SP measures the difference between the frequency of favorable outcomes in the subpopulations determined by the values of the sensitive attribute $A$ \citep{dwork12}. To quantify the violation of the SP condition, several works \citep{donini18, jiang20, cho20} consider the difference with respect to statistical parity. Assuming that the favorable outcome is $y=1$, this quantity is defined as follows:
\begin{align}
    \SPgap(\fct) = \sum_{a}|\PP(\fct(X) = 1|A=a) - \PP(\fct(X)=1)|,
\end{align}
where the sum is over all the possible values of the sensitive attribute. Note that the sum can also be replaced with a ``$\max$'' operator in the formulation above.

\paragraph{Equal opportunity (EqOpp).} This fairness definition is tailored for settings with discrete labels $y$ and sensitive attributes $a$. Intuitively, EqOpp asks that a classifier is not more likely to assign the favorable outcome to one of the groups determined by the (discrete) sensitive attribute $a$. Assuming the negative class $y=0$ is more desirable, one can quantify the fairness of a binary predictor using the following:
\begin{align}
    \fprgap(\fct) = |\PP(\fct(X)\neq Y|Y=0, A=0) - \PP(\fct(X)\neq Y|Y=0, A=1)|.
\end{align}
    
This metric can also be generalized to multiclass classification.
    
\paragraph{Equalized odds (EqOdds).} This notion of fairness is satisfied if $A \perp \Yhat | Y$. Intuitively, EqOdds penalizes the predictor if it relies on potential spurious correlations between A and Y. One can quantify the violation of this definition of fairness using $\rho(A, \Yhat | Y)$, where $\rho$ is a measure of conditional statistical independence (e.g.\ HSIC \citep{gretton05}, CKA \citep{cristianini02, cortes12}, HGR \citep{gebelein} etc). One can either use one of these quantities to evaluate the fairness of a model (e.g.\ HGR$_\infty(f(X), A | Y)$ like in \citet{mary19}) or a metric such as mean equalized odds, which, for binary classification, can be defined as:
\begin{equation}
    \text{MEO}=\frac{\tprgap(\fct) + \fprgap(\fct)}{2},
\end{equation} 
where $\tprgap(\fct)$ is the gap in the true positive rate between groups and is defined similarly to $\fprgap(\fct)$.

\subsection{Datasets}
\label{appendix:datasets}

We briefly describe the datasets used throughout the experiments presented in the paper. 

The \textbf{Adult} dataset \citep{becker96} is perhaps the most popular dataset in the algorithmic fairness literature. The task it proposes is to predict whether the income of a person is over the $50,000\$$ threshold, having access to various demographic features. In our experiments, we consider gender as the sensitive attribute. We follow the procedure described in \citet{alghamdi22} to pre-process the data.

Alongside Adult, \textbf{COMPAS} \citep{angwin16} is also a well-established dataset for evaluating fairness mitigations. It contains information about defendents detained in US prisons. The task is to predict the individual risk of recidivism, while being fair with respect to race. We adopt the pre-processing methodology of \citet{alghamdi22} for this dataset.

The \textbf{Crimes~\&~Communities} dataset \citep{redmond09} is also part of the UCI repository \citep{dua17}, like Adult, and contains information about US cities. The task is a regression problem where the goal is to predict the amount of violent crimes and the sensitive attribute is the proportion of an ethnic group in the population. Hence, the sensitive attribute takes continuous values. For this dataset, we use the same pre-processing as \citet{mary19}.

The \textbf{HSLS} dataset \citep{jeong22} contains information about over $23,000$ students from high schools in the USA. The features consist in information about the students' demographic and academic performance, as well as data about the schools. The data is pre-processed using the same procedure as \citet{alghamdi22}. The task is to predict exam scores while being fair with respect to race.

\textbf{ENEM} \citep{alghamdi22} is a dataset of exam scores collected in Brazilian high schools. The dataset contains demographic and socio-economic information about the students. Once again, we use the same pre-processing methodology as \citet{alghamdi22}. Similar to HSLS, the goal is to predict the Humanities exam score, while the sensitive attribute is race.

\subsection{Baselines}
\label{appendix:baselines_description}

We compare the performance of $\ours$ methods obtained with our framework with several competitive in- and post-processing approaches.

\paragraph{In-processing baselines.} We consider three different regularized in-processing methods and one constrained in-processing approach for which we construct $\ours$ post-processing counterparts. First, MinDiff \citep{beutel2019, prost19} is an approach that uses MMD \citep{gretton12} to induce the statistical independence required for various fairness definitions to hold (i.e.\ EqOpp, EqOdds). The remarkable performance of this method led to it being included in standard fairness toolkits such as \texttt{tensorflow-model-remediation}.\footnote{\url{https://www.tensorflow.org/responsible_ai/model_remediation}} Furthermore, the method of \citet{cho20} employs kernel density estimation to construct a regularizer for certain fairness definition violations. In addition to EqOdds, this method can also be applied to induce SP. Finally, \citet{mary19} propose to use the Hirschfeld-Gebelein-R\'enyi (HGR) Maximum Correlation Coefficient to quantify statistical independence and propose an unfairness regularizer based on this metric. Besides these regularized in-processing methods, we also consider the Reductions approach \citep{agarwal18} in our comparison, which proposes solving a constrained optimization problem.
For all the baselines, we use the hyperparameters recommended in the respective papers. 

\paragraph{Post-processing baselines.} The FairProjection method of \citet{alghamdi22} is, to the best of our knowledge, one of the best performing post-processing mitigations. FairProjection adjusts the scores output by a classification method, using a different transformation for each sensitive group in the population. Alternatively, the methods of \citet{hardt16,chzhen20} change the decision threshold in a group-dependent manner. These two approaches do not prescribe a way to obtain an entire Pareto frontier, but rather a single point on the fairness-error trade-off. Finally, the Rejection-option classification method of \citet{kamiran18} exploits uncertainty in the decision of a classifier to decide what labels to output. For all of these methods, we use the results from the public code repository of \citet{alghamdi22}.

\subsection{Experiment details for training $\ours$ methods}
\label{appendix:exp_details}

For the comparison with in-processing methods, we use the pre-trained models recommended in the respective papers (i.e.\ 3-MLP with 128 hidden units on each layer for MinDiff and \citet{cho20}, and logistic regression for \citet{agarwal18} and \citet{mary19}. For the $\ours$ post-hoc transformation, we use linear regression to model $\mult(\xvec)$. We select the optimal learning rate by minimizing the prediction error on a held-out validation set. To obtain the Pareto frontiers, we vary the $\lambda$ coefficient that balances the prediction error and the fairness regularizer terms in the loss.

For the comparison with prior post-processing works in \Cref{sec:pp_sota} we use $\ours$ MinDiff. For these experiments, we employ a variant of MinDiff tailored to EqOdds which encourages not only the FPR gap to be small between sensitive groups, but also the FNR. Once again, we select the optimal learning rate using a validation set, and train linear regression and 1-MLP models with 64 hidden units as the $\mult(\xvec)$ post-hoc transformation. The pre-trained models are obtained following the instructions in \citet{alghamdi22}.

For the analysis of the post-hoc transformation (\Cref{sec:analysis_of_t}), we use a 3-MLP as the pre-trained model and assume $\mult(X)$ to be a 1-MLP trained using $\ours$ MinDiff, like in~\cref{sec:exp_ip_vs_pp}. We only consider Adult and COMPAS since they have fewer covariates, which makes them suitable for visualization.

For the experiments with partial group labels in \Cref{sec:novel_failure} we consider $\ours$ MinDiff for EqOpp, with a 3-MLP pre-trained model with 128 hidden units on each layer. We use a 1-MLP with 64 hidden units to model the post-processing transformation. The optimal learning rate and early-stopping epoch are selected so as to minimize prediction error on a held-out validation set.

\subsection{Measuring computation cost}
\label{appendix:comp_time}

To compare the computation time of $\ours$ methods and compare it to the correspoding in-processing methods we generate Pareto frontiers for each of the settings in \Cref{fig:ip_vs_pp}, where we always select $8$ different values for the coefficient that controls the fairness-error trade-off and repeat each experiment $10$ times with different random seeds. In total, for each of the three settings in \Cref{fig:ip_vs_pp} we perform $80$ experiments sequentially. For the post-processing methods, we include the time required to train a base model in the reported computation times. The machine we used for these measurements has $32$ 1.5 GHz CPUs.

\section{More experiments}

\subsection{Equivalence between in- and post-processing on more datasets}
\label{appendix:ip_vs_pp}

\Cref{fig:ip_vs_pp} demonstrates on the HSLS dataset that $\ours$ methods preserve the good fairness-error trade-off achieved by several regularized in-processing approaches, while also enjoying the advantages of post-processing methods. \Cref{fig:ip_vs_pp_mindiff,fig:ip_vs_pp_cho,fig:ip_vs_pp_reductions} complement \Cref{fig:ip_vs_pp} and present similar results on three more datasets: Adult, COMPAS and ENEM.

\begin{figure*}[ht]
\centering
\begin{subfigure}[t]{0.32\textwidth}
    \includegraphics[width=\textwidth]{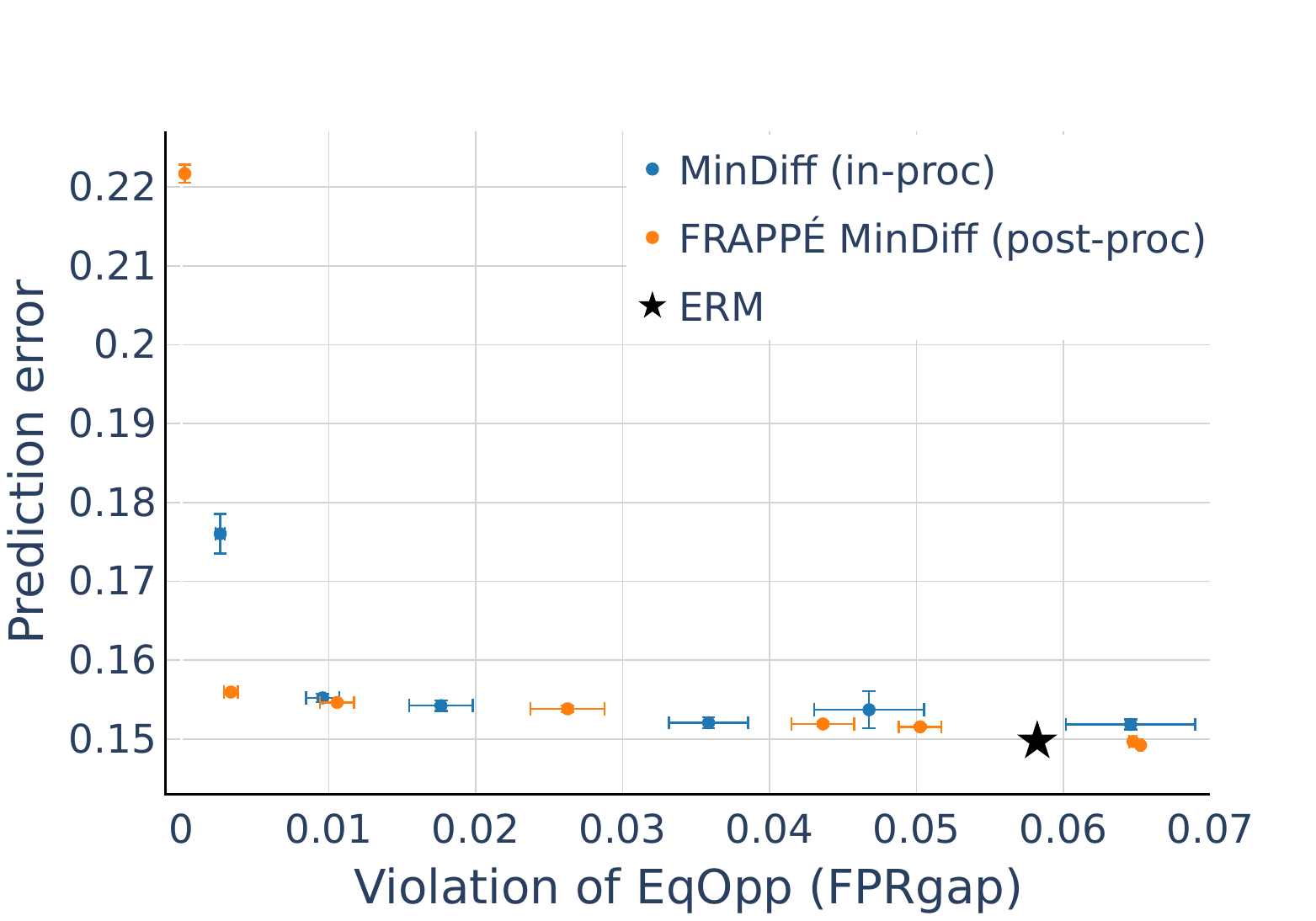}
\caption{$\optip=$ \textit{MinDiff} on Adult data.}
\end{subfigure}
\hfill
\begin{subfigure}[t]{0.32\textwidth}
    \includegraphics[width=\textwidth]{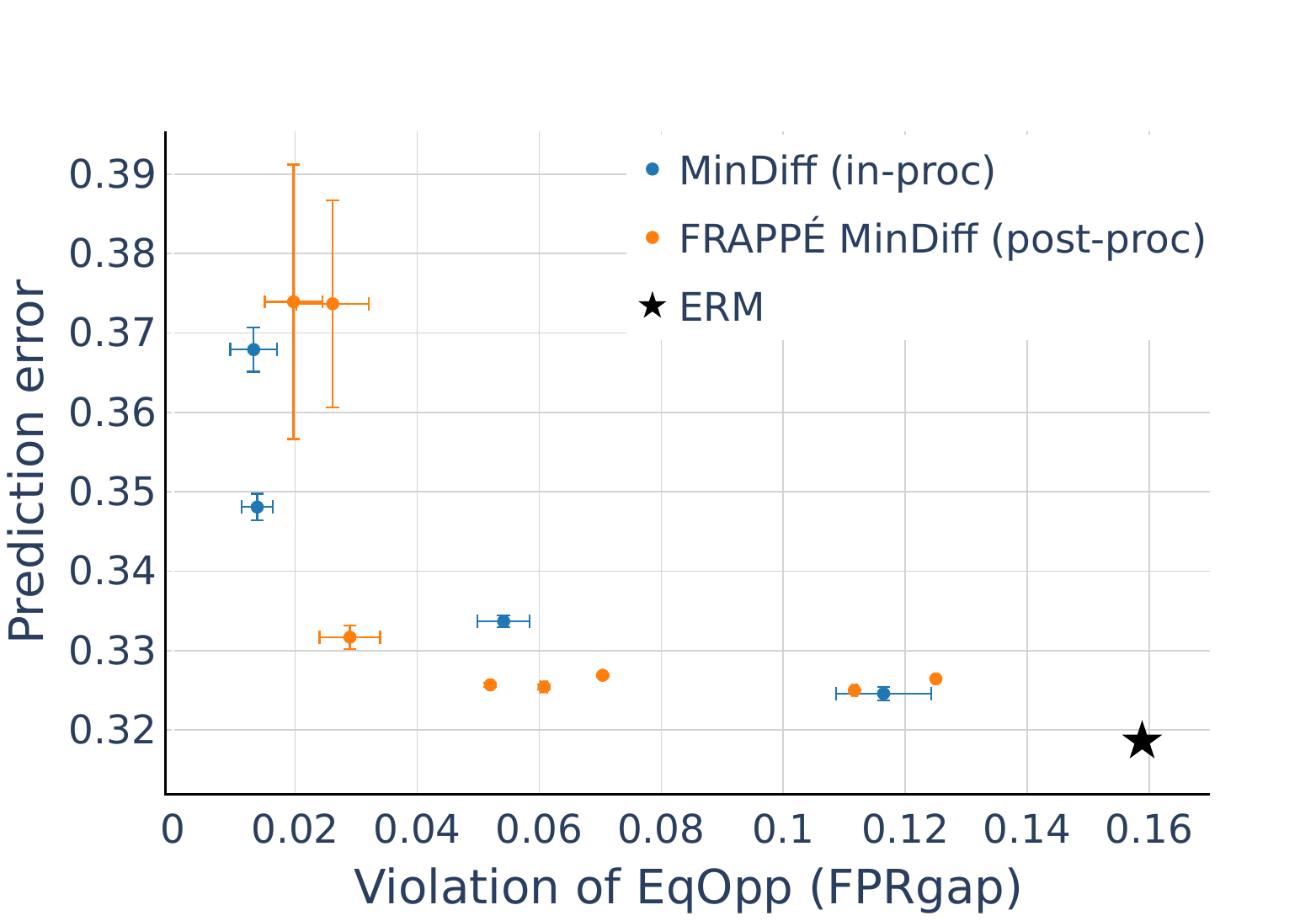}
\caption{$\optip=$ \textit{MinDiff} on COMPAS data.}
\end{subfigure}
\hfill
\begin{subfigure}[t]{0.32\textwidth}
    \includegraphics[width=\textwidth]{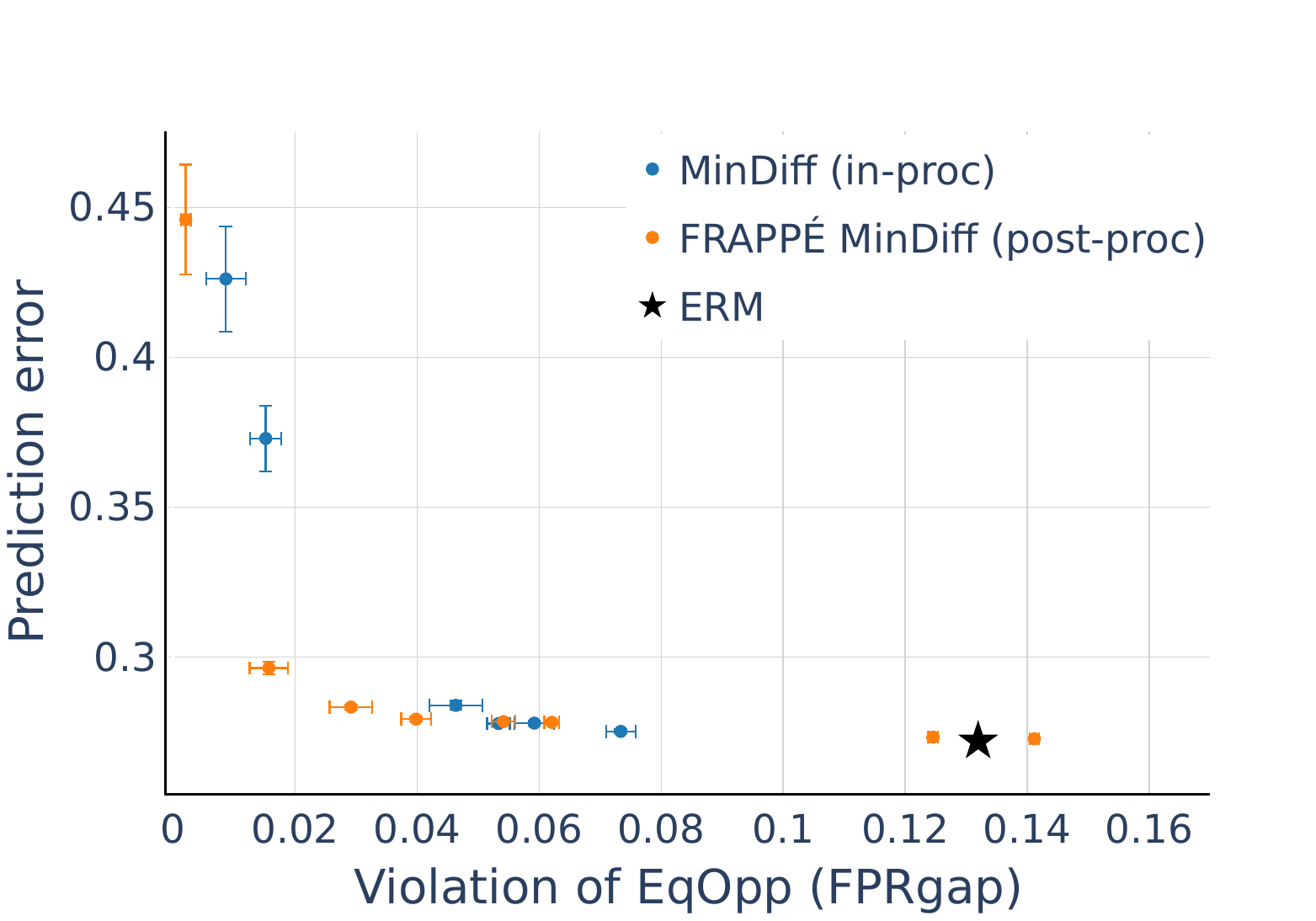}
\caption{$\optip=$ \textit{MinDiff} on ENEM data.}
\end{subfigure}
\caption{Inducing EqOpp using in-processing MinDiff and its $\ours$ post-processing variant leads to similar Pareto frontiers on several different datasets.}
\label{fig:ip_vs_pp_mindiff}
\vspace{-0.5cm}
\end{figure*}

\begin{figure*}[ht]
\centering
\begin{subfigure}[t]{0.32\textwidth}
    \includegraphics[width=\textwidth]{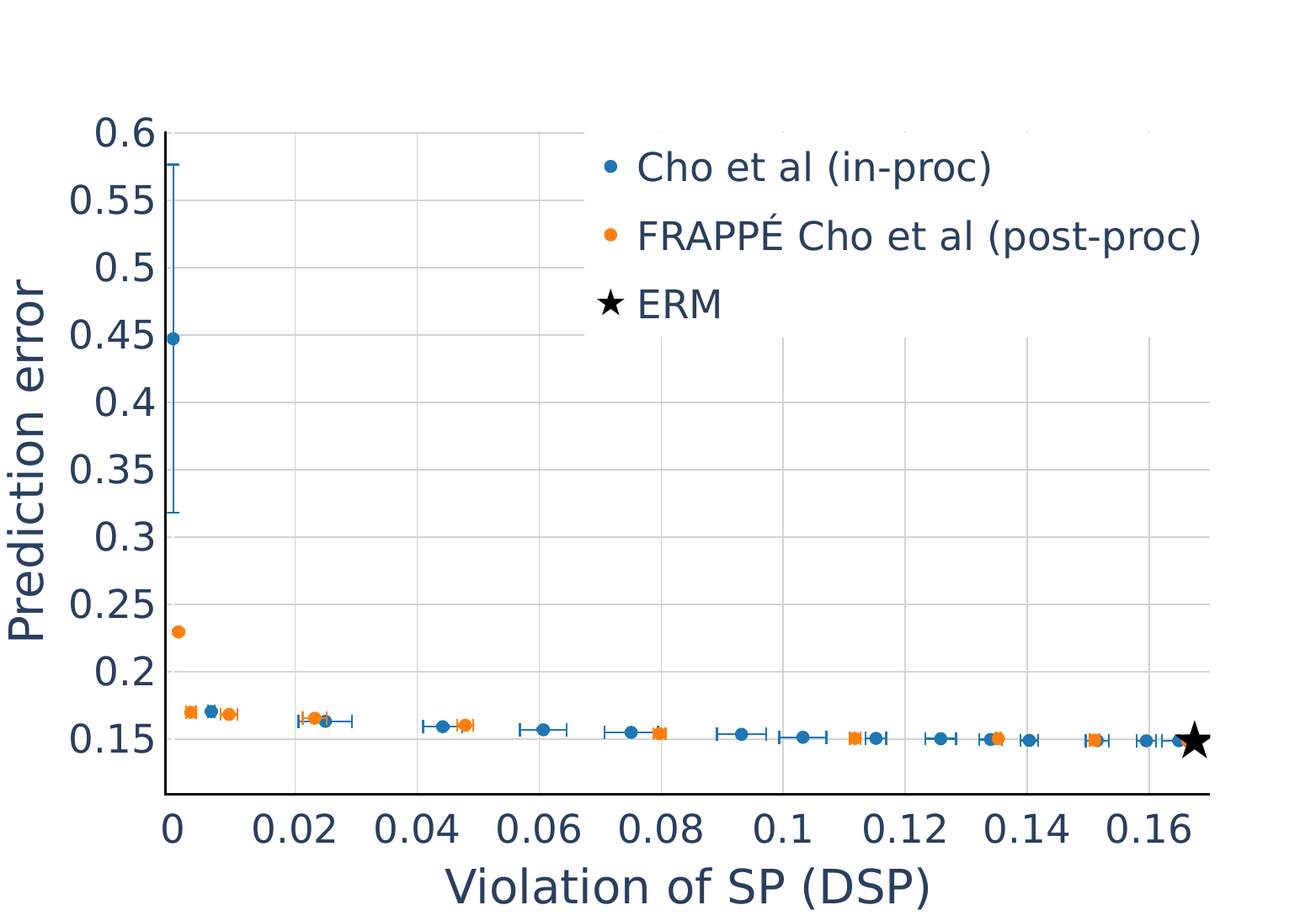}
\caption{$\optip=$ \citet{cho20} on Adult.}
\end{subfigure}
\hfill
\begin{subfigure}[t]{0.32\textwidth}
    \includegraphics[width=\textwidth]{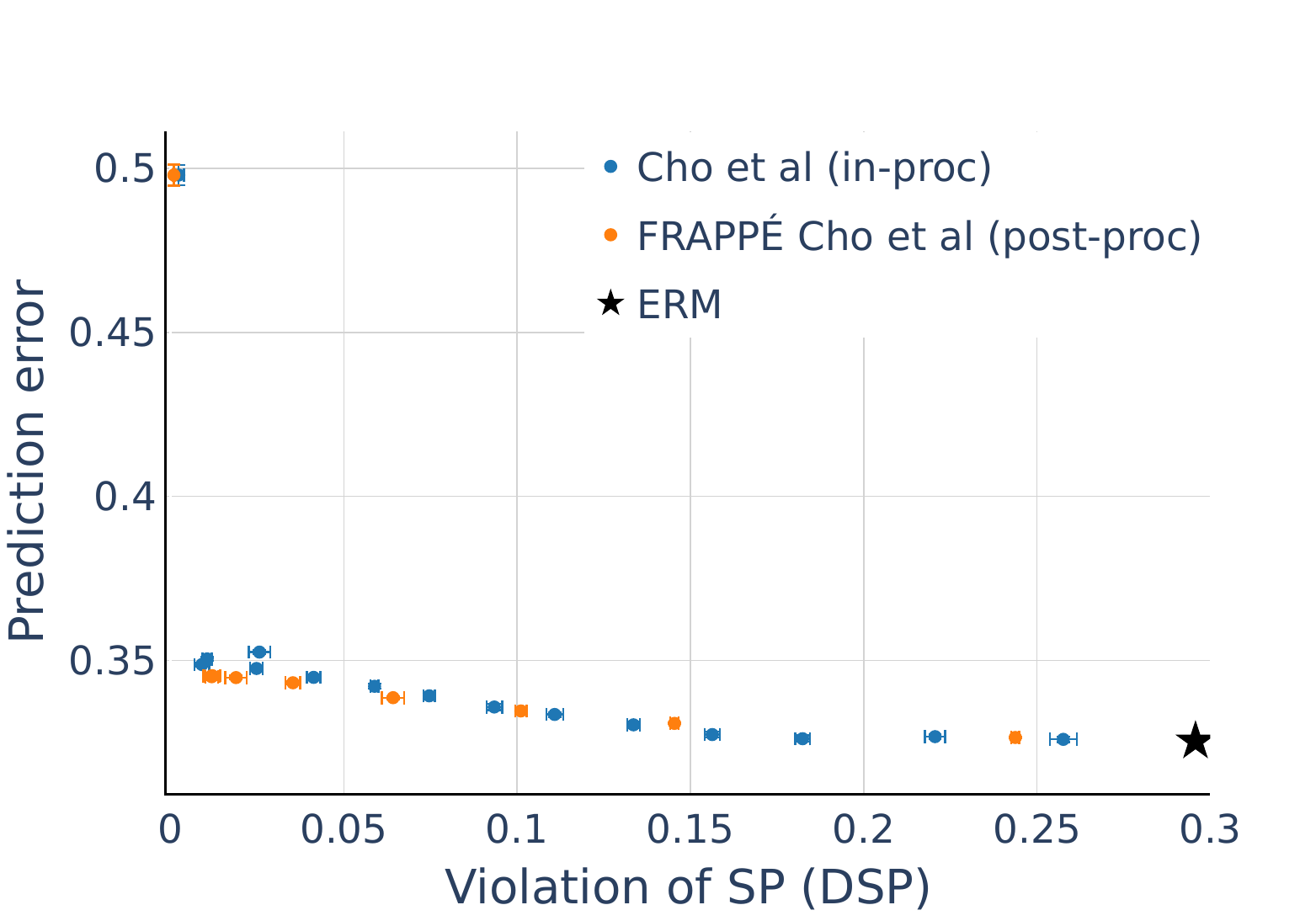}
\caption{$\optip=$ \citet{cho20} on COMPAS.}
\label{fig:ip_vs_pp_cho_compas}
\end{subfigure}
\hfill
\begin{subfigure}[t]{0.32\textwidth}
    \includegraphics[width=\textwidth]{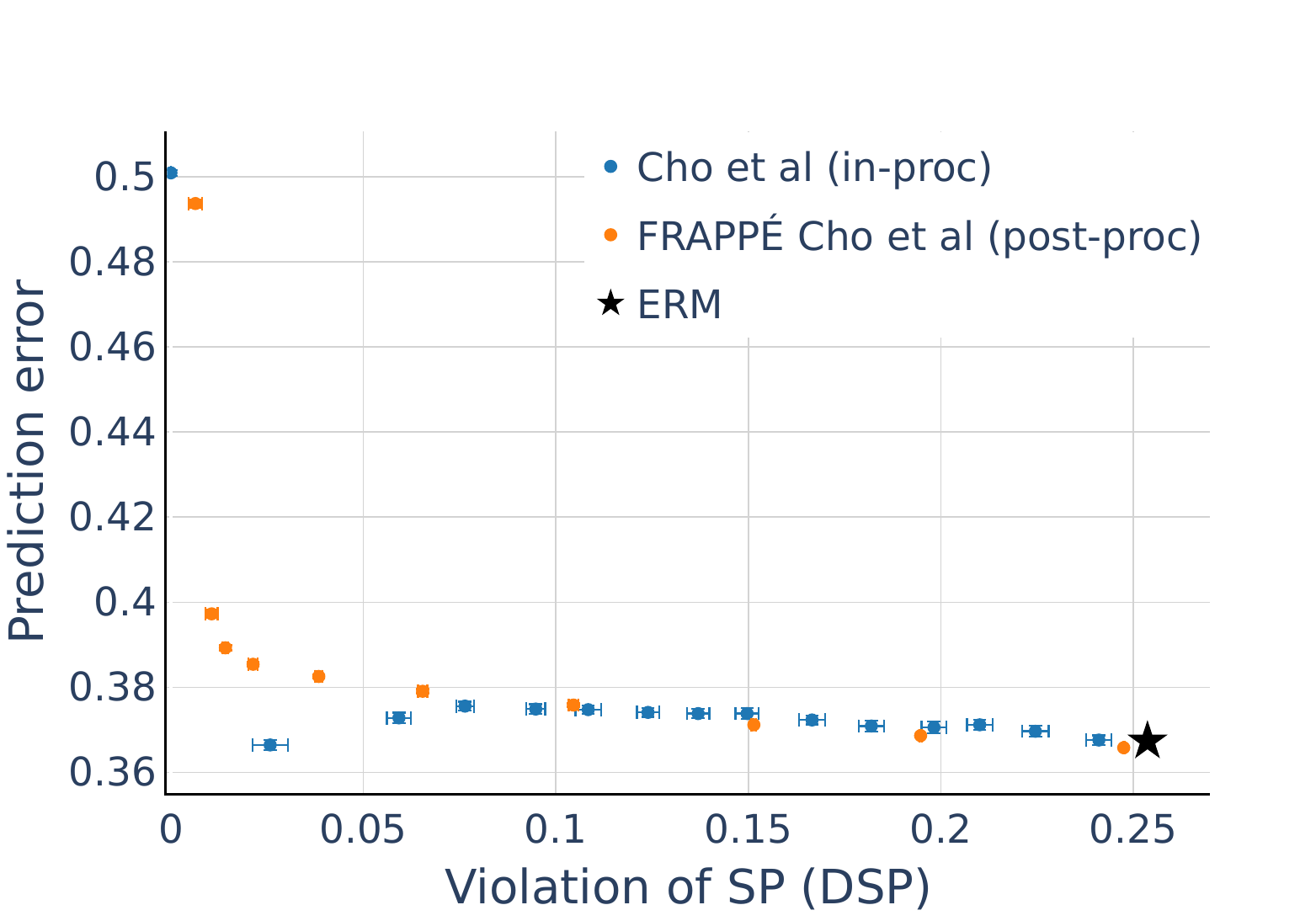}
\caption{$\optip=$ \citet{cho20} on ENEM.}
\end{subfigure}
\caption{Inducing SP using the in-processing  method of \citet{cho20} and its $\ours$ post-processing variant leads to similar Pareto frontiers on several different datasets.}
\label{fig:ip_vs_pp_cho}
\vspace{-0.2cm}
\end{figure*}

\begin{figure*}[!ht]
\centering
\begin{subfigure}[t]{0.32\textwidth}
    \includegraphics[width=\textwidth]{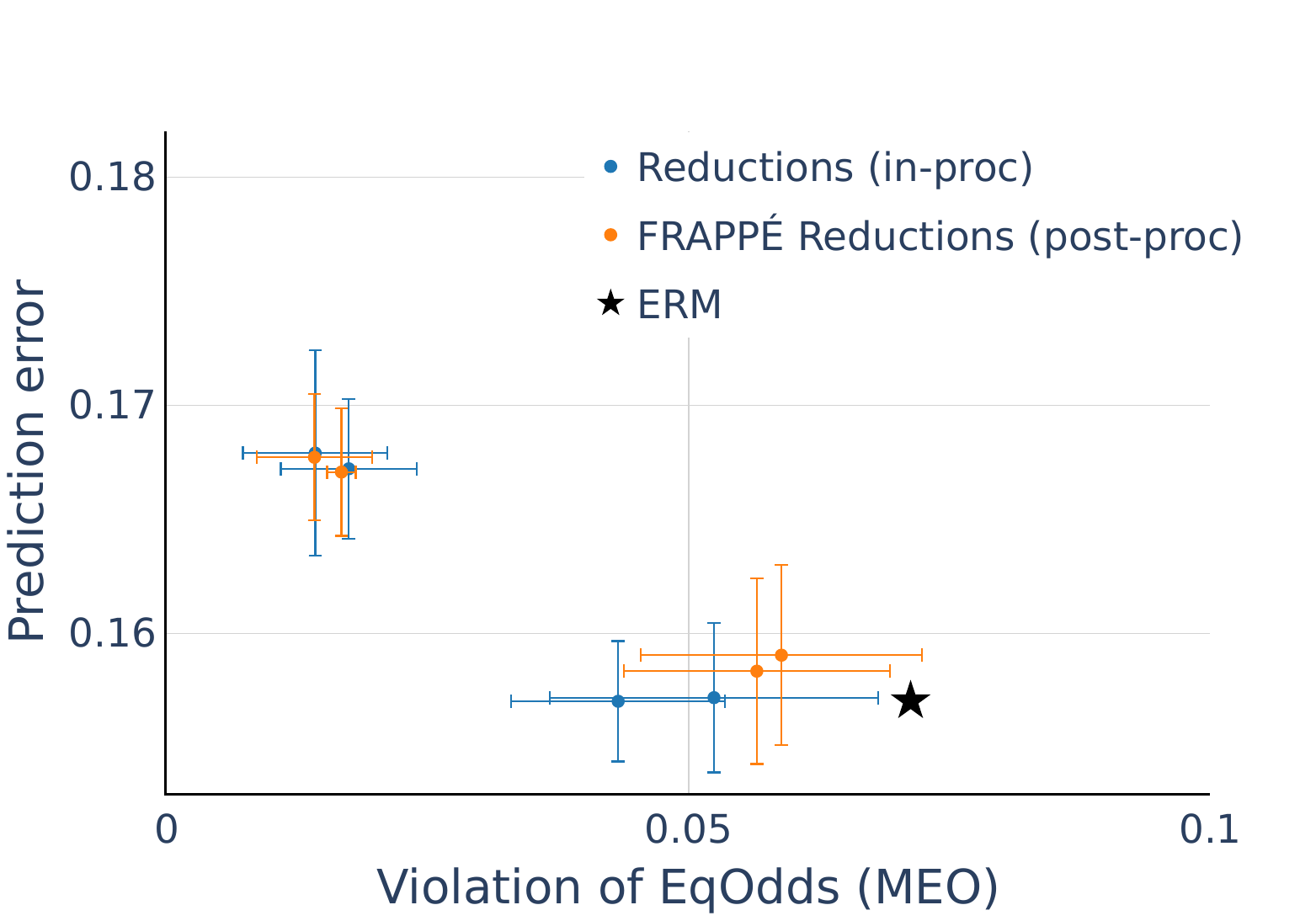}
\caption{$\optip=$ \textit{Reductions} on Adult.}
\end{subfigure}
\hfill
\begin{subfigure}[t]{0.32\textwidth}
    \includegraphics[width=\textwidth]{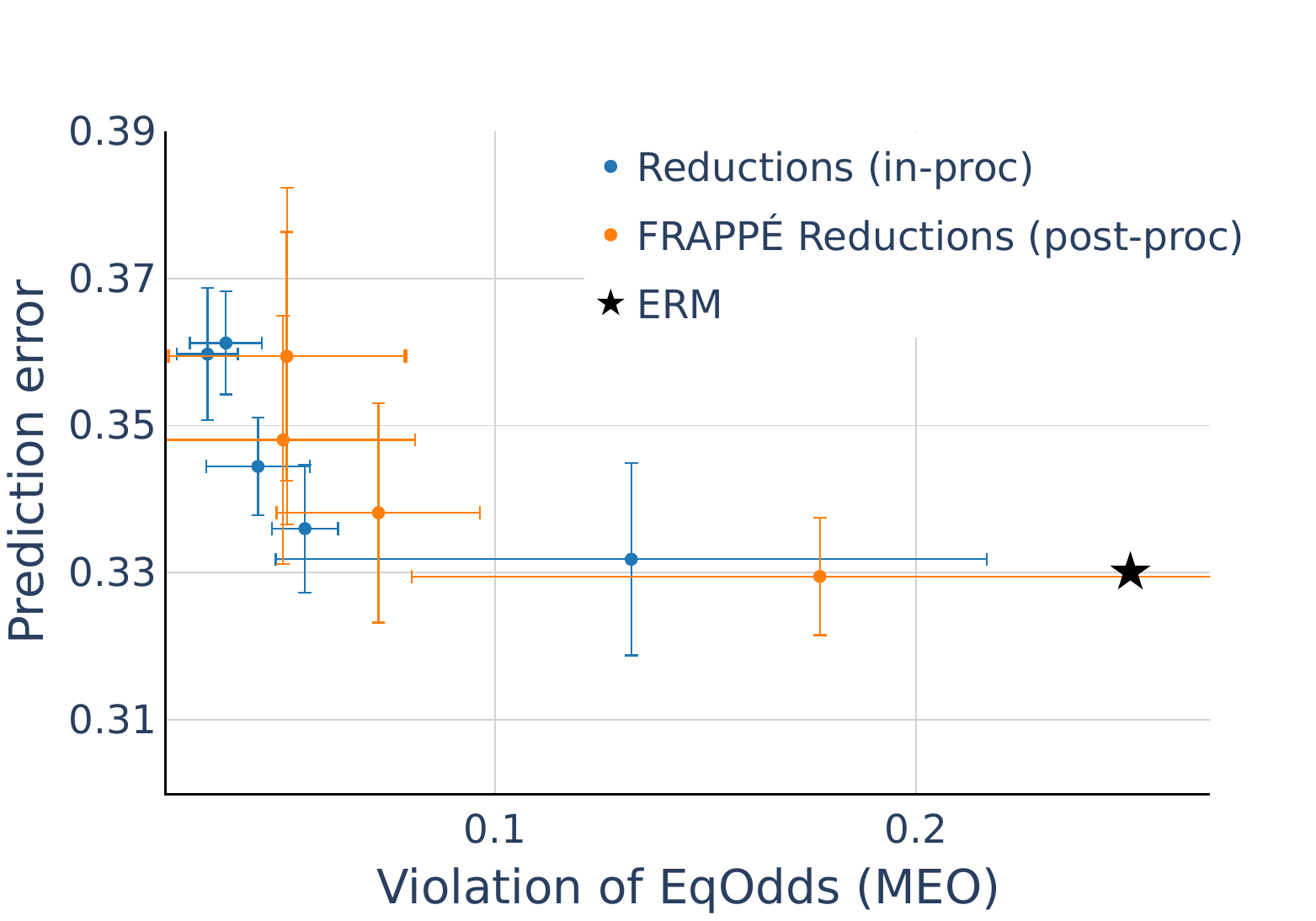}
\caption{$\optip=$ \textit{Reductions} on COMPAS.}
\end{subfigure}
\hfill
\begin{subfigure}[t]{0.32\textwidth}
    \includegraphics[width=\textwidth]{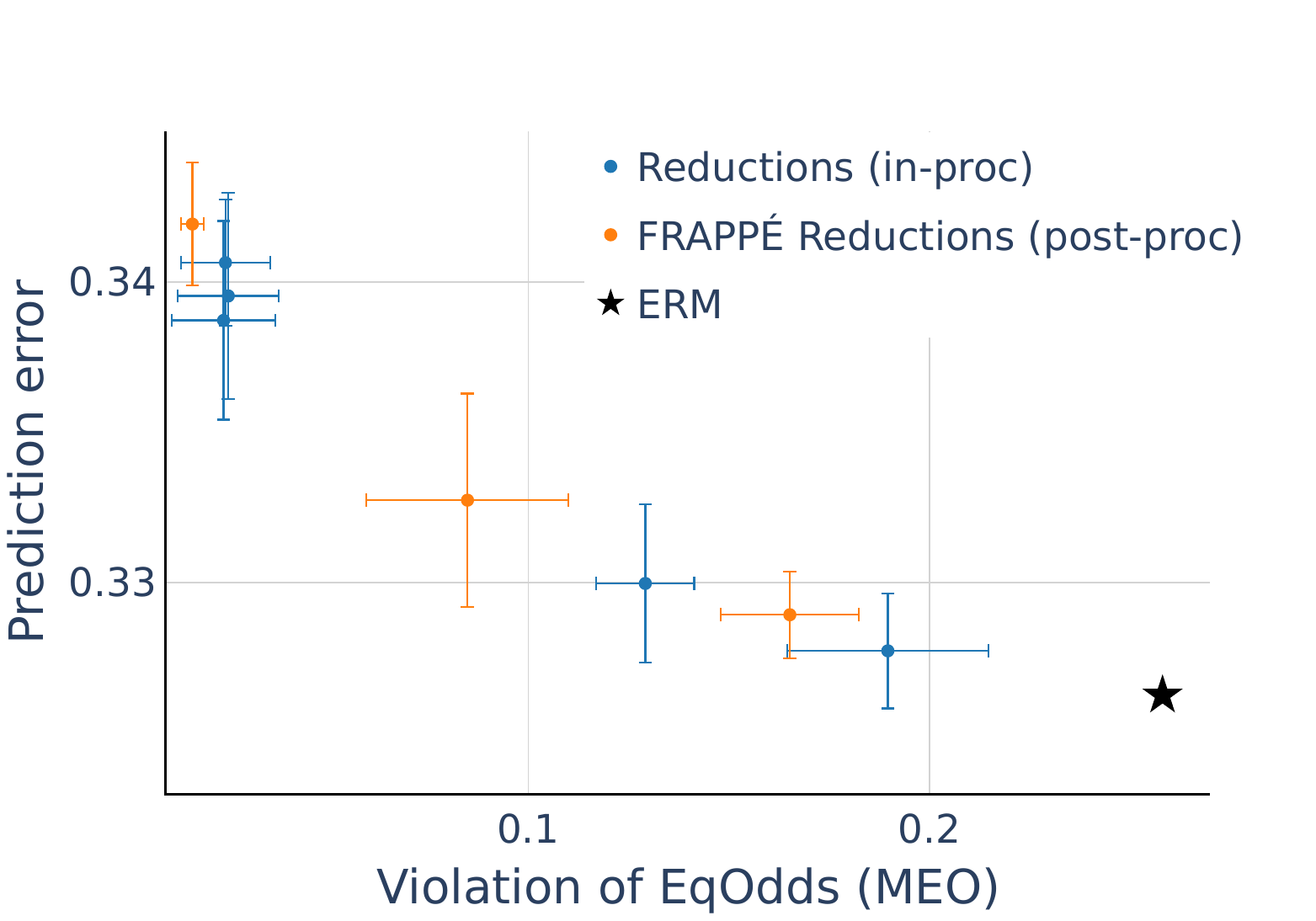}
\caption{$\optip=$ \textit{Reductions} on ENEM.}
\end{subfigure}
\caption{Inducing EqOdds using in-processing Reductions \citep{agarwal18} and its $\ours$ post-processing variant leads to similar Pareto frontiers on several different datasets.}
\label{fig:ip_vs_pp_reductions}
\vspace{-0.2cm}
\end{figure*}

\subsection{More comparisons with prior mitigations}
\label{appendix:more_baselines}

In this section we extend \Cref{fig:comparison_alghamdi} with the results obtained with more in- and post-processing baselines. We consider the same methods as \citet{alghamdi22}, described in more detail in \Cref{appendix:baselines_description}. Unless otherwise specified, the techniques presented in \Cref{fig:app_more_baselines} are post-processing approaches. Like in \Cref{fig:comparison_alghamdi}, we train $\ours$ MinDiff for EqOdds, where the post-hoc transformation is modeled by either linear regression or a simple 1-MLP with $64$ hidden units. The numbers for all the baselines are collected from the  public code of \citet{alghamdi22}.

In addition to using random forests (RF) as the base model, we also present results for logistic regression and GBMs as the pre-trained model for all three datasets in \Cref{fig:app_alghamdi_logreg} and \Cref{fig:app_alghamdi_gbm}, respectively. The figures reveal that the same trends observed for RF pre-trained models also occur for other classes of pre-trained models.

\begin{figure*}[ht]
\centering
\begin{subfigure}[t]{0.24\textwidth}
    \includegraphics[width=\textwidth]{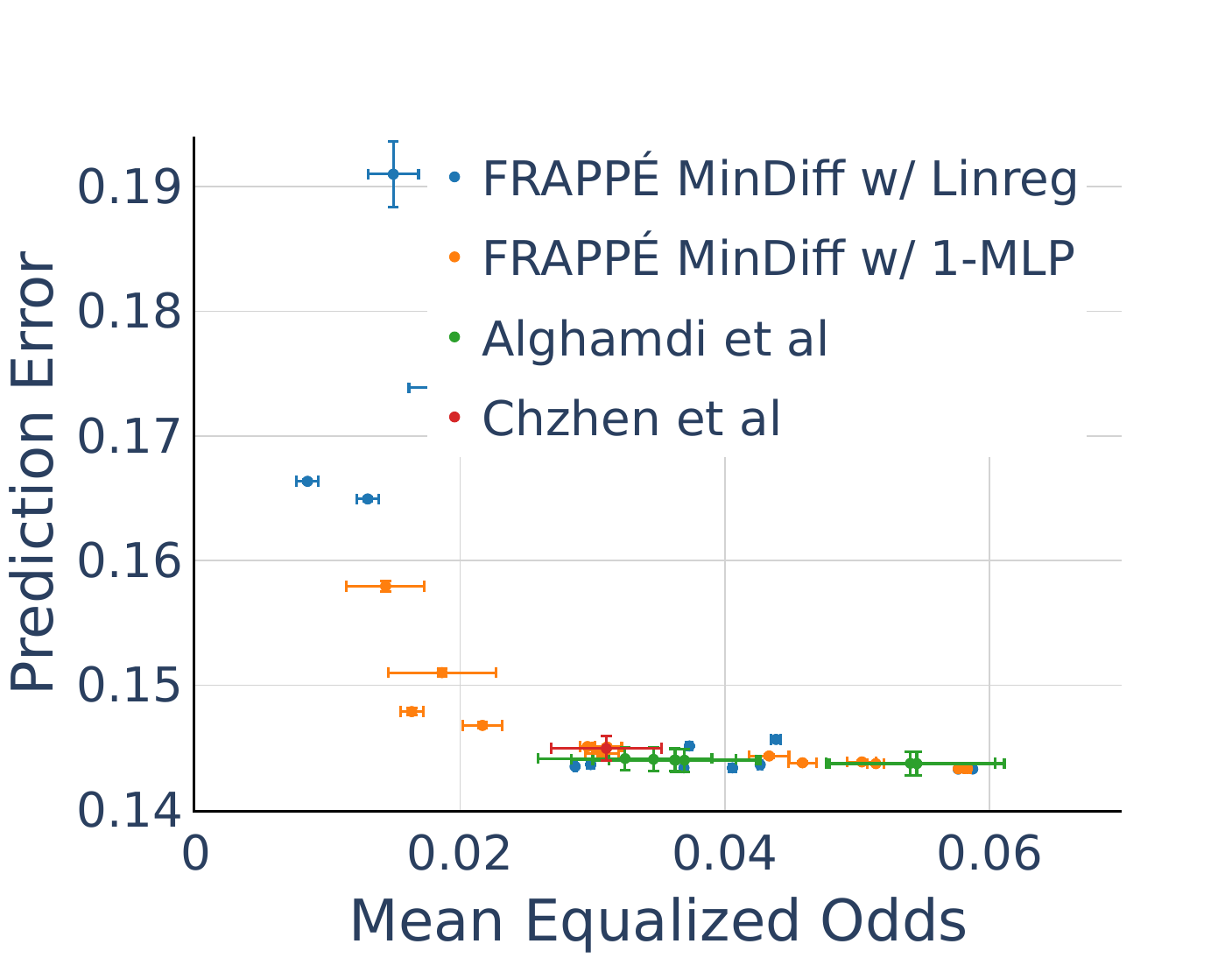}
    \caption{Adult}
\end{subfigure}
\hfill
\begin{subfigure}[t]{0.24\textwidth}
    \includegraphics[width=\textwidth]{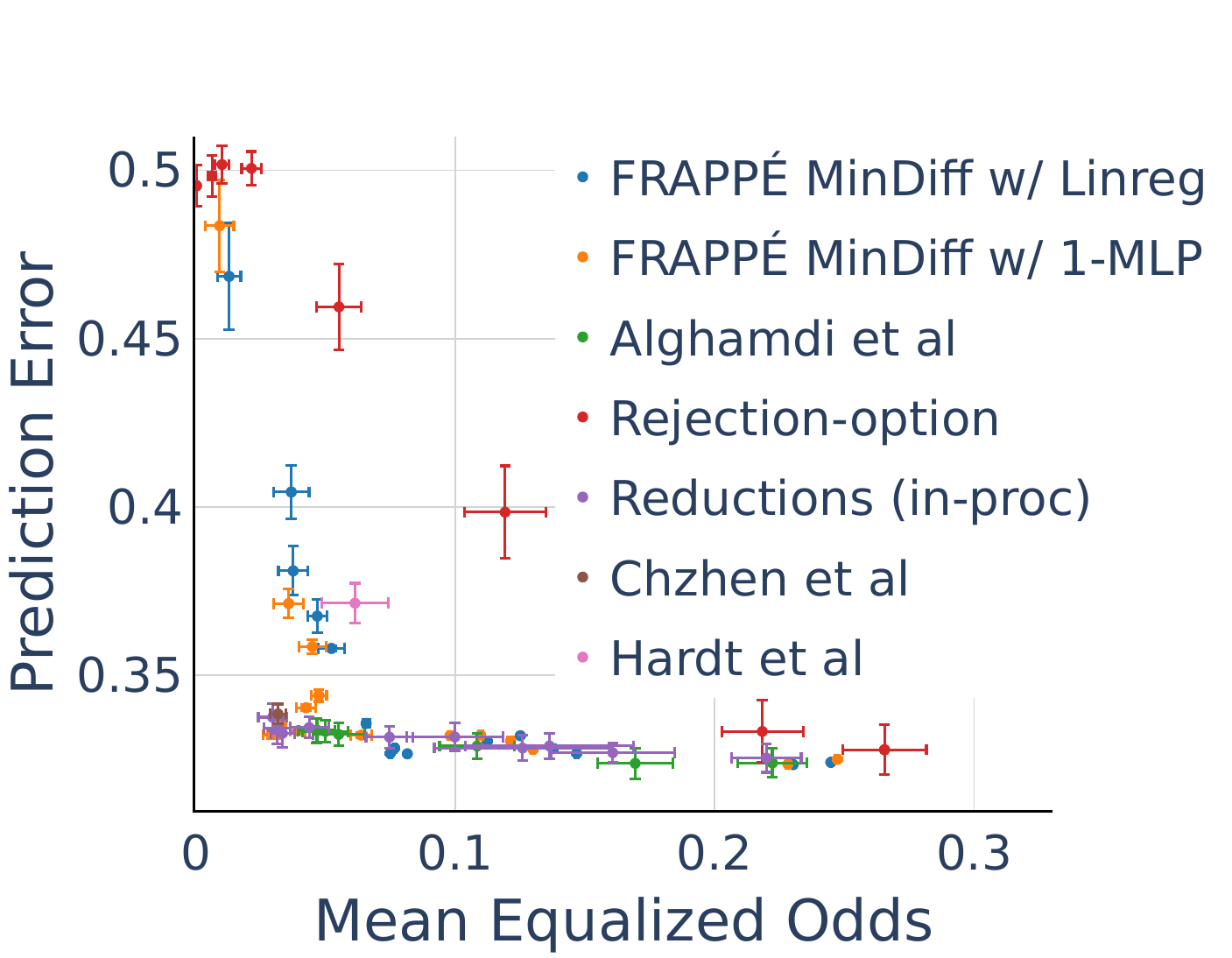}
    \caption{COMPAS}
\end{subfigure}
\hfill
\begin{subfigure}[t]{0.24\textwidth}
    \includegraphics[width=\textwidth]{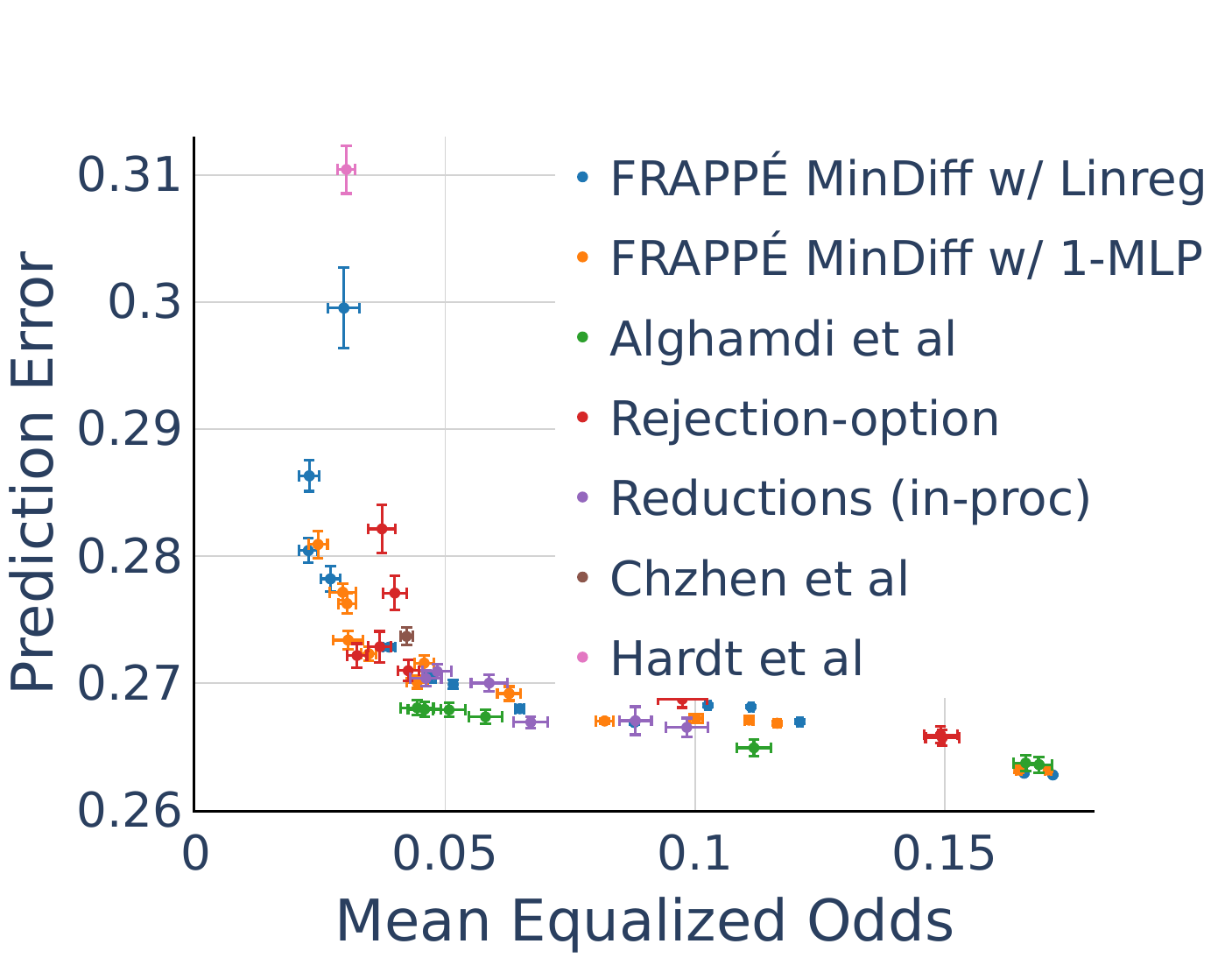}
    \caption{HSLS}
\end{subfigure}
\hfill
\begin{subfigure}[t]{0.24\textwidth}
    \includegraphics[width=\textwidth]{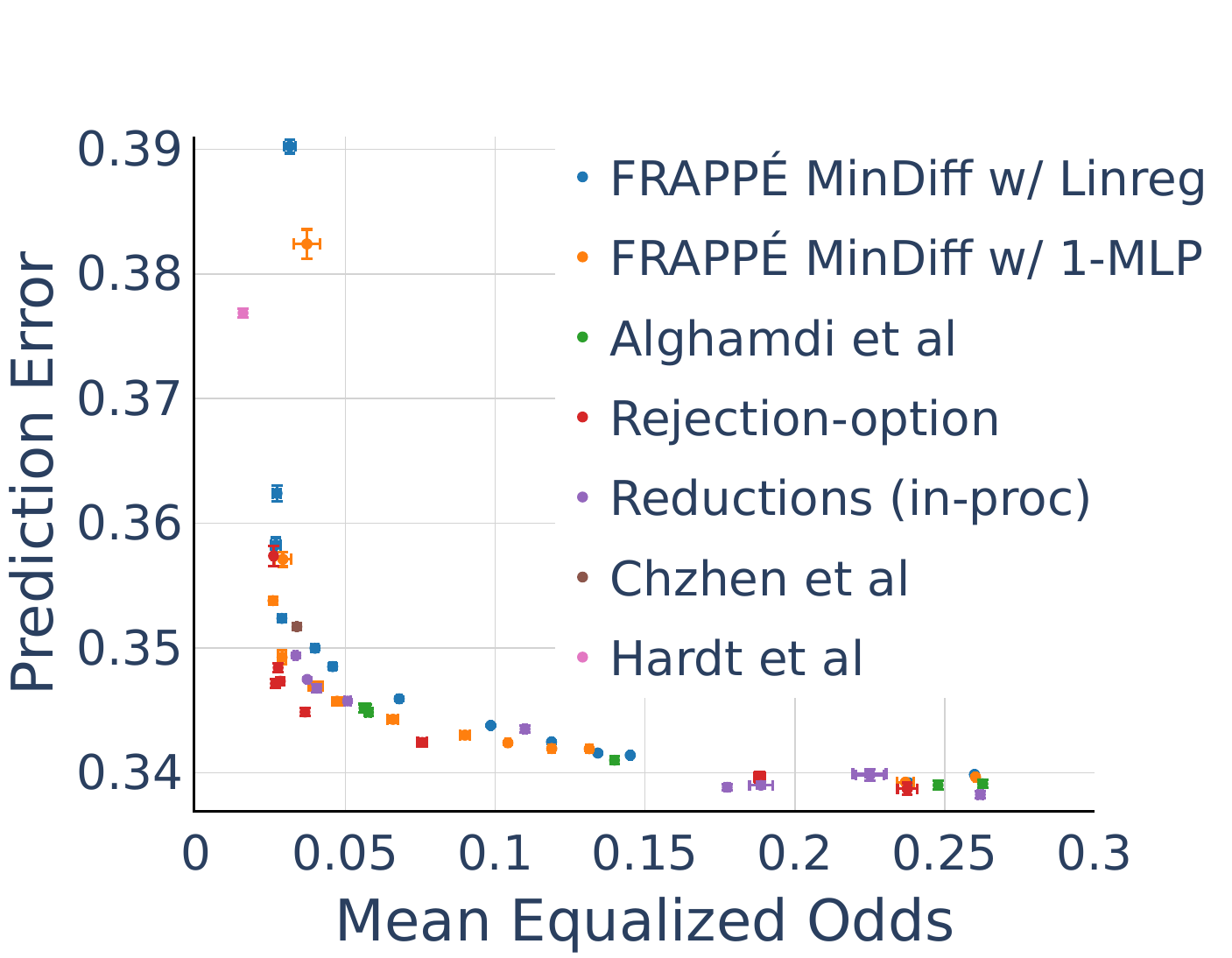}
    \caption{ENEM}
\end{subfigure}
\caption{Comparison between $\ours$ MinDiff for EqOdds and several other in- and post-processing methods for inducing group fairness. The pretrained model is \textbf{random forest (RF)}.}
\label{fig:app_more_baselines}
\vspace{-0.5cm}
\end{figure*}

\begin{figure*}[ht]
\centering
\begin{subfigure}[t]{0.24\textwidth}
    \includegraphics[width=\textwidth]{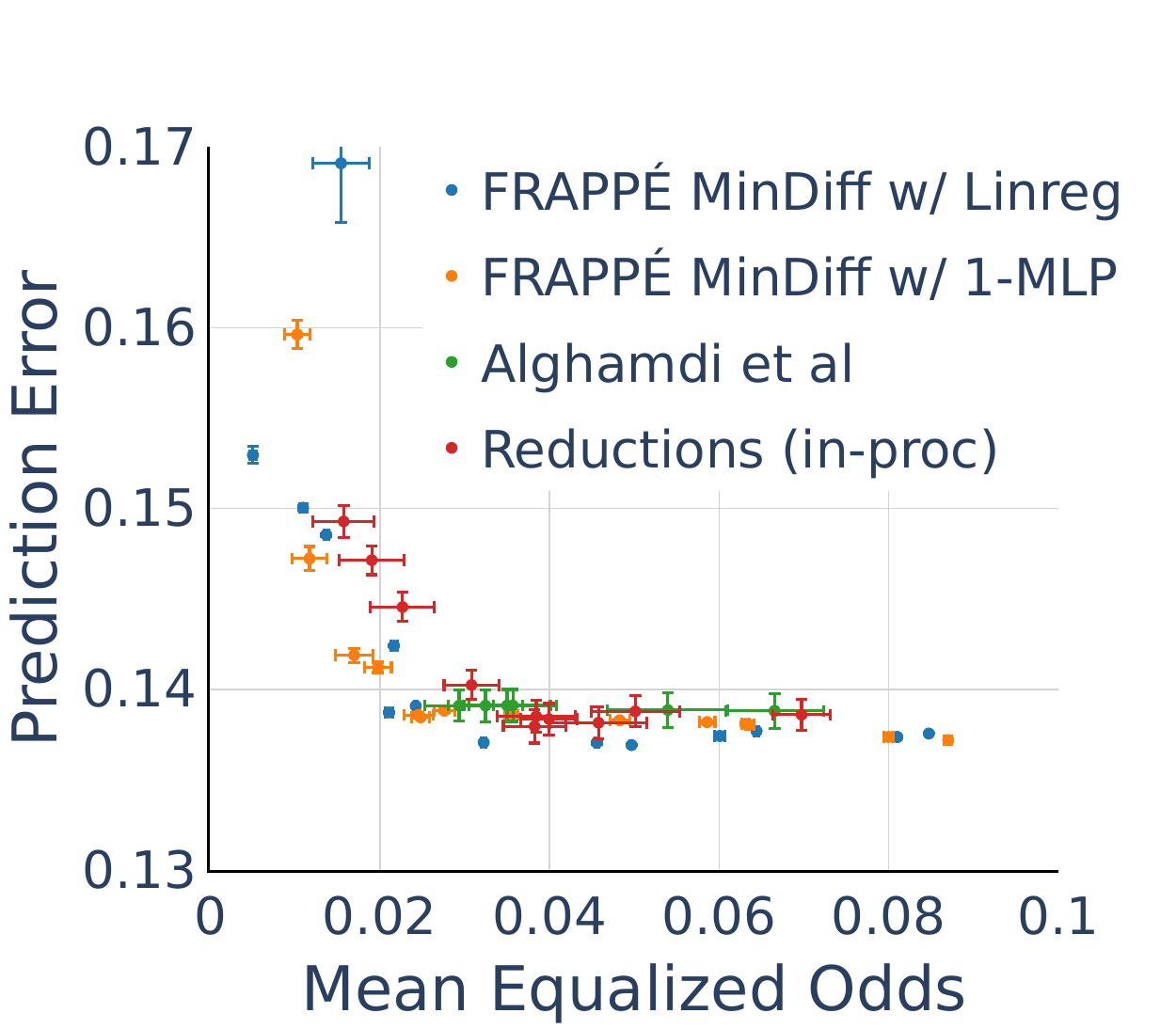}
    \caption{Adult}
\end{subfigure}
\hfill
\begin{subfigure}[t]{0.24\textwidth}
    \includegraphics[width=\textwidth]{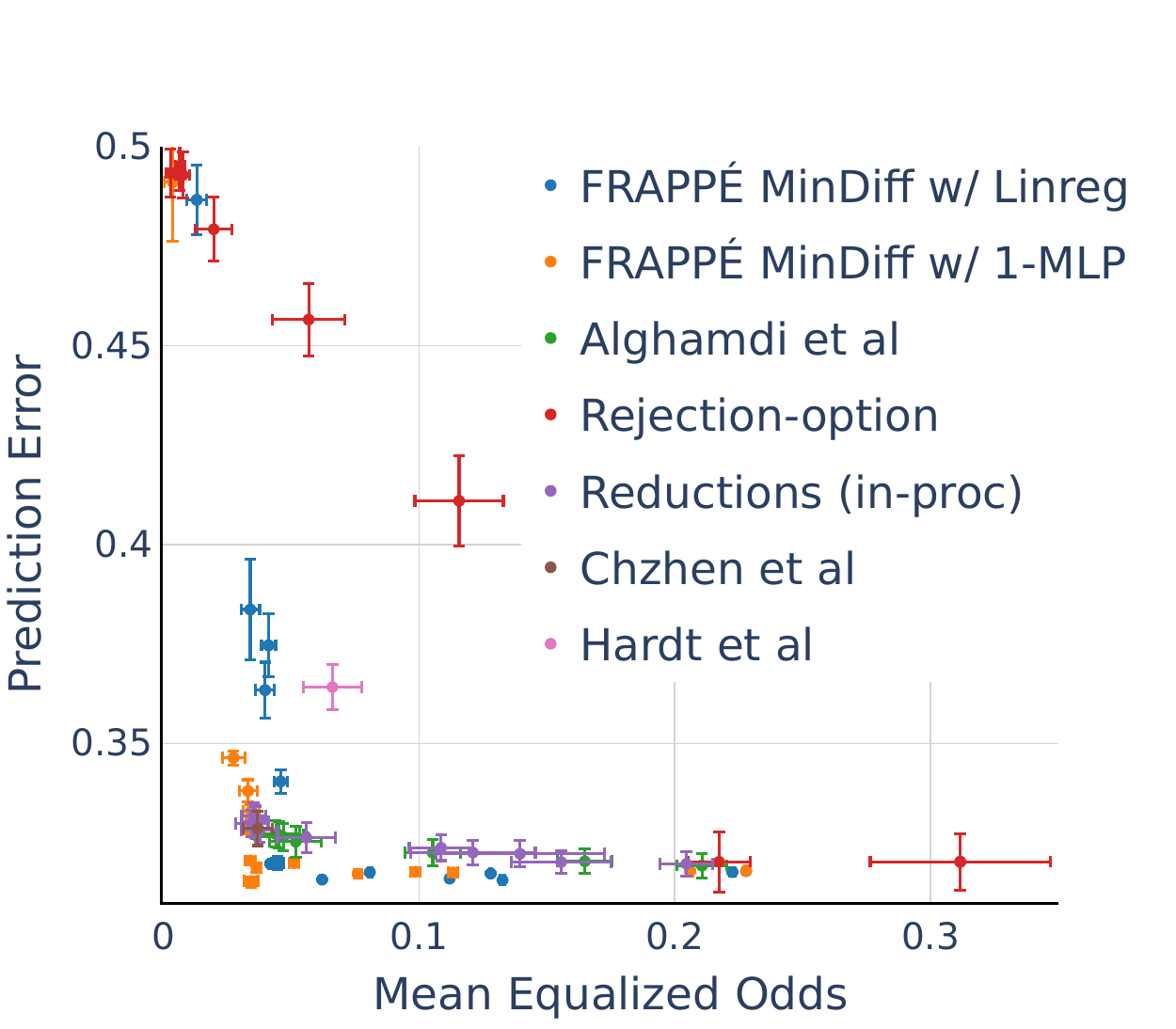}
    \caption{COMPAS}
\end{subfigure}
\hfill
\begin{subfigure}[t]{0.24\textwidth}
    \includegraphics[width=\textwidth]{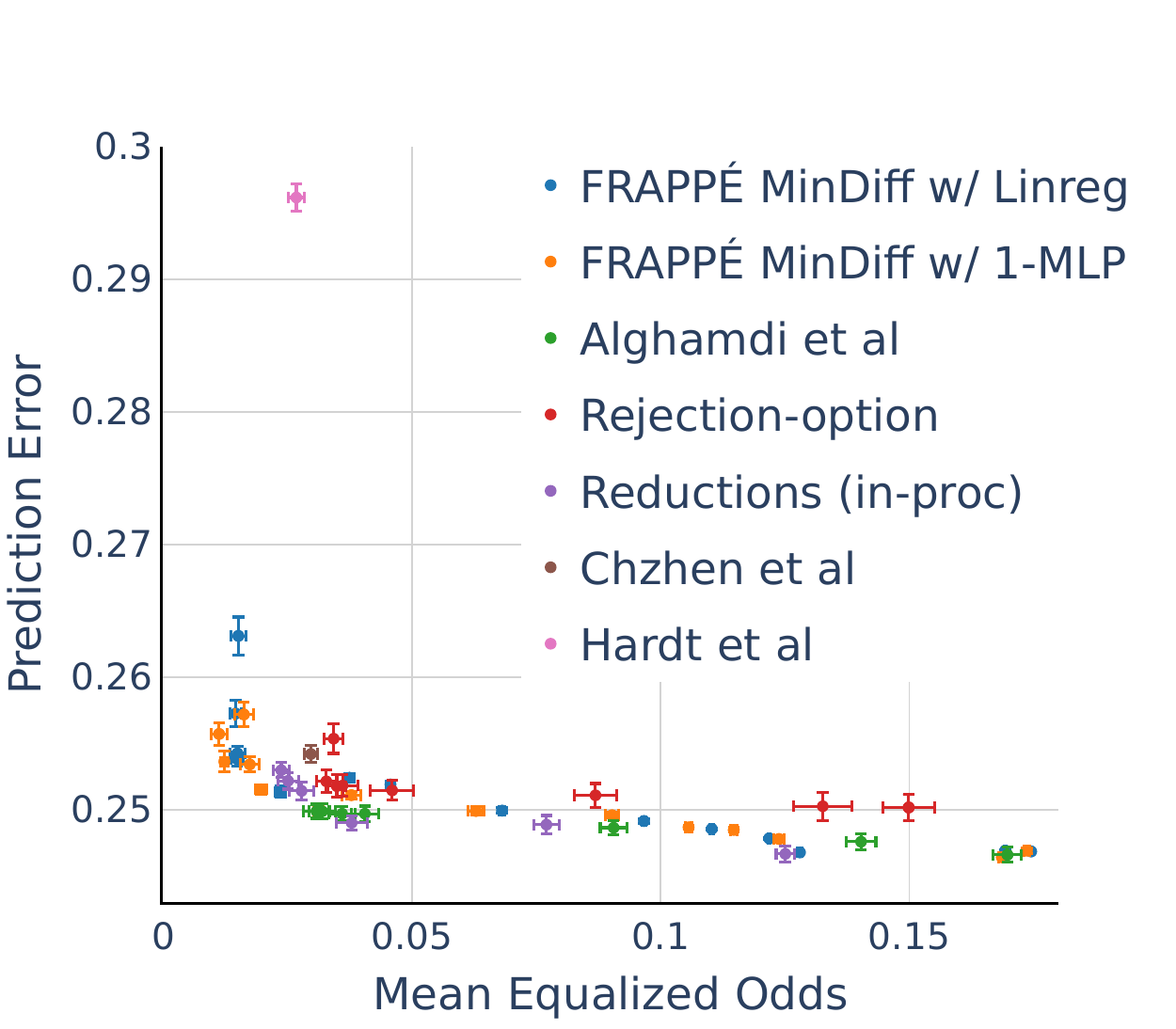}
    \caption{HSLS}
\end{subfigure}
\hfill
\begin{subfigure}[t]{0.24\textwidth}
    \includegraphics[width=\textwidth]{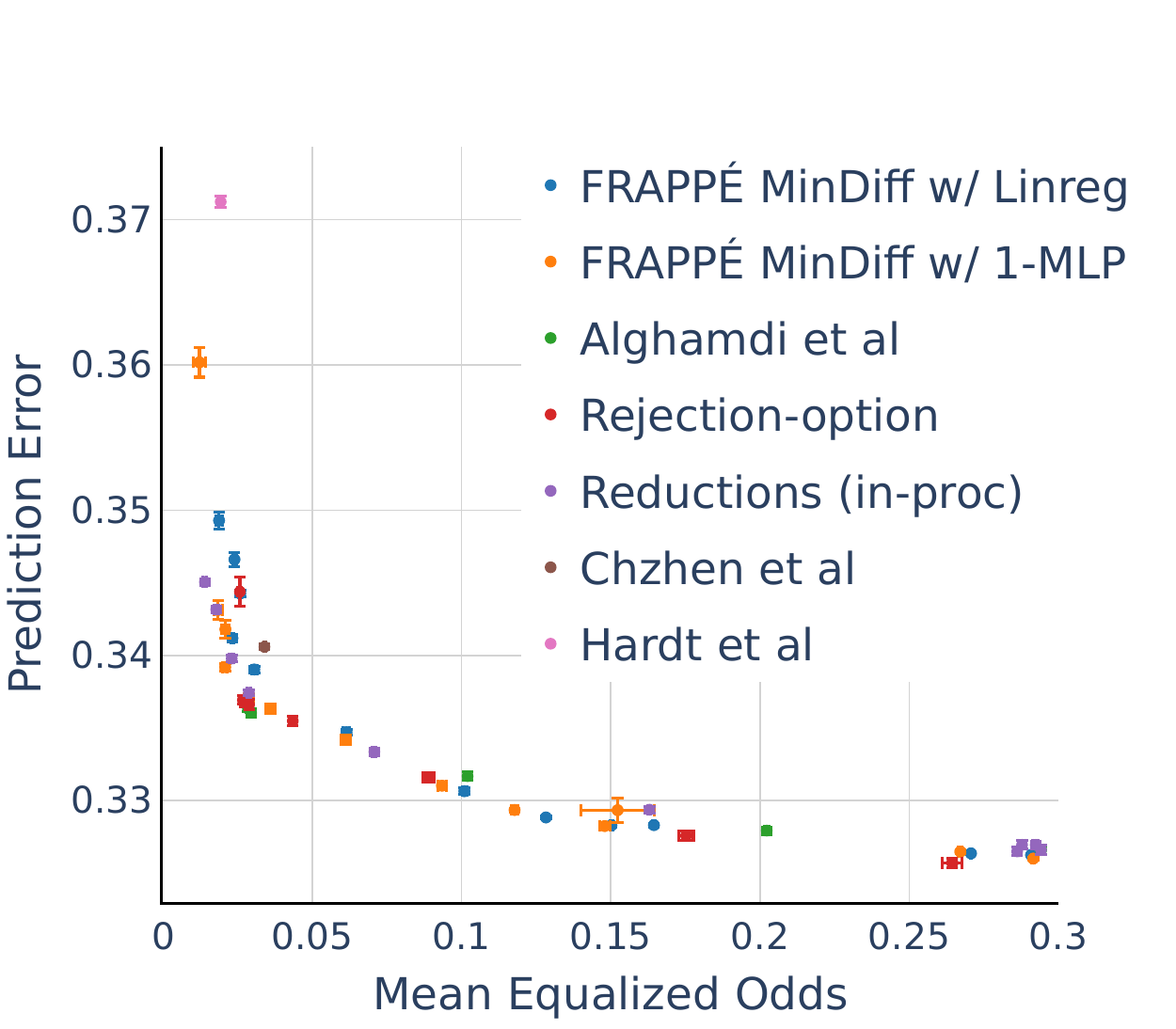}
    \caption{ENEM}
\end{subfigure}
\caption{Comparison between $\ours$ MinDiff for EqOdds and several other in- and post-processing methods for inducing group fairness. In contrast to \Cref{fig:app_more_baselines}, here the pre-trained model is \textbf{gradient-boosted machine (GBM)}.}
\label{fig:app_alghamdi_gbm}
\vspace{-0.5cm}
\end{figure*}

\begin{figure*}[!ht]
\centering
\begin{subfigure}[t]{0.24\textwidth}
    \includegraphics[width=\textwidth]{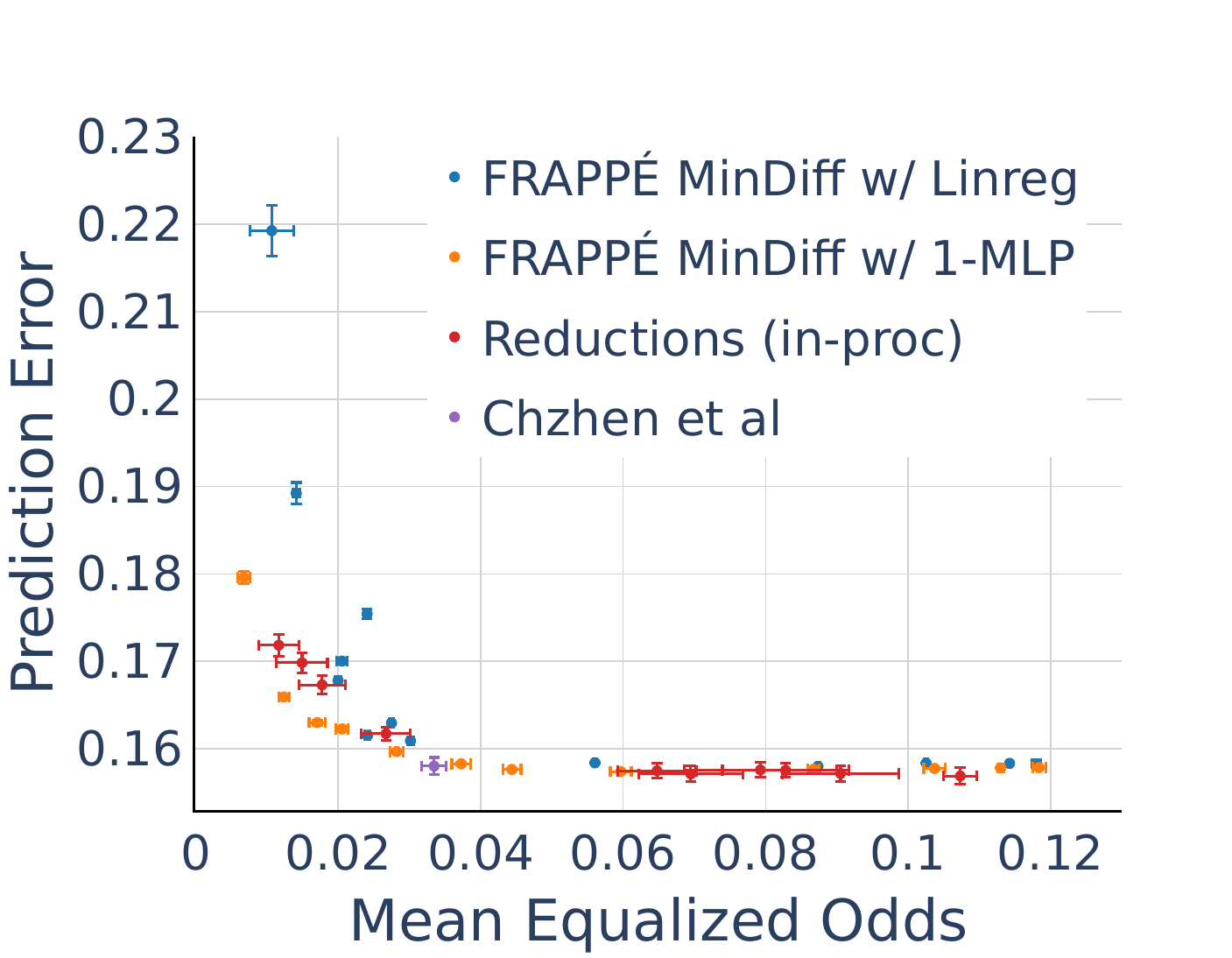}
    \caption{Adult}
\end{subfigure}
\hfill
\begin{subfigure}[t]{0.24\textwidth}
    \includegraphics[width=\textwidth]{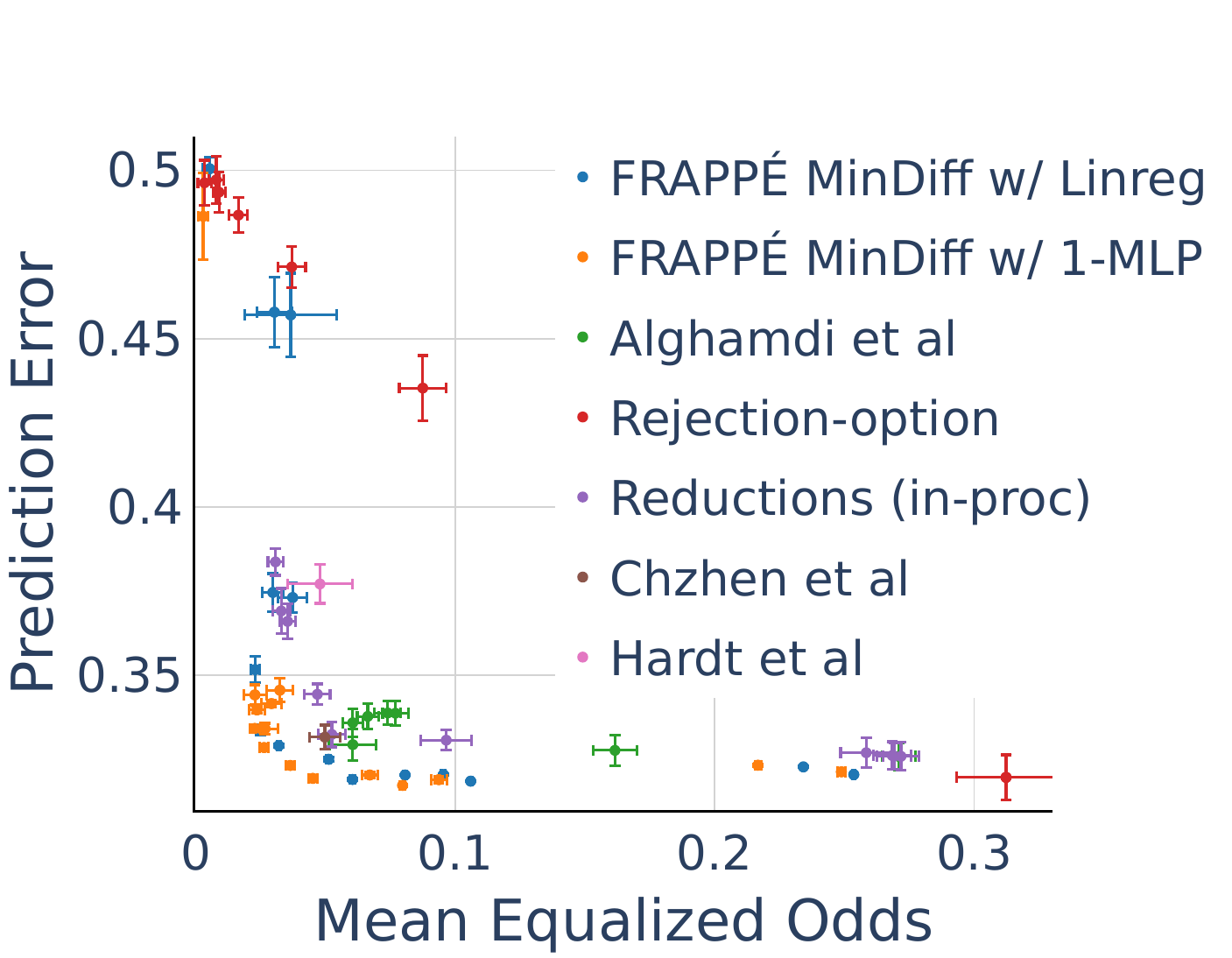}
    \caption{COMPAS}
\end{subfigure}
\hfill
\begin{subfigure}[t]{0.24\textwidth}
    \includegraphics[width=\textwidth]{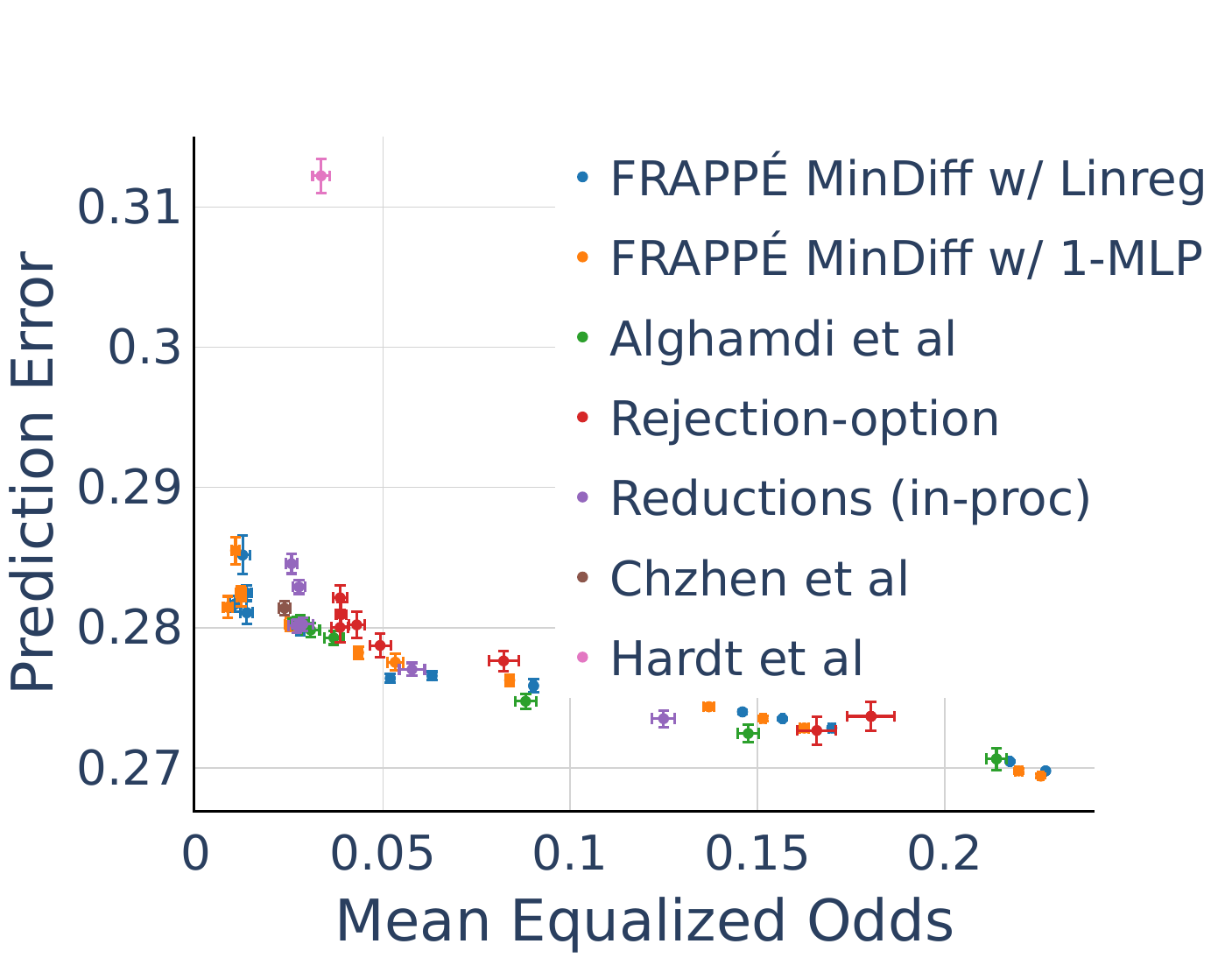}
    \caption{HSLS}
\end{subfigure}
\hfill
\begin{subfigure}[t]{0.24\textwidth}
    \includegraphics[width=\textwidth]{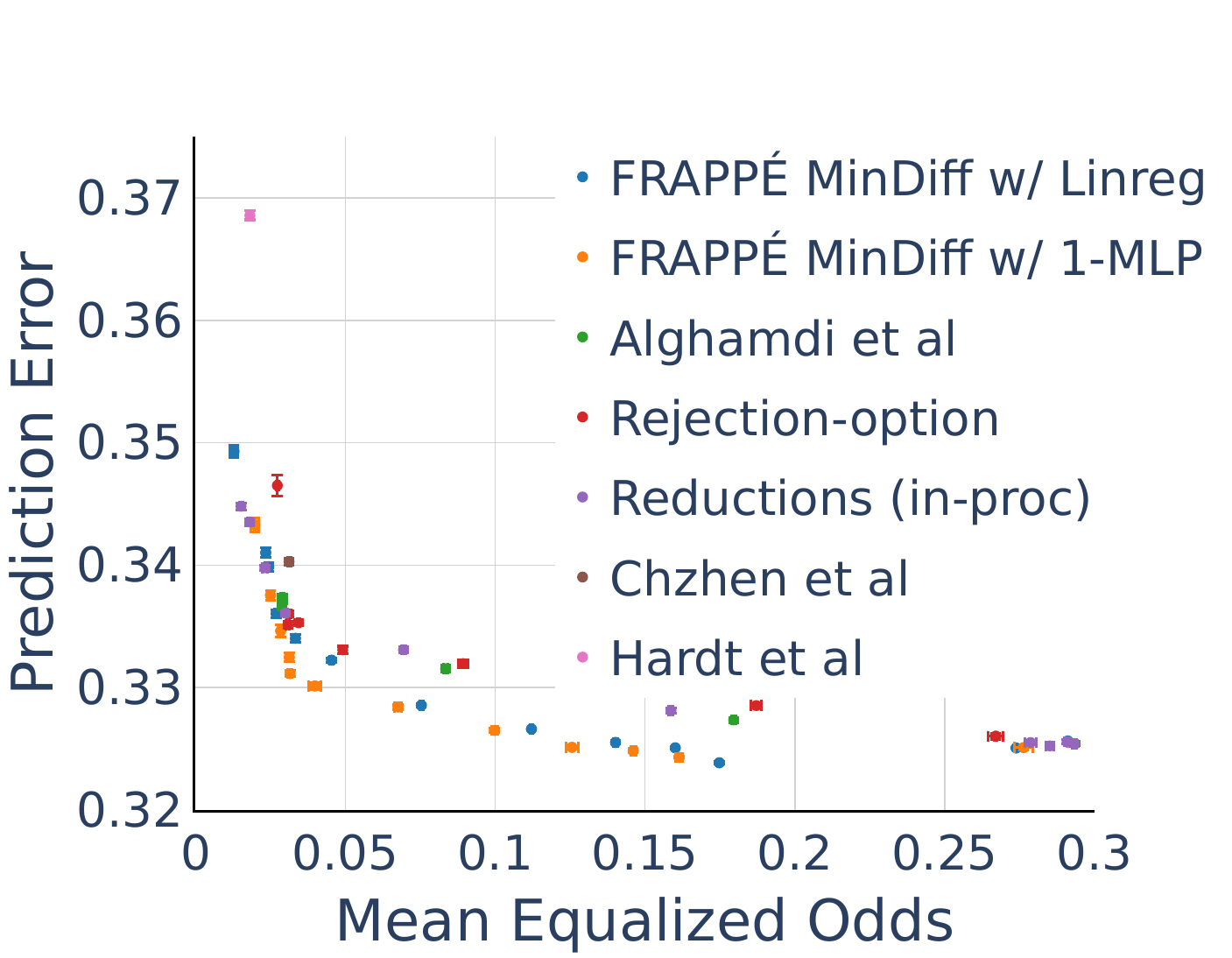}
    \caption{ENEM}
\end{subfigure}
\caption{Comparison between $\ours$ MinDiff for EqOdds and several other in- and post-processing methods for inducing group fairness. In contrast to \Cref{fig:app_more_baselines}, here the pre-trained model is \textbf{logistic regression}.}
\label{fig:app_alghamdi_logreg}
\vspace{-0.5cm}
\end{figure*}

\subsection{Comparison with model reprogramming}
\label{appendix:reprogramming}

Model reprogramming aims to reuse a pretrained model and adjust the inputs in order to elicit outputs with a desired property. In particular, for group fairness, \citet{zhang2022} consider a somewhat similar optimization objective to $\ours$. However, unlike $\ours$, the method of \citet{zhang2022} (i.e.\ Fairness Reprogramming) learns the parameters of a post-hoc transformation of the inputs of a pre-trained prediction model. On the one hand, on image data, choosing this transformation to be a border or a patch (see Figure 1 in \citet{zhang2022}) leads to remarkable results. To illustrate how $\ours$ methods perform on CelebA data, we consider the same experimental setting as in \citet{zhang2022} and provide in \Cref{fig:frappe_vs_reprogramming} the Pareto frontier obtained with the $\ours$ version of FERMI \citep{lowy2022}. 

\begin{figure}[h]
\centering
\includegraphics[width=0.5\textwidth]{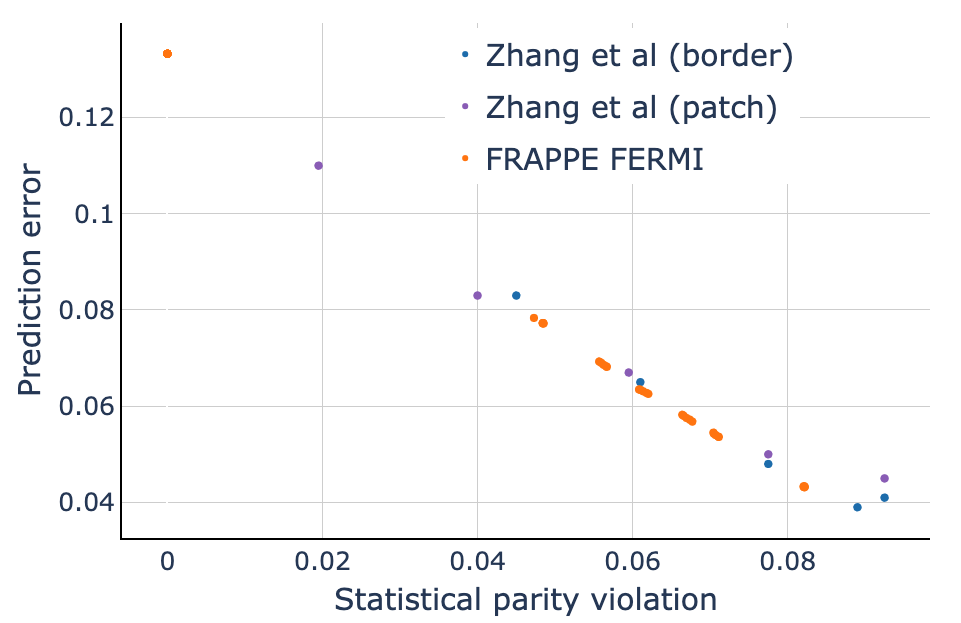}
\caption{$\ours$ FERMI leads to similar Pareto frontiers as Fairness Reprogramming \citep{zhang2022}.}
\label{fig:frappe_vs_reprogramming}
\end{figure}

Our experiments reveal that $\ours$ FERMI performs similarly to Fairness Reprogramming on this dataset. We note, however, that unlike Fairness Reprogramming, $\ours$ methods do not need to carefully select the family of post-hoc transformations. For Fairness Reprogramming, this choice has great influence on performance, as indicated, for instance, by Figures 3 and 4 in \citet{zhang2022}. Moreover, for Fairness Reprogramming the post-hoc transformation is specific to a data modality, and different problems may require a human expert to design a set of reasonable candidate post-hoc transformations. 

In fact, for tabular data, results in \citet{zhang2022} suggest that Fairness Reprogramming performs worse than standard techniques such as \citet{zhang2018}, which in turn is outperformed by more recent approaches like \citet{cho20}. We hypothesize that constructing an appropriate parametric transformation of the inputs (the so-called trigger) for the Fairness Reprogramming method is more challenging for structured tabular data than it is for image or text modalities. In contrast, the $\ours$ variant of \citet{cho20} matches the Pareto frontiers of the in-processing counterpart on several datasets (including Adult), as indicated in Figures~\ref{fig:ip_vs_pp_cho_main} and~\ref{fig:ip_vs_pp_cho}, and thus outperforms both \citet{zhang2018} and \citet{zhang2022}.

\subsection{In-processing MinDiff overfits the fairness regularizer}
\label{appendix:overfitting}

In this section we provide evidence that suggests that regularized in-processing objectives can overfit the fairness regularizer when trained on data with partial group labels. In particular, we consider MinDiff run on the Adult dataset. We subsample the dataset with sensitive attributes to be only $0.1\%$ of the original training data. As described in~\Cref{sec:novel_failure}, we use the entire training data for the predictive term in the loss.

\Cref{fig:app_overfitting} shows the median EqOpp violation (i.e.\ the FPR gap) as a function of the number of training epochs of in-processing MinDiff. The learning curves in the figure indicate that when the data with sensitive attributes is sufficiently large, the fairness violation is low on both training and test data. However, for partial group labels, in-processing MinDiff quickly achieves a vanishing fairness regularizer on the training data, while the test FPR gap continues to increase during training.

\begin{figure}[ht]
\centering
\begin{subfigure}[t]{0.4\textwidth}
\includegraphics[width=\textwidth]{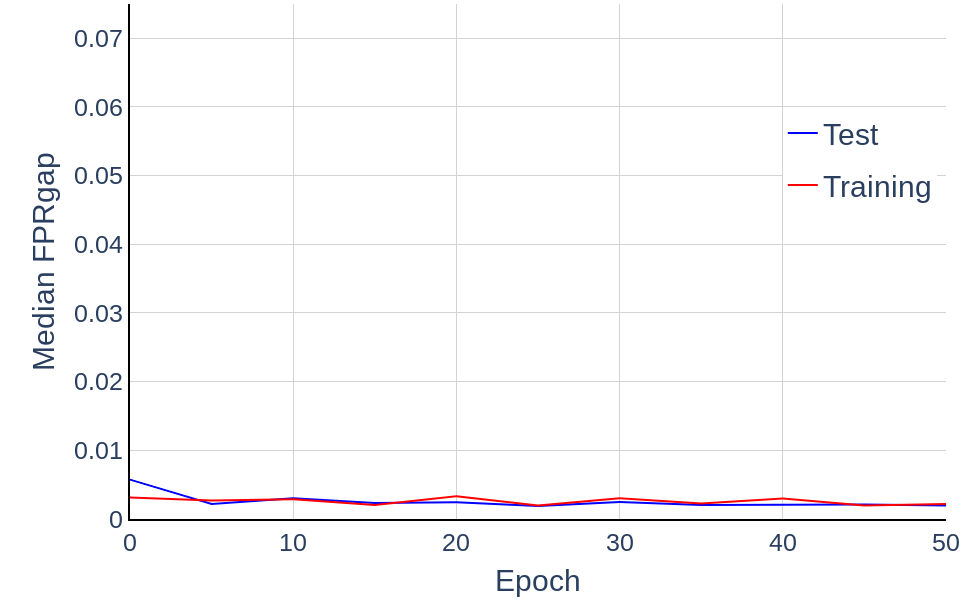}
\caption{$|\Dsens|=100\%\cdot|\Dpred|$}
\end{subfigure}
\begin{subfigure}[t]{0.4\textwidth}
\includegraphics[width=\textwidth]{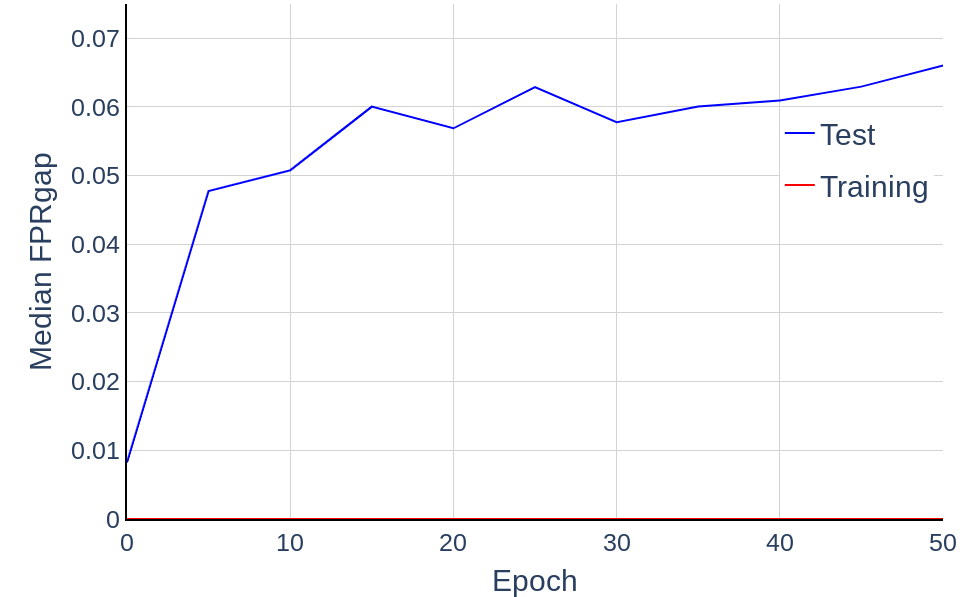}
\caption{$|\Dsens|=0.1\%\cdot|\Dpred|$}
\end{subfigure}
\caption{In-processing MinDiff achieves a low FPR both on training and test data if data with sensitive attributes is plentiful. When only partial group labels are available, the training FPR gap vanishes, while on the test set the unfairness of the model increases during training. Experiments run on the Adult dataset.}
\label{fig:app_overfitting}
\end{figure}

\subsection{$\ours$ MinDiff without early-stopping for data with partial group labels}
\label{appendix:frappe_without_es}

\Cref{fig:low_sample} shows that with optimal early-stopping regularization, $\ours$ MinDiff significantly outperforms its in-processing counterpart when only partial group labels are available for training. In this section, we present evidence that suggests that even without early-stopping, $\ours$ post-processing can perform well in this setting. In~\Cref{fig:app_low_sample} we compare in-processing MinDiff regularized with early-stopping (like in~\Cref{fig:low_sample}) to $\ours$ MinDiff \emph{with no regularization}. The Pareto frontier for $\ours$ is once again better than for the in-processing method. Moreover, without the need to regularize the $\ours$ method, training does away with careful hyperparameter tuning of the important early stopping iteration.

\begin{figure*}[ht]
\centering
\includegraphics[width=\textwidth]{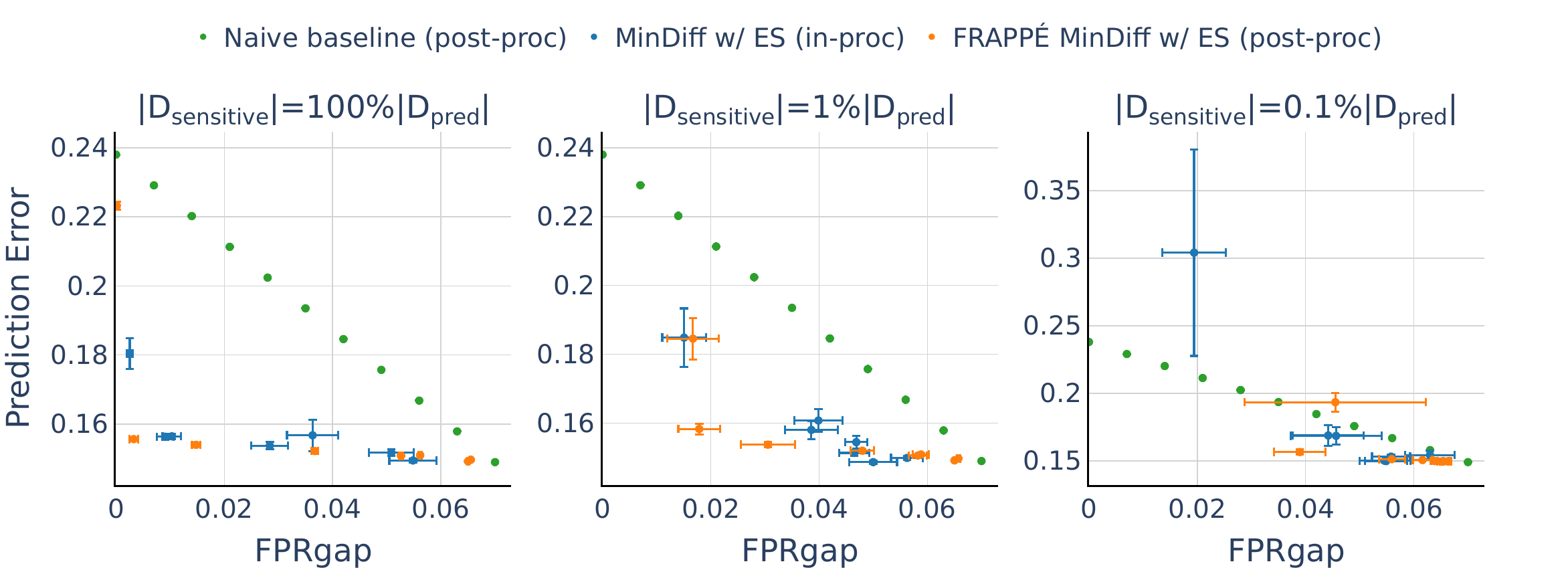}
\caption{In-processing MinDiff with early-stopping regularization and $\ours$ MinDiff \emph{without early-stopping} (ES). On data with partial group labels, the $\ours$ method continues to outperform the in-processing variant even without early-stopping regularization. Experiments run on the Adult dataset.}
\label{fig:app_low_sample}
\end{figure*}

\subsection{More results on what is captured by the learned post-hoc transformation}
\label{appendix:transformation_analysis}

In this section we complement \Cref{fig:analysis_main} with additional evidence that the post-hoc transformation learned by $\ours$ methods is correlated not only with the sensitive attribute $A$, but also with features that are predictive of the target class label. \Cref{fig:app_analysis_adult} shows that the same trends observed on COMPAS, also occur for the Adult dataset.

In addition, \Cref{fig:app_cond_analysis_compas,fig:app_cond_analysis_adult} show the (absolute value of the) conditional correlation between $\mult(X)$ and each of the features, given the sensitive attribute $A$. For reference, a group-dependent transformation, such as the prior post-processing techniques \citep{hardt16, alghamdi22, xian23} would be constant given $A$, and hence, statistically independent of all the other features. In contrast, the post-hoc transformation learned with $\ours$ is highly correlated with features that are predictive of the class label (e.g.\ priors count for COMPAS; age, education or marital status for Adult).

\begin{figure*}[!ht]
\centering
\includegraphics[width=0.6\textwidth]{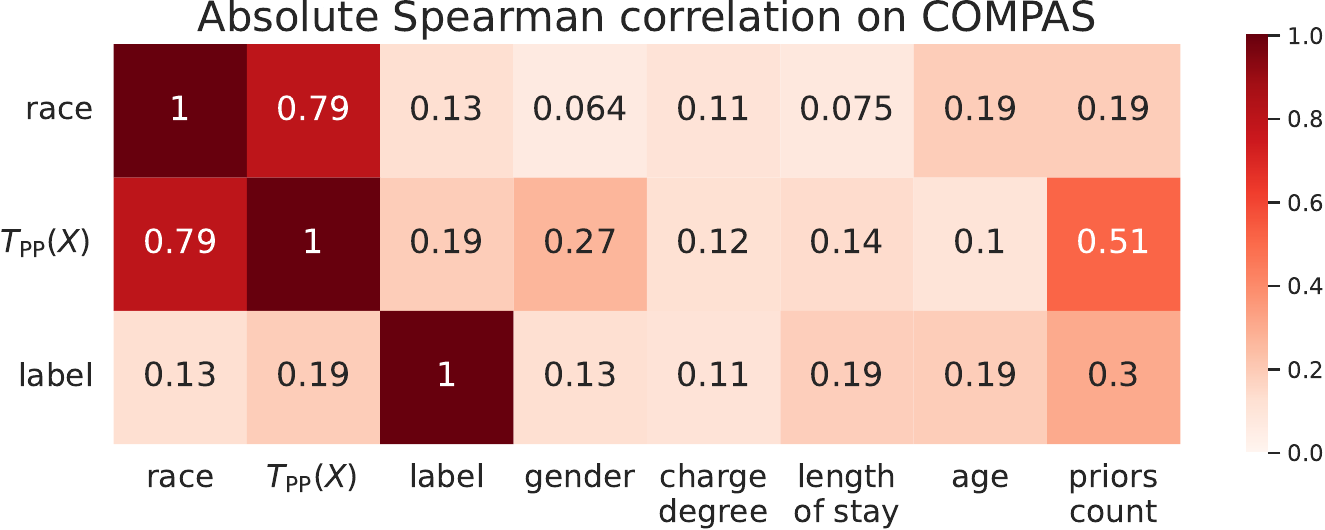}
\caption{The value of the learned post-hoc transformation $\mult(X)$ is highly correlated with both the sensitive attribute (i.e.\ race), as well as with features that are predictive of the class label (e.g.\ priors count). 
}
\label{fig:app_analysis_main}
\vspace{-0.5cm}
\end{figure*}

\begin{figure*}[!ht]
\centering
\includegraphics[width=0.8\textwidth]{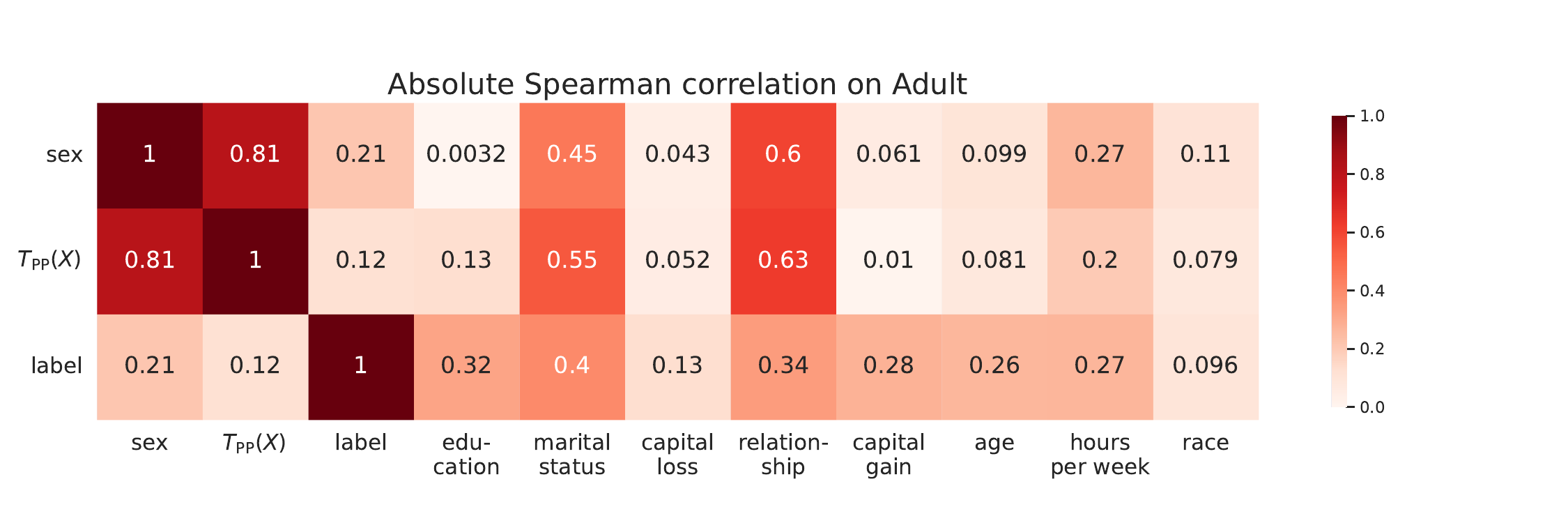}
\vspace{-0.3cm}
\caption{Counterpart of \Cref{fig:app_analysis_main}, but for the Adult dataset. The value of the learned post-hoc transformation $\mult(X)$ is highly correlated with the sensitive attribute (i.e.\ gender), as well as with features that are themselves correlated with $A$.}
\label{fig:app_analysis_adult}
\vspace{-0.5cm}
\end{figure*}

\begin{figure*}[!ht]
\centering
\includegraphics[width=0.6\textwidth]{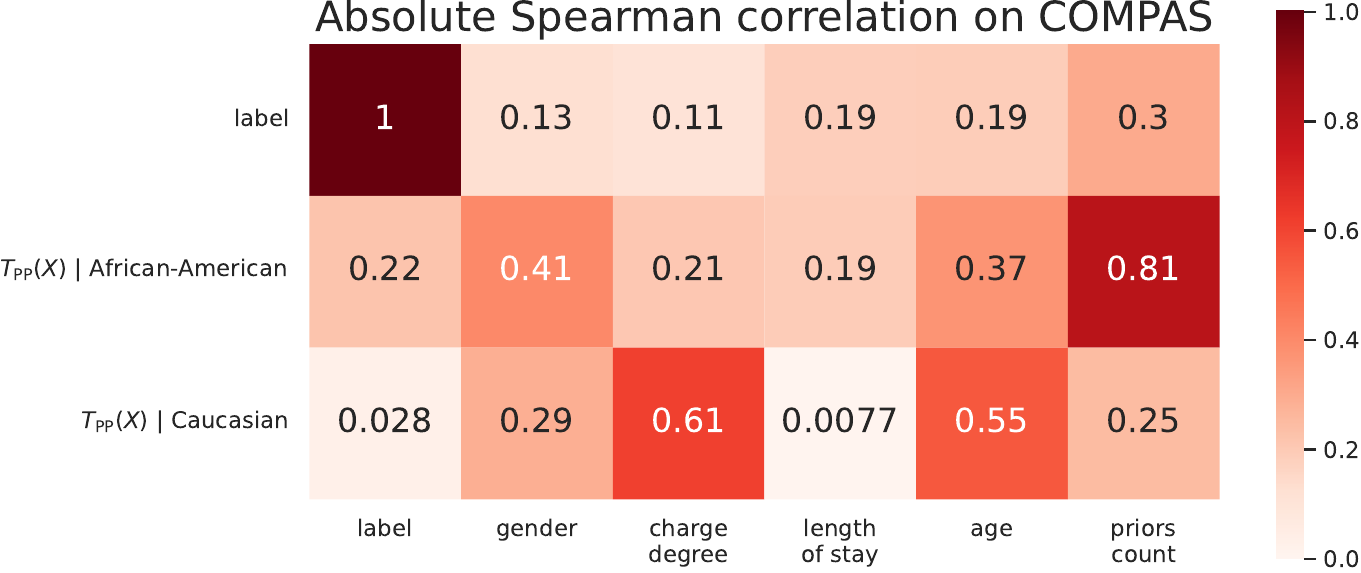}
\vspace{-0.1cm}
\caption{Conditional correlation between $\mult(X)$ and each of the features, on the COMPAS dataset. The correlation is higher for features that are predictive of the label. In contrast, a group-dependent transformation would be conditionally independent of all features given $A$.}
\label{fig:app_cond_analysis_compas}
\vspace{-0.5cm}
\end{figure*}

\begin{figure*}[!ht]
\centering
\includegraphics[width=0.7\textwidth]{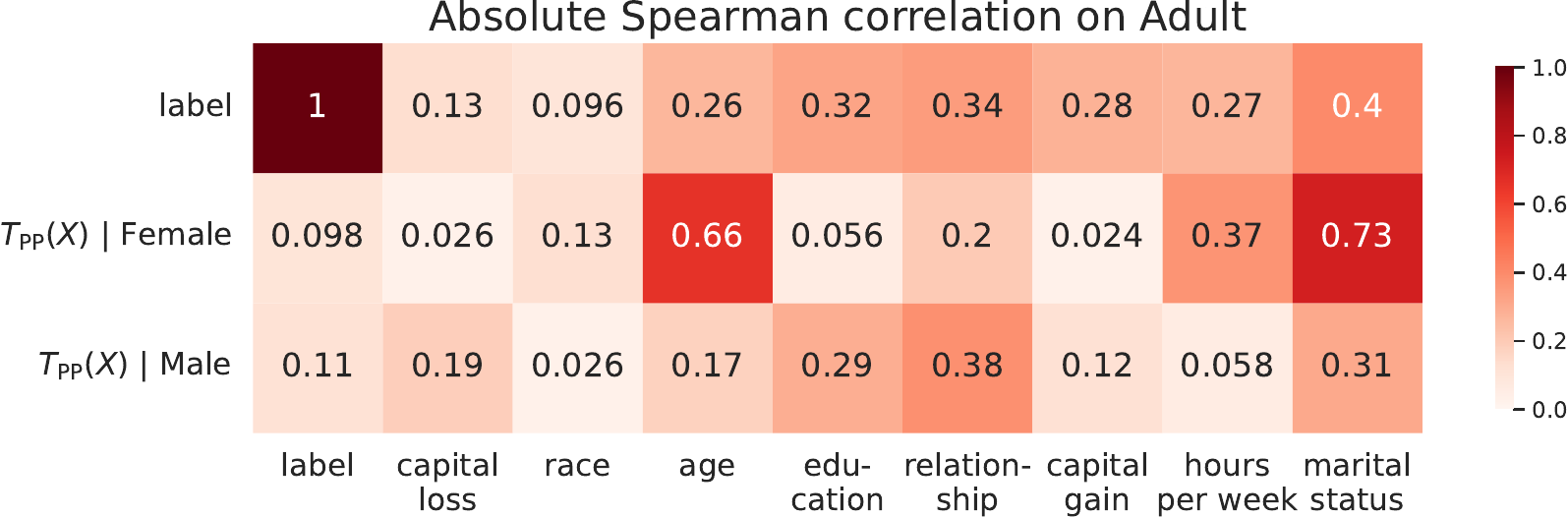}
\vspace{-0.1cm}
\caption{Conditional correlation between $\mult(X)$ and each of the features, on the Adult dataset. The correlation is higher for features that are predictive of the label. In contrast, a group-dependent transformation would be conditionally independent of all features given $A$}
\label{fig:app_cond_analysis_adult}
\vspace{-0.5cm}
\end{figure*}

\subsection{Varying function complexity for $\fctbase$ and $\mult$}
\label{appendix:vary_complexity}

In this section, we discuss how changing the function class of the prediction model $\fctbase$ or the post-hoc module $\mult$ impacts the performance of $\ours$ methods.

Figures~\ref{fig:app_more_baselines},~\ref{fig:app_alghamdi_gbm}, and~\ref{fig:app_alghamdi_logreg} show how $\ours$ MinDiff compares to several prior approaches when the base model is a random forest, a gradient-boosted machine or logistic regression, respectively. Notably, the computation cost for training $\ours$ MinDiff is roughly unchanged for all three base model classes.

Furthermore, in each of these figures, we consider two different function classes for the post-processing transformation $\mult$: a linear model and a $1$-hidden layer multi-layer perceptron. These experiments confirm the intuition that a more expressive post-hoc transformation for $\ours$ can lead to better Pareto frontiers.

Finally, in \Cref{fig:vary_inproc_complexity} we highlight how the complexity of the prediction model produced by an in-processing technique affects the Pareto frontier relative to its $\ours$ counterpart. We assume the same prediction tasks for Adult and COMPAS as in the rest of the paper. More specifically, the left and right panels use the same pre-processing steps as Figures~\ref{fig:comparison_alghamdi_compas} and~\ref{fig:ip_vs_pp_cho_compas}, respectively. For this experiment, we consider two different fairness definitions (EqOpp and EqOdds) and two different classes of base models for $\ours$ (3-layer multi-layer perceptron, i.e.\ 3-MLP, and a gradient-boosted machine, i.e.\ GBM). We compare $\ours$ MinDiff to two different in-processing methods, MinDiff \citep{beutel2019, prost19} and Reductions \citep{agarwal18}. We choose the optimal hyperparameters (e.g.\ learning rate) using a held-out validation set and the same methodology as in the rest of the experiments. These experiments reveal that $\ours$ methods match the performance of in-processing with a complex function class $\mathcal{F}$, which in turn outperforms in-processing with a simpler function class $\mathcal{F}={\text{linear models}}$. 

\begin{figure}[!h]
\centering
\begin{subfigure}[t]{0.49\textwidth}
    \includegraphics[width=\textwidth]{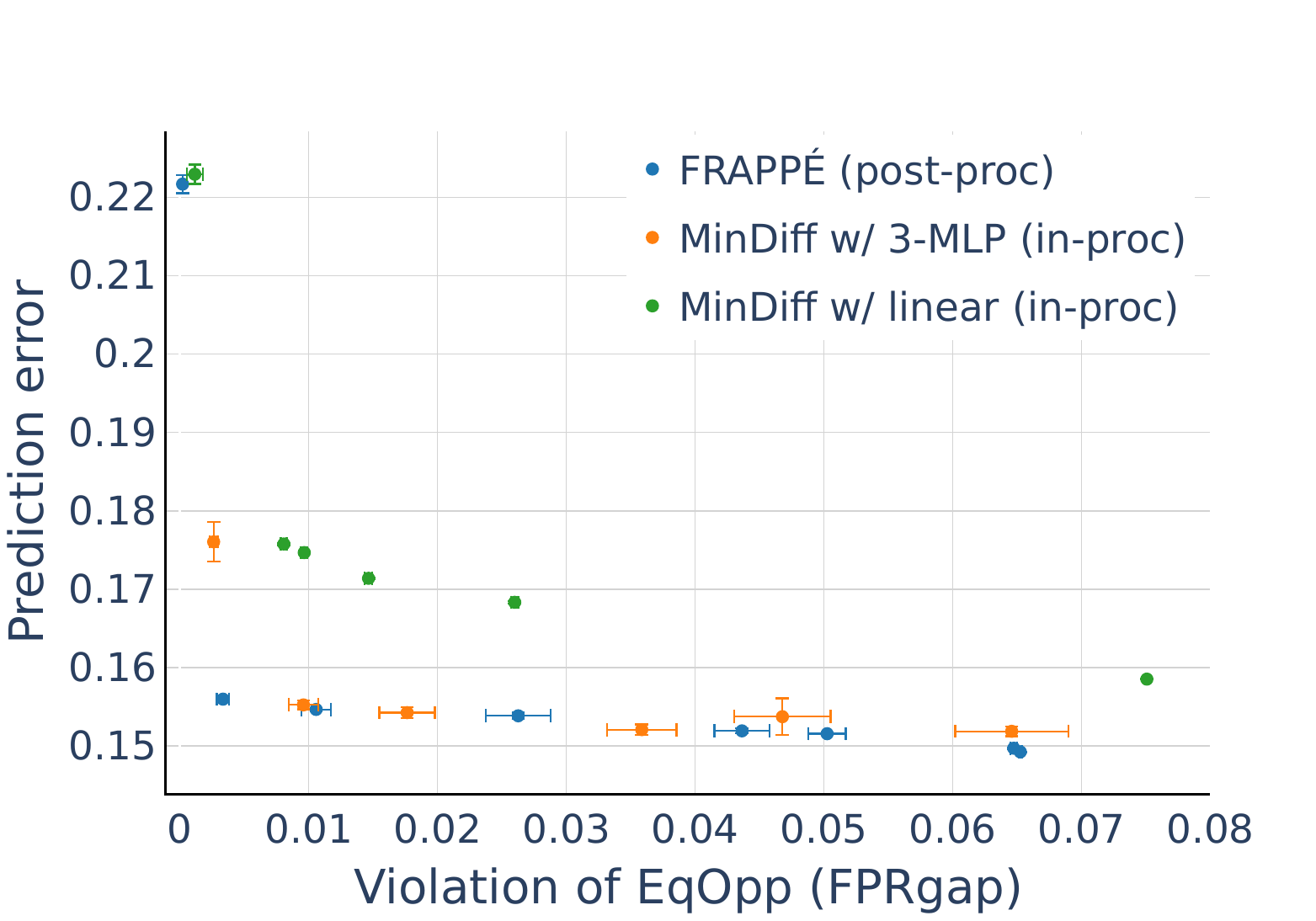}
    \caption{Inducing EqOpp on Adult.}
\end{subfigure}
\hfill
\begin{subfigure}[t]{0.47\textwidth}
    \includegraphics[width=\textwidth]{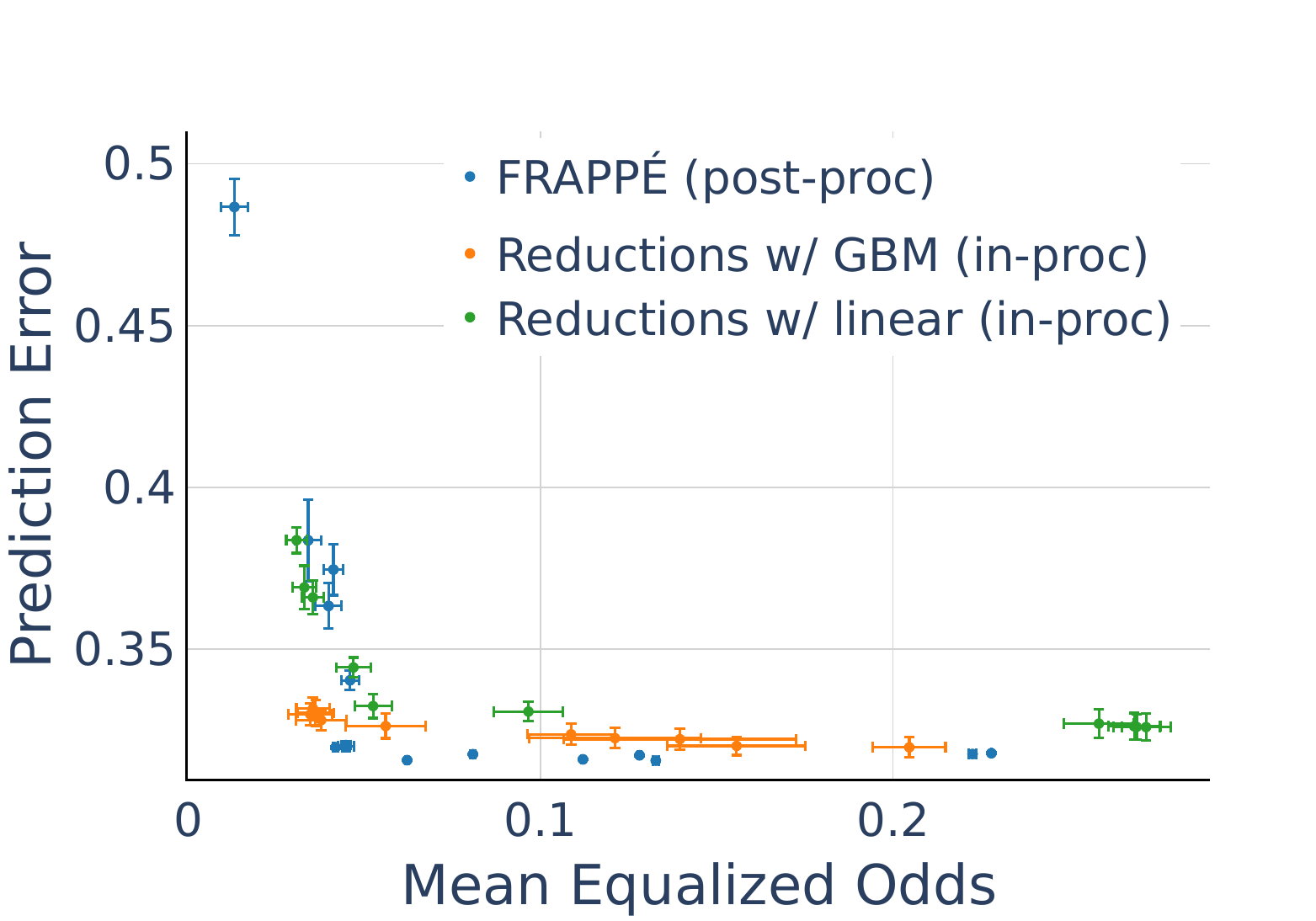}
    \caption{Inducing EqOdds on COMPAS.}
\end{subfigure}
\caption{$\ours$ (in \textcolor{blue}{blue}) matches the performance of in-processing with the same base model complexity (i.e.\ in \textcolor{orange}{orange}) and outperforms in-processing that uses less complex linear models (in \textcolor{green}{green}). For $\ours$, the post-hoc module is always linear, while the base model is a 3-layer MLP (left) or a GBM (right).}
\label{fig:vary_inproc_complexity}
\vspace{128in}
\end{figure}

\end{document}